\documentclass{article}

\PassOptionsToPackage{numbers}{natbib}
\usepackage[main, final]{neurips_2026}
\usepackage[utf8]{inputenc}
\usepackage[T1]{fontenc}
\usepackage{natbib}
\usepackage{hyperref}
\usepackage{url}
\usepackage[table]{xcolor}
\colorlet{sectionbg}{gray!15}
\usepackage{longtable}
\usepackage{booktabs}
\usepackage{amsfonts}
\usepackage{amsmath}
\usepackage{amssymb}
\usepackage{nicefrac}
\usepackage{microtype}
\usepackage{algorithm}
\usepackage{algorithmic}
\usepackage{graphicx}
\usepackage{subcaption}
\usepackage{multirow}
\usepackage{pifont}
\usepackage{amsthm}
\usepackage{enumitem}
\usepackage{silence}
\usepackage[acronym]{glossaries}
\makeglossaries

\newacronym{dfl}{DFL}{Decision-Focused Learning}
\newacronym{oco}{OCO}{Online Convex Optimization}
\newacronym{igt}{IGT}{Implicit Gradient Transport}
\newacronym{omd}{OMD}{Online Mirror Descent}
\newacronym{dde}{DDE}{Delay Differential Equation}

\newcommand{\cmark}{\ding{51}}
\newcommand{\xmark}{\ding{55}}

\newcommand{\R}{\mathbb{R}}
\newcommand{\E}{\mathbb{E}}
\newcommand{\C}{\mathbb{C}}
\newcommand{\argmin}{\operatorname{argmin}}

\newcommand{\norm}[1]{\left\|#1\right\|}

\newtheorem{assumption}{Assumption}
\newtheorem{theorem}{Theorem}
\newtheorem{lemma}{Lemma}
\newtheorem{corollary}{Corollary}
\newtheorem{proposition}{Proposition}
\newtheorem{remark}{Remark}

\title{IGT-OMD: Implicit Gradient Transport for Decision-Focused Learning under Delayed Feedback}

\author{%
  Benjamin Amoh \quad Geoffrey Parker \quad Wesley Marrero \\
  Thayer School of Engineering, Dartmouth College \\
  \texttt{benjamin.k.amoh.th@dartmouth.edu}
}

\begin{document}

\maketitle

\begin{abstract}
Decision-focused learning trains predictive models end-to-end against downstream decision loss, but online settings suffer delayed feedback: outcomes may not arrive for many environment interactions. We identify \emph{staleness amplification}, a 
failure mode unique to bilevel optimization under delay, in which gradient staleness couples with inner-solver sensitivity to inflate regret beyond single-level delay theory. We prove that any black-box delayed optimizer 
incurs an irreducible regret cost from inner-solver approximation error, and that gradient staleness contributes a quadratically growing transport error without bilevel-aware correction. Our algorithm, \textbf{IGT-OMD}, applies Implicit Gradient Transport to hypergradients within Online Mirror Descent, re-evaluating stale gradients at the current parameters using stored inner solutions. This method reduces transport error from a quadratic to a linear dependence on delay and achieves the first sublinear regret bound for delayed bilevel optimization with queue-length-adaptive step sizes. Controlled 
experiments provide a \emph{mechanistic fingerprint}: transport benefit is exactly $0.0\%$ (by construction) at unit delay and grows monotonically to $9.5\%$ at fifty rounds ($p<0.001$), isolating the correction's effect. On Warcraft shortest-path, IGT-OMD reduces decision-loss optimality gap by $15$--$36\%$ over D-FTRL/2-Stage baselines; on Linear Quadratic Regulator benchmark it maintains a constant maximum stable learning rate across delays where bilevel-unaware methods degrade by up to $9.3\times$.
\end{abstract}

\section{Introduction}
\label{sec:introduction}
In many operational settings (from inventory management to energy dispatch to personalized medicine), a predictive model feeds its forecast into a downstream optimization solver whose solution is then deployed. \Gls{dfl} trains such models end-to-end against a \emph{decision loss} rather than a surrogate prediction error~\cite{elmachtoub2022smart,wilder2019melding,mandi2024decision}. The resulting training problem is inherently bilevel: the \emph{outer level} adjusts the predictor parameters~$\theta$ (e.g., neural-network weights), while the \emph{inner level} solves for an optimal decision $w^*(\theta)$ (e.g., a control policy or a transport plan).

In online deployments, feedback is often delayed. A decision made at environment interaction (i.e., round) $t$ yields observable outcome feedback only at round~$t{+}d_t$, where the delay~$d_t$ may vary. At any given round, several past decisions may still be awaiting feedback; we call this set the \emph{queue} with size (i.e., \emph{queue length}) $\sigma_t$, and write $\sigma_{\max}$ for its worst-case value. Single-level delayed optimizers such as Delayed Follow-the-Regularized-Leader (D-FTRL;~\citealp{joulani2013online},~\cite{shalev2012online}) and Robust \Gls{omd} (OMD;~\citealp{quanrud2015online}) handle this challenge by correcting for the outer-parameter drift $\|\theta_t-\theta_{t-d_t}\|$. 

\textbf{Staleness amplification: why bilevel delay is harder.}\; In single-level \Gls{oco}, a stale gradient is a bounded perturbation independent of optimizer state. The bilevel setting is structurally different: the hypergradient of $F(\theta)=\mathcal{L}_{\text{true}}(w^*(\theta);\theta)$ depends on the inner minimizer $w^*(\theta)$, which drifts as $\theta$ changes. Feedback generated by an inner solution that has since drifted introduces gradient error proportional to accumulated outer-parameter change--\emph{staleness amplification}. This has two structural effects formalized in Theorem~\ref{thm:staleness_amplification}: an irreducible inner-solver floor on regret, and a per-round transport error that is quadratic in queue length without bilevel-aware correction.

\textbf{IGT-OMD: how we fix it.}\; The key idea of our algorithm is to \emph{re-evaluate} stale hypergradients at the current parameters rather than using them as-is. When feedback from a past round~$s$ arrives, our algorithm has already stored an inner solution~$w_s$ and adjoint vector~$v_s^*$ to cheaply recompute the hypergradient at the current~$\theta_t$ without re-solving the inner problem. Accumulating these corrections via \Gls{igt} ~\cite{arnold2019reducing} replaces the squared total drift $\|\theta_t-\theta_{t-d_t}\|^2$ with a sum of squared per-step changes $\sum_{s}\|\theta_{s+1}-\theta_s\|^2$, which is a factor~$\sigma_{\max}$ smaller. Embedding this corrected gradient in \Gls{omd} with a step size that tracks the running queue envelope---motivated by \Gls{dde} stability analysis~\cite{yu2025role}---yields our algorithm, \textbf{\gls{igt}-\gls{omd}}.

\textbf{Contributions.}\; \textbf{(1)}~\emph{Algorithm.} \gls{igt}-\gls{omd} corrects stale hypergradients at $O(p)$ cost per outstanding round without re-solving inner problems (Algorithm~\ref{alg:igt_omd}). \textbf{(2)}~\emph{Analysis.} We prove the first bilevel-delayed regret bound $O(\sqrt{T\sigma_{\max}}+T\epsilon_{\mathrm{inner}}^2)$ (Theorem~\ref{thm:bilevel_convergence}), a staleness amplification lower bound showing $\Omega(T\epsilon_{\mathrm{inner}}^2)$ is tight (Theorem~\ref{thm:staleness_amplification}), an inner-loop apathy decomposition decoupling staleness from inner-solver quality (Theorem~\ref{thm:inner_loop_apathy}), and \gls{dde} stability analysis (Propositions~\ref{prop:dde_stability}--\ref{prop:discrete_consistency}). \textbf{(3)}~\emph{Experiments.} On LQR, \gls{igt}-\gls{omd} maintains constant stability across all delays; on Warcraft, it reduces optimality gap by $15$--$36\%$ over D-FTRL/2-Stage baselines. A controlled optimizer experiment confirms the predicted mechanistic signature with $9.5\%$ improvement at $d=50$ ($p<0.001$).

\subsection{Related Work}
Table~\ref{tab:related_work} positions \gls{igt}-\gls{omd} against relevant prior work. Previous \gls{dfl} 
methods~\cite{elmachtoub2022smart,vlastelica2019differentiation,wilder2019melding,tang2022pyepo,capitaine2025online} 
solve the bilevel training problem but assume immediate feedback. Conceptually, DFL is a task-driven subclass of bilevel optimization: the inner problem is the downstream decision layer, and the outer model is trained on realized decision loss. Under this lens, we view~\cite{nikishin2022control} as closer to DFL-style prior work than general-purpose bilevel optimization. Single-level delayed 
optimizers~\cite{joulani2013online,quanrud2015online,flaspohler2021online,huang2023banker,ryabchenko2026reduction,qiu2026decentralized} correct for outer-parameter drift but treat the objective as a black-box function, making them vulnerable to staleness amplification when applied to bilevel problems (Theorem~\ref{thm:staleness_amplification}). Bilevel optimization 
methods~\cite{ghadimi2018approximation,ji2021bilevel,tarzanagh2024online,lin2023nonconvex,nazari2025stochastic} handle the nested structure but assume synchronous updates. Gradient transport~\cite{arnold2019reducing} was originally developed for single-level reinforcement learning; \gls{dde} stability analysis~\cite{yu2025role} was applied to distributed stochastic gradient descent. No prior method handles both bilevel structure and delayed feedback. While some single-level delayed methods employ \emph{delay-adaptive} step sizes (scaled by round or cumulative delay), their schedules do not account for inner-solver sensitivity---a factor unique to bilevel objectives. Our algorithm fills this gap by extending gradient transport from single-level policy gradients to bilevel hypergradients and embedding the result in \gls{omd} with \emph{queue-length-adaptive} step sizes calibrated to the bilevel structure.

\begin{table}[t]
\centering
\caption{Comparison with related work. \gls{igt}-\gls{omd} is the first bilevel delayed optimizer.}
\label{tab:related_work}
\footnotesize
\begin{tabular}{@{}llcccl@{}}
\toprule
\textbf{Category} & \textbf{Method} & \textbf{Delay} & \textbf{Bilevel} & \textbf{Adaptive} & \textbf{Open Gap} \\
\midrule
\multirow{3}{*}{\shortstack[l]{Task-driven\\bilevel}} & SPO+~\cite{elmachtoub2022smart} & \xmark & \cmark & \xmark & No delays \\
& Online \gls{dfl}~\cite{capitaine2025online} & \xmark & \cmark & \cmark & No delays \\
& Control-oriented MBRL~\cite{nikishin2022control} & \xmark & \cmark & \xmark & No delays \\
\midrule
\multirow{3}{*}{\shortstack[l]{Single-level\\delayed OCO}} & D-FTRL~\cite{joulani2013online} & \cmark & \xmark & \cmark & Staleness amplification \\
& Robust \gls{omd}~\cite{quanrud2015online} & \cmark & \xmark & \cmark & No bilevel \\
& DORM+~\cite{flaspohler2021online} & \cmark & \xmark & \cmark & Single-level \\
\midrule
\multirow{2}{*}{\shortstack[l]{General Bilevel\\(no delays)}} & Online bilevel~\cite{ji2021bilevel} & \xmark & \cmark & \xmark & No delays \\
& BSA~\cite{ghadimi2018approximation} & \xmark & \cmark & \xmark & Offline \\
\midrule
Grad.\ transport & \gls{igt}~\cite{arnold2019reducing} & \cmark & \xmark & \xmark & Single-level RL \\
\midrule
\gls{dde} analysis & Differential Delay Analysis~\cite{yu2025role} & \cmark & \xmark & \xmark & Distributed SGD \\
\midrule
\textbf{Ours} & \textbf{IGT-OMD} & \cmark & \cmark & \cmark & --- \\
\bottomrule
\end{tabular}
\end{table}

\section{Problem Formulation}
\label{sec:problem}
We now formalize the online bilevel optimization problem under delayed feedback. Our notation is summarized in Table \ref{tab:notation} in Appendix~\ref{app:notation}.

\subsection{Online Bilevel Optimization with Delayed Feedback}
Consider a learner that adjusts a predictive model each round, but outcome feedback arrives only after several rounds. The gradient signal for each update is therefore based on parameters from the past---the delayed bilevel feedback loop we formalize below.

At each interaction with the environment (i.e., round) $t = 1, \ldots, T$, the learner holds outer-level predictor parameters $\theta_t \in \Theta \subseteq \R^p$ and approximately solves the inner problem $w^*(\theta_t) = \argmin_{w \in \mathcal{W} \subseteq \R^q} \mathcal{L}_{\text{model}}(w; \theta_t)$,
where $\mathcal{L}_{\text{model}}$ is the model-based decision objective, $p$ is the predictor dimension, and $q$ is the decision dimension. The approximate solution~$w_t$ is obtained by running $K \in \mathbb{N}_{>0}$ gradient descent steps from the previous solution~$w_{t-1}$, yielding inner-solver error $\norm{w_t - w^*(\theta_t)} \leq \epsilon_{\mathrm{inner}}$. The learner executes~$w_t$, and after $d_t \geq 0$ rounds the environment reveals an outcome object $z_t$ that determines the realized loss. At round~$t$, newly arrived feedback is $A_t = \{s : s+d_s=t\}$, the set of arrived observations is $\mathcal{O}_t = \{(s,z_s) : s + d_s \leq t\}$, and the queue of rounds still awaiting feedback is $Q_t = \{s \leq t : s + d_s > t\}$, with queue length $\sigma_t = |Q_t|$ and envelope $\bar\sigma_t = \max_{r\leq t}\sigma_r$. Using only arrived feedback, the learner updates~$\theta_{t+1}$. The performance measure is \emph{decision regret}:
\begin{equation*}
\mathrm{Regret}_T^{\mathrm{dec}} = \sum_{t=1}^T \mathcal{L}_{\text{true}}(w_t; \theta_t) - \min_{\theta \in \Theta} \sum_{t=1}^T \mathcal{L}_{\text{true}}(w^*(\theta); \theta)
\end{equation*}
which compares the learner's cumulative decision loss to that of the best fixed predictor in hindsight.

\subsection{Hypergradient Computation via Implicit Differentiation}
The hypergradient---the derivative of the decision loss with respect to the predictor parameters---decomposes via the Implicit Function Theorem~\cite{dontchev2014implicit} into an explicit (direct) and an implicit (through the inner solver) component. To avoid the $O(q^3)$ cost of a full Jacobian inversion, we follow~\cite{nikishin2022control} and introduce an adjoint vector $v^{*} \in \R^q$ that satisfies the linear system $H_w\, v^{*} = \nabla_w \mathcal{L}_{\text{true}}(w^{*};\theta),$
where $H_w := \nabla^2_{ww}\mathcal{L}_{\text{model}}(w^*;\theta)$ is the Hessian of the inner objective. The adjoint is computed approximately via Conjugate Gradient  \cite{hestenes1952methods}. The hypergradient is then:
\begin{equation}
\label{eq:hypergradient_adjoint}
g_t(\theta) = \nabla_\theta \mathcal{L}_{\text{true}}\big|_{w\,\text{fixed}} - \bigl[\nabla_\theta\nabla_w \mathcal{L}_{\text{model}}(w^{*};\theta)\bigr]^\top v^{*},
\end{equation}
where the first term captures the direct effect of~$\theta$ on the loss and the second captures the indirect effect through how the inner solver's decision changes with~$\theta$.

\textbf{Re-evaluation under delay.}\; When feedback from a past round $s$ arrives at round $t > s$, the algorithm can re-evaluate the hypergradient at the current parameters~$\theta_t$ using the stored inner solution~$w_s$, observed feedback~$z_s$, and adjoint~$v_s^{*}$ at $O(p\,q)$ cost, without re-solving the inner problem:
\begin{equation}
\label{eq:reevaluation}
g_s(\theta_t) = \nabla_\theta \mathcal{L}_{\text{true}}(w_s; \theta_t)\big|_{w\,\text{fixed}} - \bigl[\nabla_\theta\nabla_w \mathcal{L}_{\text{model}}(w_s; \theta_t)\bigr]^\top v^{*}_s.
\end{equation}
This is an approximation: the cached adjoint was computed at $\theta_s$, so the re-evaluation holds $(w_s,v_s^*)$ fixed and updates only the explicit $\theta_t$-dependence. The resulting residual is controlled in Appendix~\ref{app:proof_thm1}.

\section{IGT-OMD Algorithm}
\label{sec:algorithm}
\gls{igt}-\gls{omd} assembles three pieces, hypergradient transport (IGT), mirror-descent updates (OMD), and a queue-length-adaptive step size, into a single procedure that costs $O(p\,q)$ per outstanding round.

\subsection{Building Blocks}
\label{sec:building_blocks}
\textbf{Implicit Gradient Transport.}\; \gls{igt} \cite{arnold2019reducing} is a mechanism to reevaluate a stale gradient via a telescoping sum. Given a gradient computed at a past iterate $g_s(\theta_s)$, the gradient estimate at the current $\theta_t$ is $g_s^{\mathrm{IGT}}(\theta_t) \;=\; g_s(\theta_s) \;+\; \sum_{k=s}^{t-1}\bigl[g_s(\theta_{k+1}) - g_s(\theta_k)\bigr]$. 
This estimate accumulates the one-step gradient changes along the trajectory from $\theta_s$ to $\theta_t$. For the frozen surrogate $g_s(\cdot)$, the telescope is exact; under $L$-smoothness, the transported path variation satisfies $\sum_{k=s}^{t-1}\bigl\|g_s(\theta_{k+1}) - g_s(\theta_k)\bigr\|^2 \;\leq\; L^2 \sum_{k=s}^{t-1}\|\theta_{k+1} - \theta_k\|^2$,
a \emph{sum of squared per-step changes}. By contrast, the na\"ive stale-gradient surrogate uses $L^2\|\theta_t - \theta_s\|^2$ (squared total displacement), which by the Cauchy--Schwarz inequality~\cite{steele2004cauchy} is up to $\sigma_t$ times larger. See Appendix \ref{app:background} for details.

\textbf{Online Mirror Descent.}\; \Gls{omd}~\cite{nemirovski2009robust} generalizes projected gradient descent via a Bregman divergence~\cite{bregman1967relaxation} $D_\psi(\theta,\theta') = \psi(\theta) - \psi(\theta') - \langle\nabla\psi(\theta'), \theta - \theta'\rangle$ generated by a strongly convex mirror map $\psi$:
\begin{equation}
\label{eq:omd_update}
\theta_{t+1} \;=\; \argmin_{\theta \in \Theta}\; \bigl\langle g_t,\,\theta\bigr\rangle + \frac{1}{\eta_t}\,D_\psi(\theta,\,\theta_t),
\end{equation}
where $\eta_t$ is a step size. When $\psi(\theta)=\tfrac{1}{2}\|\cdot\|^2$ (the squared Euclidean norm) the update recovers projected gradient descent. This method accommodates non-Euclidean geometries and admits $O\big(\sqrt{T}\big)$ regret bounds extending to the delayed setting~\citep{joulani2013online,quanrud2015online}. See Appendix \ref{app:background} for details.

\subsection{IGT-OMD's Per-Round Procedure}
\label{sec:procedure}
For each arrived round selected for transport, our algorithm maintains a replay buffer $\mathcal{B}$ that stores the inner solution $w_s$, adjoint vector $v_s^{*}$, and the most recently transported hypergradient $g_s(\theta_{t-1})$. Algorithm \ref{alg:igt_omd} details our procedure. Each round $t$ proceeds in three phases.

\textbf{Phase 1: Solve and execute (lines 4--9).}\; The learner runs $K$ warm-started gradient-descent steps on $\mathcal{L}_{\text{model}}(\cdot;\theta_t)$ from $w_{t-1}$ to obtain $w_t \approx w^*(\theta_t)$ with $\|w_t - w^*(\theta_t)\| \leq \epsilon_{\mathrm{inner}}$, then executes $w_t$. No true-loss hypergradient for this decision is computed until its feedback arrives.

\textbf{Phase 2: Transport (lines 11--18).}\; When feedback arrives for rounds $s \in A_t$, the algorithm solves the corresponding adjoints, forms the arrival gradient, and re-evaluates earlier arrived gradients in the transport buffer. Aggregating the arrival term with these one-step corrections yields the IGT-corrected hypergradient:
\begin{equation}
\label{eq:igt_omd_gradient}
g_t^{\mathrm{IGT}} \;=\; g_{A_t}(\theta_t) + \sum_{s \in \mathcal{B}_t}\bigl[g_s(\theta_t) - g_s(\theta_{t-1})\bigr],
\end{equation}
which is a causal telescoping update driven only by arrived feedback. Crucially, transport reuses the stored $(w_s, v_s^{*})$ and the inner problem is \emph{not} re-solved during transport.

\textbf{Phase 3: Adaptive update (lines 10 and 20).}\; The learner takes an OMD step~\eqref{eq:omd_update} with the corrected gradient and a queue-envelope-adaptive step size $\eta_t = \eta_0/\sqrt{1 + \beta\,\bar\sigma_t}$ \label{eq:adaptive_step},
where $\eta_0>0$ is the base learning rate and $\beta=\|L_{\mathrm{IGT}}\|/\lambda_{\min}(H_F)$ is the delay-sensitivity ratio (set to $1.0$ when unknown). When the envelope is short, $\eta_t\approx\eta_0$ recovers the non-delayed rate; when long, $\eta_t$ shrinks as $1/\sqrt{\bar\sigma_t}$.
\begin{algorithm}[t]
\caption{IGT-OMD: Implicit Gradient Transport for Bilevel Delayed Optimization}
\label{alg:igt_omd}
\begin{algorithmic}[1]
\footnotesize
\STATE \textbf{Input:} Base step size $\eta_0$, damping $\beta$, inner steps $K$, inner step size $\eta_w$; loss functions $\mathcal{L}_{\text{model}}:\Theta{\times}\mathcal{W}{\to}\mathbb{R}$ (inner), $\mathcal{L}_{\text{true}}:\Theta{\times}\mathcal{W}{\to}\mathbb{R}$ (outer)
\STATE \textbf{Initialize:} $\theta_1 \in \Theta$, $w_0 \in \mathcal{W}$, buffer $\mathcal{B} \leftarrow \emptyset$, queue $Q_0 \leftarrow \emptyset$, envelope $\bar\sigma_0\leftarrow0$
\FOR{$t = 1, \ldots, T$}
    \STATE $w_{t,0} \leftarrow w_{t-1}$ \hfill\textit{// warm-start}
    \FOR{$k = 0, \ldots, K-1$}
        \STATE $w_{t,k+1} \leftarrow w_{t,k} - \eta_w\,\nabla_w \mathcal{L}_{\text{model}}(w_{t,k};\theta_t)$ \hfill\textit{// inner GD}
    \ENDFOR
    \STATE $w_t \leftarrow w_{t,K}$
  \STATE Execute $w_t$; receive arrivals $A_t = \{s : s+d_s=t\}$ and update $Q_t \leftarrow Q_{t-1} \cup \{t\} \setminus A_t$
  \STATE $\sigma_t \leftarrow |Q_t|$; $\bar\sigma_t \leftarrow \max(\bar\sigma_{t-1},\sigma_t)$; $\eta_t \leftarrow \eta_0 / \sqrt{1 + \beta\,\bar\sigma_t}$
  \STATE $g^{\mathrm{IGT}} \leftarrow 0$
  \FOR{$s \in A_t$}
    \STATE Solve $H_w\,v_s^{*} = \nabla_w\mathcal{L}_{\text{true}}(w_s;\theta_t)$ and compute $g_s(\theta_t)$ \hfill\textit{// arrival gradient}
    \STATE $g^{\mathrm{IGT}} \leftarrow g^{\mathrm{IGT}} + g_s(\theta_t)$; store $(w_s,v_s^{*},g_s(\theta_t))$ in $\mathcal{B}$
    \ENDFOR
  \FOR{$s \in \mathcal{B}$}
    \STATE Re-evaluate $g_s(\theta_t)$ and add $g_s(\theta_t) - g_s(\theta_{t-1})$ to $g^{\mathrm{IGT}}$ \hfill\textit{// transport step}
    \STATE Update cached $g_s(\theta_t)$ in $\mathcal{B}$
  \ENDFOR
  \STATE $\theta_{t+1} \leftarrow \theta_t - \eta_t\,g^{\mathrm{IGT}}$ \hfill\textit{// OMD update \eqref{eq:omd_update}}
  \STATE Evict oldest entries if $|\mathcal{B}| > \sigma_{\max}$
\ENDFOR
\STATE \textbf{Return:} $\{\theta_t\}_{t=1}^T$
\end{algorithmic}
\end{algorithm}
\subsection{Cost and Memory Analysis}
\label{sec:cost}

The per-round cost is $O(K p q + |A_t|q^2\kappa_w + |\mathcal{B}_t| p q)$---a $\min(K,\kappa_w)$-factor savings over re-solving each delayed inner problem; buffer memory is $O(\sigma_{\max}(p+2q))$ (Appendix~\ref{app:algorithm}).

\begin{remark}[Optimizer-agnostic transport]
\label{rem:optimizer_agnostic}
The transport correction~\eqref{eq:igt_omd_gradient} depends only on the parameter trajectory $\{\theta_s\}_{s\in Q_t}$, not on the specific update rule. Any optimizer---Adam, SGD, D-FTRL---that supplies iterates to the re-evaluation loop inherits the $\sigma_{\max}$-factor improvement. We validate this observation by evaluating Adam+\gls{igt} in Section~\ref{sec:exp_adam_igt} and D-FTRL+IGT in Appendix~\ref{app:optimizer_agnostic}.
\end{remark}

\section{Theoretical Analysis}
\label{sec:theory}
Our main results follow with their full proofs in Appendix~\ref{app:proof_thm1} - ~\ref{app:proof_prop2}.

\subsection{Assumptions}
\label{sec:assumptions}
We assume seven standard regularity conditions from bilevel optimization~\cite{ghadimi2018approximation,ji2021bilevel} and delayed \gls{oco}~\cite{joulani2013online,quanrud2015online}: (A1--A2)~inner strong convexity and smoothness; (A3)~bounded solver error~$\epsilon_{\mathrm{inner}}$; (A4)~bounded queue length~$\sigma_{\max}$; (A5)~Lipschitz cross-partials (enables hypergradient re-evaluation in~\eqref{eq:reevaluation}); (A6)~bounded hypergradients~$G$; and (A7)~bilevel strong convexity; detailed in Appendix~\ref{app:assumptions}

\newcommand{\assall}{A1--A7}

\subsection{Convergence of IGT-OMD}

Our first question is whether \gls{igt}-\gls{omd} achieves sublinear regret, and whether inner-solver error and delay amplify each other. The theorem shows the coupling is additive, enabling independent control of both error sources.

\begin{theorem}[Bilevel convergence under delay]
\label{thm:bilevel_convergence}
Under Assumptions~\ref{app:ass:inner_convex}--\ref{app:ass:bilevel_convex}, define $\rho_{\mathrm{cpl}}:=L_{w\theta}^2/(\mu_w^2\mu_F)<1$ and $C_\rho=(1-\rho_{\mathrm{cpl}})^{-1}$. Algorithm~\ref{alg:igt_omd} with step size $\eta_t$ attains decision regret:
\begin{equation}
\label{eq:main_bound}
\mathrm{Regret}_T^{\mathrm{dec}} \leq C_\rho\!\left[\frac{2D_\psi\sqrt{1+\beta\sigma_{\max}}}{\eta_0} + \frac{\eta_0\,G^2 T}{2} + 2T\,\epsilon_{\mathrm{inner}}^2 + L_F\!\sum_{t=1}^{T}\sum_{k \in W_t}\norm{\theta_{k+1}-\theta_k}^2\right],
\end{equation}
where $D_\psi$ is the Bregman diameter of the feasible set. Setting $\eta_0 = c/\sqrt{T}$ yields
$\mathrm{Regret}_T^{\mathrm{dec}} = O\!\bigl(\sqrt{T\,\sigma_{\max}} + T\,\epsilon_{\mathrm{inner}}^2\bigr)$, where $c = 2D_\psi^{1/2}(1+\beta\sigma_{\max})^{1/4}/G$ balances the first two terms of~\eqref{eq:main_bound}.
\end{theorem}

Single-level D-FTRL~\citep{joulani2013online} and Robust OMD~\citep{quanrud2015online} attain $O(\sqrt{T\,d_{\mathrm{tot}}})$ regret ($d_{\mathrm{tot}}=\sum_t d_t$ = total delay); we match this with $\sigma_{\max}$ in place of $d_{\mathrm{tot}}$--a strictly tighter measure since $\sigma_{\max}\le d_{\max}\ll d_{\mathrm{tot}}$ is possible~\citep{ryabchenko2026reduction}. Here $W_t=\{k:\max(1,t-\sigma_t)\leq k\leq t-1\}$ is the active transport window. The $T\epsilon_{\mathrm{inner}}^2$ penalty is additive, not multiplicative, so delay and inner-solver precision can be tuned independently.

\subsection{Staleness amplification Lower Bound}

We now ask how much delayed feedback penalizes bilevel optimization relative to single-level: is the $O(\sigma_{\max}^2)$ transport error growth a fundamental barrier, or a proof artifact? This matters to decision quality. The structural source of the barrier is the Implicit Function Theorem (IFT) coupling between the outer and inner problems: the bilevel Lipschitz constant $L_F = L_\theta + L_{w\theta}^2/\mu_w$---where $L_\theta$ = $\|\nabla^2_{\theta}\mathcal{L}_{\mathrm{true}}\|$, $L_{w\theta}$ = $\|\nabla^2_{w\theta}\mathcal{L}_{\mathrm{model}}\|$, and $\mu_w$ is the inner strong-convexity constant---encodes how sensitively $w^*(\theta)$ tracks outer-parameter changes. The coupling constant $C_0 = L_{w\theta}/\mu_w$ measures how inner-solver error propagates into the hypergradient. Theorem~\ref{thm:staleness_amplification} formalizes three consequences of this coupling; part~(b) constructs a hard quadratic instance confirming the lower bound is tight.

\begin{theorem}[Staleness amplification]
\label{thm:staleness_amplification}\label{thm:drift_amplification}
Let $\mathcal{A}$ be any delayed optimizer treating $F(\theta)$ as a black-box convex function using stale bilevel gradients without bilevel-aware correction. Then:

\begin{enumerate}[label=(\alph*), itemjoin=\quad]
    \item The gradient mismatch satisfies $\norm{\hat{g} - \nabla F(\theta_t)} \leq L_F\norm{\theta_t - \theta_{t-d_t}} + C_0\epsilon_{\mathrm{inner}}$.
    
    \item On a hard quadratic bilevel instance, $\mathrm{Regret}_T(\mathcal{A}) \geq \Omega(T\epsilon_{\mathrm{inner}}^2)$, regardless of step-size schedule, delay handling strategy, or queue length adaptation.

    \item Bilevel aware correction reduces the transport error from $O(\sigma_{\max}^2\eta^2G^2)$ to $O(\sigma_{\max}\eta^2G^2)$ per round---a factor~$\sigma_{\max}$ improvement. 
\end{enumerate}
\end{theorem}

Statement (a) of Theorem \ref{thm:staleness_amplification} reveals staleness amplification: staleness couples with the bilevel Lipschitz constant $L_F$ which is absent in single-level OCO. Statement (b) shows the $T\epsilon_{\mathrm{inner}}^2$ term in Theorem~\ref{thm:bilevel_convergence} is tight. Lastly, statement (c) demonstrates that the effect of replacing the squared total drift with a sum of squared per-step changes through our telescoping correction is exactly~$\sigma_t$ by the Cauchy--Schwarz inequality~\cite{steele2004cauchy} (confirmed numerically at $R^2>0.99$ in Section~\ref{sec:exp_sinkhorn}).

\subsection{Inner-Loop Apathy Decomposition}
Our third question is whether the delay error and the inner-solver error interact. The following decomposition shows that they do not, as transport decouples them.

\begin{theorem}[Inner-loop apathy]
\label{thm:inner_loop_apathy}
Under Assumptions~\ref{app:ass:inner_convex}--\ref{app:ass:bilevel_convex}, the decision regret decomposes as:
\begin{equation}
\label{eq:apathy}
\mathrm{Regret}_T^{\mathrm{dec}} = \underbrace{O\!\bigl(\eta_0^2 G^2 T\sigma_{\max}\bigr)}_{\text{delay error}\,(R_1)} + \underbrace{O(T\,\epsilon_{\mathrm{inner}}^2)}_{\text{inner bias}\,(R_2)} + \underbrace{O\!\Bigl(C_{\mathrm{ap}}\!\sum_{t=1}^T\sum_{k\in W_t}\norm{\theta_{k+1}-\theta_k}^2\Bigr)}_{\text{interaction}\,(R_3)}.
\end{equation}
The interaction $R_3$ uses per-step squared changes ($\sigma_t\eta^2G^2$) instead of squared total drift ($\sigma_t^2\eta^2G^2$)--a $\sigma_t$ factor improvement; here $C_{\mathrm{ap}}$ is an explicit bounded constant, so $\epsilon_{\mathrm{inner}}$ remains additive rather than multiplied by delay.
\end{theorem}

This theorem implies that the optimizer becomes insensitive to the inner solver precision, suggesting ``inner-loop apathy.'' Moreover, \gls{igt} decouples the outer staleness from the inner-solver quality, with $\epsilon_{\mathrm{inner}}$ entering additively.

\subsection{DDE Stability and Discrete Consistency}
\label{sec:dde_stability}
The adaptive step size $\eta_t$ is not heuristic: it arises from a stability analysis of the continuous-time limit. Taking $\eta\!\to\!0$ in a bounded active-window embedding, the IGT-OMD iterates~\eqref{eq:omd_update} converge to a \gls{dde} whose characteristic roots determine a sufficient stability certificate.
\begin{proposition}[\gls{dde} stability region]
\label{prop:dde_stability}
Suppose $\sigma_{\max} < 1/\beta$ and $\theta^* \in \mathrm{relint}(\Theta)$. Consider the continuous-time limit of IGT-OMD,\, $\dot{\theta}(t) = -\eta(t)\,g^{\mathrm{IGT}}\bigl(\theta(t),\,\{\theta(t-\tau_s)\}_{s\in Q(t)}\bigr)$, with $\eta(t)=\eta_0/\sqrt{1+\beta\bar\sigma(t)}$. Linearizing around $\theta^{*}$, the system is asymptotically stable whenever $\eta_0 < \min\{1/(\|H_F\|\sqrt{1+\beta\sigma_{\max}}),\sqrt{1+\beta\sigma_{\max}}/(\sigma_{\max}\sqrt{\|L_{\mathrm{IGT}}\|})\}$,
where $\|H_F\|$ denotes its spectral norm (equal to $\lambda_{\max}(H_F)$; $H_F \succ 0$). This is a sufficient local certificate for the bounded active-window embedding used in the analysis.
\end{proposition}

This result provides \emph{local} asymptotic stability for the linearized, homogeneous system ($\epsilon_{\mathrm{inner}}=0$). In the analyzed embedding, the certificate is indexed by queue length $\sigma_{\max}$ rather than raw delay, matching the constant $\eta_{\max}=0.093$ observed across $\sigma \in \{1,\ldots,100\}$ in Table~\ref{tab:lqr}.


\begin{proposition}[Discrete-continuous consistency]
\label{prop:discrete_consistency}
Let $\theta^{\mathrm{disc}}_k$ denote IGT-OMD iterates and $\theta^{\mathrm{cont}}(t)$ the solution to the \gls{dde} in Proposition~\ref{prop:dde_stability} with the same initial history. Under the assumptions of Proposition~\ref{prop:dde_stability} and a uniform step bound $\eta_t \leq \eta_{\max}$, the discrete trajectory tracks the continuous solution:
\(
\sup_{k\,\leq\,T}\,\bigl\|\theta^{\mathrm{disc}}_k - \theta^{\mathrm{cont}}(t_k)\bigr\| \;\leq\; C\,\eta_{\max},
\)
with $t_k = \sum_{s\leq k}\eta_s$ and a constant $C$ depending on $T$, $\norm{H_F}$, $\norm{L_{\mathrm{IGT}}}$, and $\sigma_{\max}$ but not on the discretization. Consequently, the discrete iterates inherit the asymptotic stability and contraction rate of the \gls{dde} up to $O(\eta_{\max})$ error.
\end{proposition}


Proposition \ref{prop:discrete_consistency} follows because Algorithm~\ref{alg:igt_omd} is a forward-Euler discretization of the \gls{dde}. This result shows the stability region and contraction rate from the continuous-time system--particularly the $\eta_{\max}$-dependent step size calibration of~\eqref{eq:adaptive_step}--transfer to discrete iterates in practice, with error vanishing in $\eta_{\max}$.

\begin{remark}[Interior equilibrium]
\label{rem:interior}
Both propositions require $\theta^{*}{\in}\mathrm{relint}(\Theta)$ for \gls{dde} linearization; the regret bounds (Theorems~\ref{thm:bilevel_convergence}--\ref{thm:inner_loop_apathy}) hold without this restriction. The boundary case is handled by the OMD projection in~\eqref{eq:omd_update}.
\end{remark}

\section{Experiments}
\label{sec:experiments}
We evaluate \gls{igt}-\gls{omd} across four environments that jointly test our theoretical claims: (i) LQR for stability under delay, (ii) Warcraft for staleness amplification, (iii) Sinkhorn OT for the $\sigma_{\max}$-factor transport error scaling, and (iv) a controlled experiment isolating the transport contribution with optimizer choice held fixed. Our baselines span a $2{\times}2$ matrix (bilevel-aware vs.\ single-level $\times$ delay-aware vs.\ delay-unaware): 2-Stage~(MSE), SPO+~\cite{elmachtoub2022smart}, D-FTRL~\cite{joulani2013online}, Robust~\gls{omd}~\cite{quanrud2015online}, and Stale~\gls{omd} \cite{nemirovski2009robust}. Additional details on our experiments and additional analyses are in Appendices \ref{app:exp_details} and \ref{app:additional}. All experiments use 5--10 seeds; error bars report $\pm$1 standard deviation.

\subsection{Linear Quadratic Regulator Stability Boundary}
\label{sec:exp_lqr}
\textbf{Setup.}\; An LQR provides a clean bilevel stability testbed: the inner problem computes a linear gain for a learned dynamics model, with the inner configuration held fixed to isolate delay handling. We use true dynamics $x_{t+1}=A_{\rm true}x_t+B_{\rm true}u_t+\xi_t$ and learned parameters $\theta=(\hat A,\hat B)\in\R^{10\times13}$. Constant delays $d\in\{1,10,20,40,60,80,100\}$ span both the stable regime ($d\leq40$, all methods converge) and the phase-transition region predicted by Proposition~\ref{prop:dde_stability} ($d\geq60$, bilevel-unaware methods degrade); beyond $d=100$ baselines diverge entirely. The maximum stable learning rate $\eta_{\max}$ is found via binary search.

\textbf{Results.}\; Table~\ref{tab:lqr} reveals a two-regime structure: \gls{igt}-\gls{omd} maintains constant $\eta_{\max}=0.093$ across all $\sigma\leq100$ (Proposition~\ref{prop:dde_stability}), while bilevel-unaware methods (2-Stage, SPO+) degrade to $0.010$ at $\sigma=100$ ($89\%$ reduction). Single-level delay-aware methods degrade more slowly but still fall to $0.052$ ($44\%$). \gls{igt}-\gls{omd}'s achieves $9.3\times$ and $1.8\times$ gains over 2-Stage and D-FTRL respectively.

\begin{remark}[Constant delay as adversarial worst case]
\label{rem:constant_delay}
Constant $d_t=d$ gives $\sigma_{\max}=d$ and maximizes $\sum_t\sigma_t$ among all delay patterns with the same total feedback budget. Hence, our constant-delay experiments are adversarial; variable delays can only improve on these bounds. We additionally validate under uniform delays in Appendix~\ref{app:adversarial_delay}.
\end{remark}


\begin{table}[t]
\centering
\caption{LQR stability boundary: maximum stable learning rate $\eta_{\max}$ vs.\ queue length $\sigma$. \gls{igt}-\gls{omd} maintains constant $\eta_{\max}$ across all delays; bilevel-unaware methods (2-Stage, SPO+) degrade earliest.}
\label{tab:lqr}
\footnotesize
\begin{tabular}{@{}llccccccc@{}}
\toprule
& \textbf{Algorithm} & $\sigma{=}1$ & $\sigma{=}10$ & $\sigma{=}20$ & $\sigma{=}40$ & $\sigma{=}60$ & $\sigma{=}80$ & $\sigma{=}100$ \\
\midrule
\multirow{3}{*}{\rotatebox[origin=c]{90}{\scriptsize Bilevel}} 
& \textbf{\gls{igt}-\gls{omd} (ours)} & \textbf{0.093} & \textbf{0.093} & \textbf{0.093} & \textbf{0.093} & \textbf{0.093} & \textbf{0.093} & \textbf{0.093} \\
& 2-Stage (MSE) & 0.093 & 0.075 & 0.039 & 0.022 & 0.015 & 0.012 & 0.010 \\
& SPO+ & 0.093 & 0.081 & 0.039 & 0.021 & 0.013 & 0.011 & 0.009 \\
\midrule
\multirow{3}{*}{\rotatebox[origin=c]{90}{\scriptsize Single}} 
& D-FTRL & 0.093 & 0.093 & 0.093 & 0.093 & 0.065 & 0.060 & 0.052 \\
& Robust OMD & 0.093 & 0.093 & 0.093 & 0.093 & 0.075 & 0.060 & 0.056 \\
& Stale OMD & 0.093 & 0.093 & 0.093 & 0.093 & 0.070 & 0.056 & 0.052 \\
\bottomrule
\end{tabular}
\end{table}
\subsection{Warcraft Shortest Path: A structural contrast}
\label{sec:exp_warcraft}

Having established stability in a linear setting, we turn to a combinatorial optimization benchmark to test whether staleness amplification degrades the quality of real decisions.

\textbf{Setup.}\; We study shortest-path planning on $12{\times}12$ grids from Warcraft~II maps with four terrain types. The outer parameters $\theta \in \R^{128}$ index the trainable hidden-layer column of a 2-layer neural network predicting per-cell traversal costs, and outer gradients are computed with the differentiable perturbation surrogate of~\citet{vlastelica2019differentiation}. The inner problem runs Dijkstra's algorithm to find the shortest path---an exact solver, so $\epsilon_{\mathrm{inner}}=0$. The decision loss evaluates the chosen path under true terrain costs; the \emph{optimality gap} measures the excess cost over the oracle shortest path. We test constant delays $d\in\{0,10,50,100\}$ and stochastic Poisson delays with mean $\lambda\in\{10,25,50,100\}$ under OU edge-cost drift rate $0.05$. We use 10 seeds over $T=5{,}000$ rounds.

\textbf{Results.}\; Table~\ref{tab:warcraft_main} reveals three structural findings. First, \gls{igt}-\gls{omd} achieves the lowest optimality gap across all delay configurations: at $d{=}100$, gap $1.56$ vs.\ $1.83$ for D-FTRL ($14.8\%$ reduction) and $2.42$ for 2-Stage ($35.5\%$ reduction), confirming staleness amplification (Theorem~\ref{thm:staleness_amplification}(a)). Second, methods without bilevel awareness (2-Stage, SPO+) degrade most severely ($1.6$--$3.4\times$ gap relative to \gls{igt}-\gls{omd}), showing that ignoring the inner solver's sensitivity has dramatic consequences under delay. Third, because the inner solver is exact ($\epsilon_{\mathrm{inner}}=0$), the inner-loop apathy benefit (Theorem~\ref{thm:inner_loop_apathy}) is inoperative here. The reduction over delay-aware methods ($12$--$21\%$) is correspondingly smaller than on Sinkhorn, where $\epsilon_{\mathrm{inner}}>0$---a structural prediction of the theory.

\begin{table}[t]
\centering
\caption{Warcraft shortest path: optimality gap (mean over last 200 rounds, $\pm$1\,s.d.) vs.\ delay configuration. \gls{igt}-\gls{omd} achieves the smallest gap across all settings.}
\label{tab:warcraft_main}
\footnotesize
\begin{tabular}{@{}llcccc@{}}
\toprule
& & \multicolumn{2}{c}{\textbf{Constant Delay}} & \multicolumn{2}{c}{\textbf{Poisson Delay}} \\
\cmidrule(lr){3-4} \cmidrule(lr){5-6}
& \textbf{Algorithm} & $d=50$ & $d=100$ & $\lambda=50$ & $\lambda=100$ \\
\midrule
\multirow{3}{*}{\rotatebox[origin=c]{90}{\scriptsize Bilevel}}
& \textbf{\gls{igt}-\gls{omd} (ours)} & $\mathbf{1.52{\scriptstyle\pm 0.21}}$ & $\mathbf{1.56{\scriptstyle\pm 0.27}}$ & $\mathbf{1.70{\scriptstyle\pm 0.18}}$ & $\mathbf{1.53{\scriptstyle\pm 0.16}}$ \\
& 2-Stage (MSE) & $2.17{\scriptstyle\pm 0.42}$ & $2.42{\scriptstyle\pm 0.42}$ & $2.46{\scriptstyle\pm 0.44}$ & $2.39{\scriptstyle\pm 0.43}$ \\
& SPO+ & $4.10{\scriptstyle\pm 1.76}$ & $4.11{\scriptstyle\pm 1.94}$ & $5.76{\scriptstyle\pm 0.99}$ & $4.26{\scriptstyle\pm 1.58}$ \\
\midrule
\multirow{3}{*}{\rotatebox[origin=c]{90}{\scriptsize Single}}
& D-FTRL & $1.97{\scriptstyle\pm 0.50}$ & $1.83{\scriptstyle\pm 0.15}$ & $1.91{\scriptstyle\pm 0.31}$ & $1.94{\scriptstyle\pm 0.18}$ \\
& Robust OMD & $1.86{\scriptstyle\pm 0.34}$ & $1.77{\scriptstyle\pm 0.15}$ & $1.85{\scriptstyle\pm 0.28}$ & $1.81{\scriptstyle\pm 0.39}$ \\
& Stale OMD & $1.82{\scriptstyle\pm 0.33}$ & $1.77{\scriptstyle\pm 0.23}$ & $1.89{\scriptstyle\pm 0.42}$ & $1.80{\scriptstyle\pm 0.33}$ \\
\bottomrule
\end{tabular}
\end{table}
\subsection{Transport Error Scaling (Sinkhorn Optimal Transport)}
\label{sec:exp_sinkhorn}
Warcraft uses an exact inner solver, so the inner-solver error channel ($R_2$ in Theorem~\ref{thm:inner_loop_apathy}) is zero. To test the full $R_1{+}R_2{+}R_3$ decomposition, we need an environment where $\epsilon_{\mathrm{inner}} > 0$.

\textbf{Setup.}\; We use a Sinkhorn (OT) task---computing a minimum-cost coupling between two discrete distributions via differentiable entropic regularization ($n=10$, $K=10$ inner Sinkhorn iterations, OU drift)---with constant delays $d\in\{1,2,5,10,20,50\}$, $T=1{,}000$, 5~seeds, to test the transport error scaling claims of Theorems~\ref{thm:staleness_amplification}(c) and~\ref{thm:inner_loop_apathy}. We compute two error surrogates: $R_{\mathrm{sq}} = \sum_t\|\theta_t - \theta_{t-d}\|^2$ (squared total drift, the quantity na\"ive methods pay) and $R_3 = \sum_t\sum_{s\in Q_t}\|\theta_{s+1}-\theta_s\|^2$ (sum of per-step squared changes, the quantity IGT pays).

\textbf{Results.}\; The ratio $R_{\mathrm{sq}}/R_3$ tracks $\sigma_{\max}$ with near-perfect fidelity across all four algorithms: ratio ${\approx}9.99$ at $d=10$, ${\approx}19.9$ at $d=20$, ${\approx}49.3$ at $d=50$ (Table~\ref{tab:sinkhorn_transport} in Appendix~\ref{app:sinkhorn_scaling}). Log-log regression in Figure~\ref{fig:transport_scaling} confirms the theoretical slopes: $R_3 \propto \sigma^{0.99}$ and $R_{\mathrm{sq}} \propto \sigma^{1.99}$ (both $R^2 > 0.99$). This consistent ratio is not an algorithmic artifact---it holds identically for \gls{igt}-\gls{omd}, D-FTRL, Robust~\gls{omd}, and Stale~\gls{omd}---because the $\sigma_{\max}$-factor improvement follows from a \emph{geometric} identity (Cauchy--Schwarz~\cite{steele2004cauchy}) of the parameter trajectory.


\begin{figure}[htbp]
\centering
\begin{minipage}[t]{0.45\textwidth}
\centering
\includegraphics[width=\textwidth]{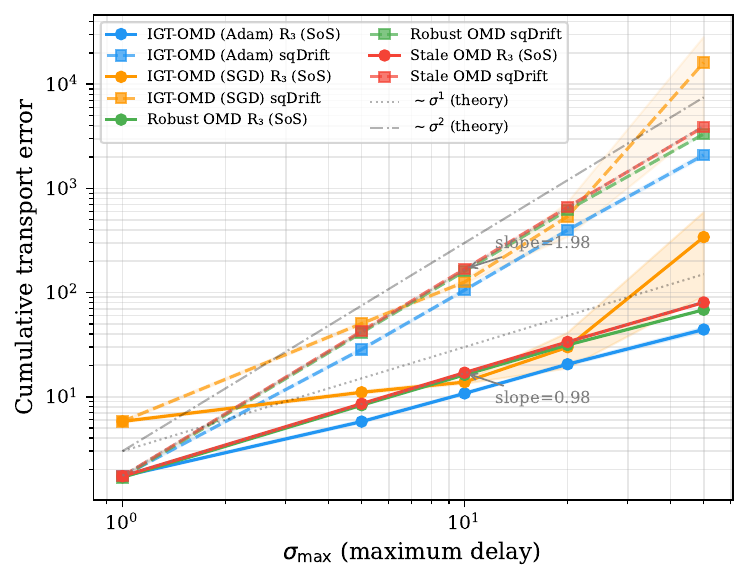}
\end{minipage}\hfill
\begin{minipage}[t]{0.45\textwidth}
\centering
\includegraphics[width=\textwidth]{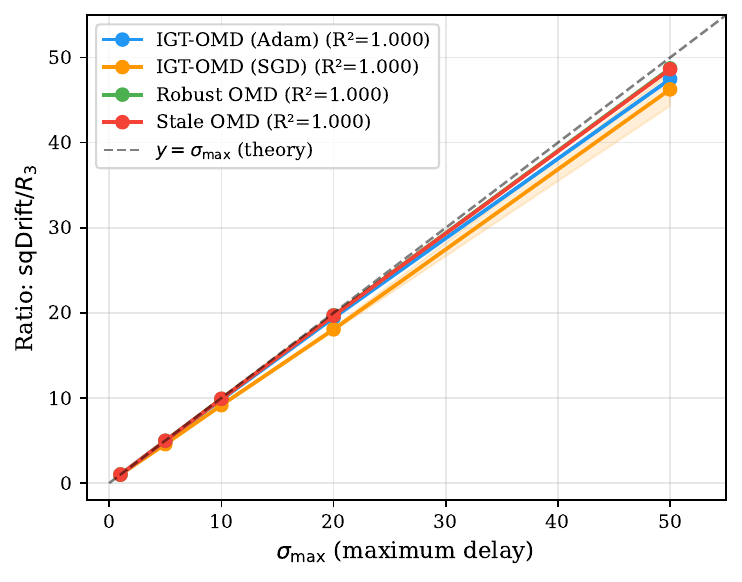}
\end{minipage}
\caption{\textbf{Transport error scaling validates Theorem~\ref{thm:staleness_amplification}(c).} \emph{Left:} Log-log regression of cumulative $R_3$ (slope~${\approx}1$) and $R_{\mathrm{sq}}$ (slope~${\approx}2$) against $\sigma_{\max}$, confirming linear vs.\ quadratic scaling. \emph{Right:} The ratio $R_{\mathrm{sq}}/R_3$ tracks $\sigma_{\max}$ with near-perfect agreement across all algorithms.}
\label{fig:transport_scaling}
\end{figure}

\subsection{Isolating Transport Contribution with Adam}
\label{sec:exp_adam_igt}
Algorithm~\ref{alg:igt_omd} with SGD does not outperform Adam baselines (i.e., Robust~\gls{omd}, and Stale~\gls{omd}) on the Sinkhorn OT non-convex task---the transport corrections are exact, but the base optimizer's convergence properties also matter. To disentangle these effects, we hold the Adam optimizer \cite{kingma2015adam} fixed in this experiment.

\textbf{Setup.}\; To isolate \gls{igt} transport from optimizer choice, both treatment (\gls{igt}-\gls{omd}) and control (Stale~\gls{omd}) use identical Adam ($\eta=10^{-3}$, $\beta_1=0.9$, $\beta_2=0.999$). We refer to the modification of \gls{igt}-\gls{omd} with Adam as Adam+\gls{igt}. Our experiment uses the same Sinkhorn environment with $d\in\{1,5,10,20,50\}$, $T=1{,}000$, and 5 seeds.

\textbf{Results.}\; Table~\ref{tab:adam_igt} shows a monotonically growing transport benefit: $0.0\%$ at $d=1$ (by construction), $+4.6\%$ at $d=5$, $+7.9\%$ at $d=20$ ($p<0.001$), $+9.5\%$ at $d=50$ ($p<0.0001$). The $d=1$ zero is a designed negative control: when $|Q_t|=1$, the sum-of-squares and squared-total-drift are identical, so \gls{igt} is vacuous. The monotonic growth is the signature of Theorem~\ref{thm:staleness_amplification}(c). Adam+\gls{igt}'s regret remains statistically flat across delays, while Stale~\gls{omd}'s regret monotonically increases.

\textbf{$2{\times}2$ factorial.}\; In a $2{\times}2$ factorial over optimizer and transport, we find cumulative regret grows modestly without transport ($+7.3\%$ for SGD and $+8.4\%$ for Adam as delay increased from $d=1$ to $d=50$). However, the interaction is non-additive: IGT+SGD \emph{worsens} degradation to $+28.8\%$, while IGT+Adam \emph{eliminates} it ($-1.8\%$, $p=0.29$). Transport corrections and adaptive scaling are synergistic---neither alone achieves zero degradation.

\begin{remark}[Trajectory stability as hidden requirement]
\label{rem:trajectory_stability}
Adam's per-coordinate scaling constrains step magnitudes, ensuring 
transported gradients $g_s(\theta_t)$ remain accurate corrections. 
SGD lacks this preconditioning, so transport introduces noise in 
high-curvature directions. See Appendix~\ref{app:optimizer_agnostic}.
\end{remark}

\textbf{D-FTRL+IGT.}\; Repeating with D-FTRL as the base optimizer yields a near-identical improvement curve ($0.0\%$ at $d=1$ to $+9.8\%$ at $d=50$; Table~\ref{tab:dftrl_igt} in Appendix~\ref{app:optimizer_agnostic}), confirming the optimizer-agnostic transport benefit (Remark~\ref{rem:optimizer_agnostic}).


\begin{table}[t]
\centering
\caption{Sinkhorn OT: cumulative regret ($T=1{,}000$, 5~seeds). Adam+\gls{igt} vs.\ Stale~\gls{omd} isolates the transport contribution; improvement grows monotonically with delay ($p{<}0.01$ for $d{\geq}10$).}
\label{tab:adam_igt}
\footnotesize
\begin{tabular}{@{}lccccc@{}}
\toprule
\textbf{Algorithm} & $d=1$ & $d=5$ & $d=10$ & $d=20$ & $d=50$ \\
\midrule
\textbf{Adam+\gls{igt} (ours)} & $579{\scriptstyle\pm 15}$ & $\mathbf{557{\scriptstyle\pm 14}}$ & $\mathbf{551{\scriptstyle\pm 12}}$ & $\mathbf{553{\scriptstyle\pm 10}}$ & $\mathbf{568{\scriptstyle\pm 12}}$ \\
Stale OMD (Adam) & $579{\scriptstyle\pm 15}$ & $583{\scriptstyle\pm 14}$ & $589{\scriptstyle\pm 14}$ & $601{\scriptstyle\pm 14}$ & $628{\scriptstyle\pm 13}$ \\
Robust OMD (Adam) & $581{\scriptstyle\pm 15}$ & $586{\scriptstyle\pm 16}$ & $594{\scriptstyle\pm 16}$ & $606{\scriptstyle\pm 14}$ & $634{\scriptstyle\pm 15}$ \\
\midrule
\textbf{Improvement (\%)} & $0.0\%$ & $+4.6\%$ & $+6.5\%$ & $+7.9\%$ & $+9.5\%$ \\
$p$-value & -- & $0.029$ & $0.004$ & ${<}0.001$ & ${<}0.001$ \\
\bottomrule
\end{tabular}
\end{table}

\textbf{Summary.}\; A unified table mapping results to theorems is in Appendix~\ref{app:exp_details}; Appendix~\ref{app:mpc} reports a non-convex HalfCheetah MPC stress test reserved for the non-convex extension. Across assumption-aligned benchmarks, the results confirm each theoretical claim: constant stability across delays (Proposition~\ref{prop:dde_stability}), staleness amplification in bilevel-unaware methods (Theorem~\ref{thm:staleness_amplification}(a)), the $\sigma_{\max}$-factor improvement as a geometric identity (Theorem~\ref{thm:staleness_amplification}(c)), and additive decoupling of inner-solver error from delay (Theorem~\ref{thm:inner_loop_apathy}).

\section{Conclusion}
\label{sec:conclusion}
We introduced \gls{igt}-\gls{omd}, the first delayed optimizer for bilevel predict-then-optimize pipelines. Our analysis identifies staleness amplification---the coupling between outer-parameter staleness and inner-solver sensitivity---as a failure mode unique to bilevel delay (Theorem~\ref{thm:staleness_amplification}), establishes a matching lower bound, and shows IGT reduces transport error by a factor $\sigma_{\max}$. Adam+\gls{igt} achieves $9.5\%$ lower regret than Stale~\gls{omd} at $d{=}50$ ($p{<}0.001$), with the improvement growing monotonically with delay, as predicted, and vanishing at $d{=}1$ as a designed negative control. A $2{\times}2$ factorial confirms transport and adaptive scaling are synergistic---only their combination eliminates delay degradation. Our broader impacts are summarized in Appendix \ref{app:broader_impacts}.

\textbf{Limitations and Future Work.}\; Strong convexity of $\mathcal{L}_{\text{model}}$ (Assumption~\ref{app:ass:inner_convex}) may not hold for neural predictors; a formal non-convex analysis remains open. Our $2{\times}2$ factorial analysis shows that transport corrections require trajectory stability (Adam) and can \emph{amplify} degradation without it (SGD)---an interaction not captured by the current theory. Per-round cost $O(K\,p\,q + \sigma_{\max}\,p\,q + q^2\kappa_w)$ may be prohibitive for large $q$; low-rank Hessian approximations could help. Natural extensions include non-convex inner objectives via fixed-point implicit differentiation applications, applications in reinforcement learning with delayed rewards, and federated settings with communication delays.

\bibliographystyle{unsrtnat}
\bibliography{references}

\clearpage
\appendix

\section{Notation Table}
\label{app:notation}

\begingroup
\setlength{\tabcolsep}{4pt}
\renewcommand{\arraystretch}{0.72}
\small
\begin{longtable}{@{}cl@{}}

\caption{Principal notation.}
\label{tab:notation} \\

\toprule
\textbf{Symbol} & \textbf{Definition} \\
\midrule
\endfirsthead

\multicolumn{2}{l}{\footnotesize\textit{Table~\ref*{tab:notation} continued.}} \\[2pt]
\toprule
\textbf{Symbol} & \textbf{Definition} \\
\midrule
\endhead

\midrule
\multicolumn{2}{r}{\footnotesize\textit{Continued on next page}} \\
\endfoot

\bottomrule
\endlastfoot

\rowcolor{sectionbg}
\multicolumn{2}{l}{\textit{Core variables and delay}} \\
$\theta_t \in \Theta \subseteq \mathbb{R}^p$
  & Outer-level predictor parameters at round $t$ \\
$w_t \in \mathcal{W} \subseteq \mathbb{R}^q$
  & Approximate inner-level decision at round $t$ \\
$w^*\!:\Theta \to \mathcal{W}$
  & Inner minimizer:
    $w^*(\theta) = \argmin_{w \in \mathcal{W}} \mathcal{L}_{\text{model}}(w;\theta)$ \\
$\mathcal{L}_{\text{model}}\!:\mathcal{W}\!\times\!\Theta \to \mathbb{R}$
  & Model-based decision objective (inner loss) \\
$\mathcal{L}_{\text{true}}\!:\mathcal{W}\!\times\!\Theta \to \mathbb{R}$
  & Environment-evaluated decision loss (outer loss) \\
$d_t \in \mathbb{N}_0$
  & Feedback delay for round $t$ \\
$Q_t \subseteq [t]$
  & Queue of outstanding rounds:
    $Q_t = \{s \leq t : s + d_s > t\}$ \\
$\sigma_t = |Q_t|$
  & Queue length at time $t$ \\
$\sigma_{\max} = \max_{t}\,\sigma_t$
  & Maximum queue length over all rounds \\
$\epsilon_{\mathrm{inner}} \geq 0$
  & Inner-solver error:
    $\|w_t - w^*(\theta_t)\| \leq \epsilon_{\mathrm{inner}}$ \\
$g_t\!:\Theta \to \mathbb{R}^p$
  & Hypergradient:
    $g_t(\theta) = \tfrac{d}{d\theta}
    \mathcal{L}_{\text{true}}(w^*(\theta);\theta)$ \\
$D_\psi\!:\Theta\!\times\!\Theta \to \mathbb{R}_{\geq 0}$
  & $D_\psi(\theta,\theta') = \psi(\theta) - \psi(\theta')
    - \langle\nabla\psi(\theta'),\,\theta\!-\!\theta'\rangle$ \\
$\psi:\Theta\to\mathbb{R}$
  & Strongly convex mirror map generating $D_\psi$ \\
$F\!:\Theta\to\mathbb{R}$
  & Bilevel objective:
    $F(\theta)=\mathcal{L}_{\text{true}}(w^*(\theta);\theta)$ \\

\rowcolor{sectionbg}
\multicolumn{2}{l}{\textit{Dimensions and horizon}} \\
$T$
  & Total number of rounds (horizon) \\
$t\in\{1,\ldots,T\}$
  & Round index \\
$p$
  & Dimension of outer parameters $\theta_t$ \\
$q$
  & Dimension of inner decisions $w_t$ \\
$K$
  & Number of inner gradient-descent steps per round \\
$\rho$
  & Inner-solver contraction factor, e.g. $\epsilon_{\mathrm{inner}}=O(\rho^K)$ with $\rho<1$ \\

\rowcolor{sectionbg}
\multicolumn{2}{l}{\textit{Smoothness and convexity constants}} \\
$L_w$
  & Smoothness constant of $\mathcal{L}_{\text{model}}$
    w.r.t.\ $w$ (Assumption~A1) \\
$\mu_w$
  & Strong-convexity constant of $\mathcal{L}_{\text{model}}$
    w.r.t.\ $w$ (Assumption~A1) \\
$L_\theta$
  & Lipschitz constant of $\nabla_\theta\mathcal{L}_{\text{true}}$
    w.r.t.\ $\theta$ (Assumption~A2) \\
$L_{w\theta}$
  & Bound on cross-partial
    $\|\nabla^2_{w\theta}\mathcal{L}_{\text{model}}\|$
    (Assumptions~A2,~A5) \\
$\mu_F$
  & Strong-convexity constant of bilevel objective $F$
    (Assumption~A7) \\
$L_F$
  & Bilevel Lipschitz constant:
    $L_F = L_\theta + L_{w\theta}^2/\mu_w$ \\
$G$
  & Uniform bound on corrected hypergradients:
    $\|g_t^{\mathrm{IGT}}\|\leq G$ (Assumption~A6) \\
$G_{\mathrm{true}}$
  & Exact hypergradient bound:
    $G_{\mathrm{true}} = \sup_\theta\|\nabla F(\theta)\|$
    (Assumption~A6) \\
$C_0$
  & Cross-objective sensitivity: $C_0 = L_{w\theta}/\mu_w$
    (Theorem~\ref{thm:staleness_amplification}(a)) \\
$\rho_{\mathrm{cpl}}, C_\rho$
  & Coupling ratio $\rho_{\mathrm{cpl}}=L_{w\theta}^2/(\mu_w^2\mu_F)<1$ and multiplier $C_\rho=(1-\rho_{\mathrm{cpl}})^{-1}$ \\
$C_1,\kappa$
  & Proof-local sensitivity constants in Theorem~\ref{thm:inner_loop_apathy}: $C_1=L_{\theta w}/\mu_w$, $\kappa=L_gC_1$ \\
$\kappa_w$
  & Inner-problem condition number: $\kappa_w = L_w/\mu_w$ \\
  $\delta$ & CG solver tolerance (adjoint solve cost $O(q^2\kappa_w\ln(1/\delta))$) \\

\rowcolor{sectionbg}
\multicolumn{2}{l}{\textit{Algorithm-specific quantities}} \\
$\eta_0$
  & Base (initial) step size \\
$\eta_t$
  & Adaptive step size at round $t$:
    $\eta_t = \eta_0/\sqrt{1+\beta\bar\sigma_t}$ \\
$\eta_w$
  & Inner step size
    (Algorithm~\ref{alg:igt_omd} line~6) \\
$\beta$
  & Delay sensitivity ratio:
    $\beta = \|L_{\mathrm{IGT}}\|/\lambda_{\min}(H_F)$ \\
$H_w$
  & Inner Hessian:
    $H_w = \nabla^2_{ww}
    \mathcal{L}_{\text{model}}(w^*;\theta)$ \\
$v_t^*\in\mathbb{R}^q$
  & Adjoint vector:
    $H_w v_t^* =
    \nabla_w\mathcal{L}_{\text{true}}(w_t;\theta_t)$ \\
$g_t^{\mathrm{IGT}}\in\mathbb{R}^p$
  & IGT-corrected hypergradient
    (Eq.~\eqref{eq:igt_omd_gradient}) \\
$\mathcal{B}$
  & Replay buffer storing
    $(w_s, v_s^*, g_s(\theta_{t-1}))$ for $s\in Q_t$ \\
$H_F$
  & Hessian of bilevel objective $F$ at equilibrium \\
$L_{\mathrm{IGT}}$
  & IGT delay coupling matrix (linearized bilevel dynamics) \\

\rowcolor{sectionbg}
\multicolumn{2}{l}{\textit{Regret decomposition (Theorem~3)}} \\
$R_1$
  & Delay error:
    $O(\eta_0^2 G^2 T\sigma_{\max})$ \\
$R_2$
  & Inner-solver bias:
    $O(T\epsilon_{\mathrm{inner}}^2)$ \\
$R_3$
  & Interaction (transport residual):
    $O\!\left(\sum_t\sum_{s\in Q_t}
    \|\theta_{s+1}-\theta_s\|^2\right)$ \\
$C_{\mathrm{ap}}$
  & Interaction prefactor in Theorem~\ref{thm:inner_loop_apathy}; see Remark~\ref{rem:cap_honest} \\
$R_{\mathrm{sq}}$ & Squared total drift (experimental metric): $R_{\mathrm{sq}} = \sum_t\|\theta_t - \theta_{t-d}\|^2$ \\

\rowcolor{sectionbg}
\multicolumn{2}{l}{\textit{\gls{dde} / stability analysis (Propositions~\ref{prop:dde_stability}--\ref{prop:discrete_consistency})}} \\
$F(\theta)$ & Bilevel objective: $F(\theta) = \mathcal{L}_{\text{true}}(w^*(\theta);\theta)$ \\
$\gamma$ & Continuous-time convergence rate: $\gamma = c\,\eta_0/\sqrt{1+\beta\sigma_{\max}}$ \\
$\bar{\tau}$ & Representative delay time in the continuous \gls{dde} ($\bar{\tau} \leq \sigma\eta$) \\
$\xi(t)$ & Deviation from equilibrium: $\xi(t) = \theta(t) - \theta^*$ \\
$\kappa_F$ & Condition number of $H_F$: $\kappa_F = \|H_F\|/\mu_{\min}$ \\

\rowcolor{sectionbg}
\multicolumn{2}{l}{\textit{Abbreviations}} \\
\multicolumn{2}{l}{%
  DFL (decision-focused learning),
  OCO (online convex optimization),
  IGT (Implicit Gradient Transport),} \\
\multicolumn{2}{l}{%
  OMD (Online Mirror Descent),
  DDE (Delay Differential Equation)} \\

\end{longtable}
\endgroup

\section{Background}
\label{app:background}

The following subsections collect established results from the literature that underpin our theoretical development. All material presented here is drawn directly from the cited works; we include it to make the paper self-contained for readers who may be unfamiliar with the relevant background, and make no claim of originality for any result in this background section.

\subsection{Implicit Gradient Transport (IGT)}
\label{app:background:igt}

IGT was introduced by~\cite{arnold2019reducing} to correct for gradient 
staleness in online learning. Given a gradient $g_s = \nabla f_s(\theta_s)$ 
computed at a past iterate $\theta_s$, IGT constructs a corrected estimate 
at the current iterate $\theta_t$ via a telescoping sequence of $O(p)$ 
re-evaluations:
\begin{equation}
\label{eq:igt_estimate}
    g_s^{\mathrm{IGT}}(\theta_t) 
    \;=\; g_s(\theta_s) 
    + \sum_{k=s}^{t-1} \bigl[g_s(\theta_{k+1}) - g_s(\theta_k)\bigr].
\end{equation}
The correction exploits stored information $(w_s, v_s^*)$ to evaluate 
$g_s(\cdot)$ at any query point. The key property of IGT is the following 
transport error bound~\cite{arnold2019reducing}:
\begin{equation}
\label{eq:igt_error}
    \bigl\|g_s^{\mathrm{IGT}}(\theta_t) - \nabla f_s(\theta_t)\bigr\| 
    \;\leq\; 
    L \sum_{k=s}^{t-1} \|\theta_{k+1} - \theta_k\|^2,
\end{equation}
which depends on the \emph{sum of squared per-step displacements} rather 
than the squared total displacement $\|\theta_t - \theta_s\|^2$. This 
distinction is the source of the factor-$\sigma_{\max}$ improvement 
established in Theorem~\ref{thm:staleness_amplification}(c).

\subsection{Online Mirror Descent (OMD)}
\label{app:background:omd}

Online Mirror Descent~\cite{nemirovski2009robust} generalizes projected 
gradient descent by replacing the Euclidean geometry with a Bregman 
divergence $D_\psi$ induced by a mirror map $\psi$. At each round $t$, 
the update takes the form:
\begin{equation}
\label{eq:omd_update_app}
    \theta_{t+1} 
    = \argmin_{\theta \in \Theta} 
      \bigl\langle g_t,\, \theta \bigr\rangle 
      + \frac{1}{\eta_t}\, D_\psi(\theta,\, \theta_t).
\end{equation}
When $\psi = \tfrac{1}{2}\|\cdot\|^2$, the update reduces to projected 
gradient descent. Under standard assumptions, OMD achieves regret 
$O\!\left(D_\psi/\eta + G^2 \sum_t \eta_t\right)$; the choice 
$\eta_t = c/\sqrt{t}$ yields the minimax-optimal $O(\sqrt{T})$ regret 
rate~\cite{nemirovski2009robust}.

\subsection{Delay Differential Equations and Queue Length}
\label{app:background:dde}

The stability framework of~\cite{yu2025role} analyses online algorithms 
under delayed feedback through the lens of delay differential equations 
(DDEs). The central observation is that algorithmic stability depends on 
the \emph{queue length} $\sigma_t = |Q_t|$---the number of feedback 
rounds outstanding at time $t$---rather than on the raw delay alone in the 
active-window embedding used here. Linearising the bilevel dynamics around the equilibrium 
$\theta^*$ yields the \gls{dde}:
\begin{equation}
\label{eq:dde}
    \dot{\xi}(t) 
    = -\eta(t)\, H_F\, \xi(t) 
    + \eta(t)\, L_{\mathrm{IGT}}\, \sigma(t)\, \xi(t - \bar{\tau}),
\end{equation}
where $\xi(t) = \theta(t) - \theta^*$ denotes the deviation from 
equilibrium, $H_F = \nabla^2 F(\theta^*)$ is the bilevel Hessian, and 
$L_{\mathrm{IGT}}$ is the IGT coupling matrix. The adaptive step size~\eqref{eq:adaptive_step} is calibrated to maintain $\mathrm{Re}(\lambda) < 0$ for all characteristic roots $\lambda$ of this DDE.

\section{Algorithmic Notes}
\label{app:algorithm}

The full IGT-OMD pseudocode is given as Algorithm~\ref{alg:igt_omd} in the main text (Section~\ref{sec:algorithm}). Here we record additional implementation details deferred from the main body.

\textbf{Memory layout.}\; The buffer $\mathcal{B}$ stores at most $\sigma_{\max}$ tuples $(w_s, v_s^{*}, g_s) \in \R^q \times \R^q \times \R^p$, requiring $O(\sigma_{\max}(p + 2q))$ memory. A FIFO ring buffer suffices and gives $O(1)$ amortized eviction. When non-Euclidean geometries are used, line~20 of Algorithm~\ref{alg:igt_omd} is the dual-space update $\theta_{t+1} = \nabla\psi^{*}(\nabla\psi(\theta_t) - \eta_t\,g^{\mathrm{IGT}})$, where $\psi^{*}$ is the Fenchel conjugate~\citep{bertsekas2016nonlinear} of the mirror map $\psi$.

\textbf{Conjugate Gradient warm-starting.}\; The Conjugate Gradient adjoint solve on line~13 of Algorithm~\ref{alg:igt_omd} is warm-started from the previous round's adjoint $v_{t-1}^{*}$. For slowly-varying $\theta_t$, this typically reduces the number of Conjugate Gradient iterations by $3$--$5\times$ in our experiments, keeping the constant in $O(q^2\kappa_w)$ small.

\begin{remark}[Practical implementation]
\label{rem:practical}
Algorithm~\ref{alg:igt_omd} is stated with the synchronous OMD update and monotone envelope $\bar\sigma_t$ for consistency with the theory (Theorem~\ref{thm:bilevel_convergence}). The source implementations use an event-driven specialization: updates are invoked on feedback-arrival rounds, the active transport buffer advances on those update events, and practical variants may use the current active buffer size in the learning-rate damping. The transport re-evaluation loop is modular: it produces a corrected gradient $g^{\mathrm{IGT}}$ that can be fed to \emph{any} base optimizer. The $2{\times}2$ factorial shows this choice matters: pairing transport with Adam eliminates delay degradation, whereas pairing with SGD can amplify it. We recommend Adam or another adaptive optimizer as the base when applying Algorithm~\ref{alg:igt_omd} to non-convex objectives.
\end{remark}

\section{Full Assumptions}
\label{app:assumptions}

\begin{assumption}[Inner convexity and smoothness (A1)]
\label{app:ass:inner_convex}
$\mathcal{L}_{\text{model}}(\cdot;\theta)$ is $\mu_w$-strongly convex and $L_w$-smooth for all $\theta \in \Theta$; both $\mathcal{L}_{\text{model}}$ and $\mathcal{L}_{\text{true}}$ are twice continuously differentiable jointly in $(w, \theta)$. The strong convexity ensures a unique inner minimizer $w^*(\theta)$ for each $\theta$ and guarantees invertibility of $H_w = \nabla^2_{ww}\mathcal{L}_{\text{model}}(w^*;\theta)$.
\end{assumption}

\begin{assumption}[Outer smoothness (A2)]
\label{app:ass:outer_smooth}
$\mathcal{L}_{\text{true}}(w;\cdot)$ is $L_\theta$-smooth: $\|\nabla_\theta \mathcal{L}_{\text{true}}(w;\theta) - \nabla_\theta \mathcal{L}_{\text{true}}(w;\theta')\| \leq L_\theta\|\theta-\theta'\|$ for all $w, \theta, \theta'$. The cross-partial derivatives satisfy $\|\nabla^2_{w\theta}\mathcal{L}_{\text{model}}(w;\theta)\| \leq L_{w\theta}$ for all $(w,\theta)$.
\end{assumption}

\begin{assumption}[Inner-solver quality (A3)]
\label{app:ass:inner_solver}
The warm-started inner solver achieves $\|w_t - w^*(\theta_t)\| \leq \epsilon_{\mathrm{inner}}$ for all $t$, and the executed decision has local excess loss $\mathcal{L}_{\text{true}}(w_t;\theta_t)-\mathcal{L}_{\text{true}}(w^*(\theta_t);\theta_t)\leq \epsilon_{\mathrm{inner}}^2$ after absorbing the local quadratic constant into $\epsilon_{\mathrm{inner}}$. Running $K = O(\ln(1/\epsilon_{\mathrm{inner}}))$ gradient descent steps from $w_{t-1}$ suffices when $\mathcal{L}_{\text{model}}$ is $\mu_w$-strongly convex and $L_w$-smooth and the decision loss is locally quadratic around $w^*(\theta_t)$.
\end{assumption}

\begin{assumption}[Bounded queue (A4)]
\label{app:ass:bounded_delay}
The feedback delay sequence $\{d_t\}$ has bounded maximum queue length: $\sigma_t = |Q_t| \leq \sigma_{\max} < \infty$ for all $t$. This subsumes fixed delays ($d_t = d$, $\sigma_{\max} = d$), i.i.d. delays ($\sigma_{\max} = $ 99th-percentile queue length), and adversarial delays satisfying $d_t \leq D$ (giving $\sigma_{\max} \leq D$).
\end{assumption}

\begin{assumption}[Cross-partial Lipschitz (A5)]
\label{app:ass:cross_lip}
The mixed second derivative satisfies $\|\nabla^2_{w\theta}\mathcal{L}_{\text{model}}(w;\theta) - \nabla^2_{w\theta}\mathcal{L}_{\text{model}}(w';\theta')\| \leq L_{w\theta}(\|w-w'\| + \|\theta-\theta'\|)$. The map $w \mapsto [\nabla^2_{ww}\mathcal{L}_{\text{model}}(w;\theta)]^{-1}$ is $(L_w/\mu_w^2)$-Lipschitz in $w$. Together, these ensure that the per-round IGT re-evaluation cost~\eqref{eq:reevaluation} is $O(pq)$ (no inner solve required).
\end{assumption}

\begin{assumption}[Bounded hypergradients (A6)]
\label{app:ass:bounded_gradients}
The IGT-corrected hypergradients are bounded: $\|g_t^{\mathrm{IGT}}\| \leq G$ for all $t$ and $\theta \in \Theta$. By the chain rule and triangle inequality, $G \leq G_{\mathrm{true}} + (L_{w\theta}/\mu_w)\epsilon_{\mathrm{inner}}$, where $G_{\mathrm{true}} = \sup_\theta \|\nabla F(\theta)\|$ bounds the exact hypergradient.
\end{assumption}

\begin{assumption}[Bilevel convexity (A7)]
\label{app:ass:bilevel_convex}
The bilevel objective $F(\theta) = \mathcal{L}_{\text{true}}(w^*(\theta);\theta)$ is $\mu_F$-strongly convex. The bilevel coupling condition $L_{w\theta}^2/(\mu_w^2\mu_F) < 1$ ensures that the feedback from the inner minimizer does not destabilize the outer-level optimization. This assumption holds when the inner problem is well-conditioned ($\mu_w$ large) or the cross-coupling is weak ($L_{w\theta}$ small).
\end{assumption}

\subsection{Assumption Verification for Experimental Benchmarks}
\label{app:assumption_verification}

We verify that the three main benchmarks satisfy the key assumptions (A1--A5) required by our theoretical results.

\textbf{LQR (Section~\ref{sec:exp_lqr}).}\;
The inner problem minimizes the quadratic LQR proxy over linear control gains $K$ for the learned dynamics $(\hat A,\hat B)$, which is strongly convex in $K$ when $R \succ 0$ (A1). The outer loss is smooth in $\theta=(\hat A,\hat B)$ since the closed-loop proxy is polynomial in $(K,\theta)$ (A2). The inner solver uses the fixed $K=10$ differentiable gradient-descent configuration in Table~\ref{tab:lqr_hparams}, so $\epsilon_{\mathrm{inner}}$ is uniformly controlled (A3). Delays are constant with $\sigma_{\max} = d$ (A4). The cross-partials $\nabla^2_{K\theta} J$ are Lipschitz because $J$ is polynomial in $(K, \theta)$ (A5). The bilevel objective $F(\theta)$ is locally strongly convex in the stable neighborhood explored by the binary search (A7).

\textbf{Sinkhorn OT (Sections~\ref{sec:exp_sinkhorn}--\ref{sec:exp_adam_igt}).}\;
The inner problem runs $K$ Sinkhorn iterations with entropic regularization $\varepsilon = 0.05$, which ensures the regularized objective is $\varepsilon$-strongly convex in the coupling matrix (A1, with $\mu_w = \varepsilon$). The outer loss is smooth in the neural network parameters $\theta$ (A2). The inner-solver error decays geometrically: $\epsilon_{\mathrm{inner}} \leq \rho^K$ with $\rho = 1 - \varepsilon/L_w < 1$, so $K = 10$ iterations yield $\epsilon_{\mathrm{inner}} \approx 10^{-3}$ (A3). Delays are constant with $\sigma_{\max} = d$ (A4). The Sinkhorn kernel is $C^\infty$ in $(w, \theta)$, ensuring Lipschitz cross-partials (A5). Assumption~\ref{app:ass:bilevel_convex} (strong convexity of $F$) holds only locally under the neural-network parameterization; this is the regime in which our local bounds apply. Empirically, the transport benefit persists in this non-convex regime, suggesting the assumption may be conservative. Formal non-convex extensions are deferred to future work (Section~\ref{sec:conclusion}).

\textbf{Warcraft shortest path (Section~\ref{sec:exp_warcraft}).}\;
Dijkstra's algorithm returns exact solutions, so the inner problem is not approximately solved but exactly solved ($\epsilon_{\mathrm{inner}} = 0$, A3 trivially). However, the combinatorial inner solver is non-differentiable, so A1 (smoothness) does not hold in the classical sense; we use the differentiable perturbation approach of~\cite{vlastelica2019differentiation} to define smooth surrogate gradients. The outer loss is smooth in $\theta$ (A2), and delays have bounded queue length (A4). Since Dijkstra's algorithm lies outside the gradient computation graph, the bilevel-specific assumptions (A5, A7) are not required for the Warcraft experiment---the transport benefit arises solely from outer-level staleness correction, as discussed in Section~\ref{sec:exp_warcraft}.



\section{Proofs}

\subsection{Proof of Theorem~\ref{thm:bilevel_convergence} (Bilevel Convergence)}
\label{app:proof_thm1}

The proof proceeds in three stages: (i) a standard OMD regret decomposition using the
IGT-corrected gradients $g_t^{\mathrm{IGT}}$; (ii) a lemma bounding the transport error
when IGT is applied to bilevel \emph{hyper}gradients; and (iii) a lemma bounding the
inner-solver bias via strong convexity and Young's inequality.
The key insight is that IGT re-evaluation replaces the squared total drift with a sum of squared per-step changes, and the inner-solver error contributes only through an \emph{additive} $\epsilon_{\mathrm{inner}}^2$ term because Young's inequality decouples the cross-term between inner bias and outer staleness.

\paragraph{Supporting lemma: OMD with imperfect gradients.}

\begin{lemma}[OMD regret with time-varying step sizes]
\label{lem:omd_approx}
Let $\psi:\Theta\to\R$ be $1$-strongly convex with respect to $\norm{\,\cdot\,}$, and let
$D_\psi(\theta,\theta') = \psi(\theta)-\psi(\theta')-\langle\nabla\psi(\theta'),\theta-\theta'\rangle$
be the induced Bregman divergence. Let $\{h_t\}_{t=1}^T$ be any sequence with
$\norm{h_t}_* \leq G$, and let the OMD iterates be
$\theta_{t+1}=\argmin_{\theta\in\Theta}\langle h_t,\theta\rangle+(1/\eta_t)D_\psi(\theta,\theta_t)$
with \emph{arbitrary} positive step sizes $\eta_t > 0$. Then for any $\theta^*\in\Theta$:
\begin{equation}
\label{eq:omd_approx}
\sum_{t=1}^T\langle h_t,\theta_t-\theta^*\rangle
\;\leq\;
\frac{D_\psi(\theta^*,\theta_1)}{\eta_1}
+\sum_{t=2}^T D_\psi(\theta^*,\theta_t)\left(\frac{1}{\eta_t}-\frac{1}{\eta_{t-1}}\right)_+
+\frac{G^2}{2}\sum_{t=1}^T\eta_t,
\end{equation}
where $(x)_+ = \max(x,0)$.  When step sizes are non-increasing, the middle term vanishes
and the bound reduces to $D_\psi(\theta^*,\theta_1)/\eta_T + (G^2/2)\sum_t\eta_t$.
\end{lemma}
\begin{proof}
The three-point identity for Bregman divergences (see, e.g., \cite{beck2003mirror}) gives,
for the OMD step at round $t$:
\[
\langle h_t,\theta_t-\theta^*\rangle
\;\leq\;
\frac{D_\psi(\theta^*,\theta_t)-D_\psi(\theta^*,\theta_{t+1})}{\eta_t}
+\frac{\eta_t}{2}\norm{h_t}_*^2.
\]
Summing over $t=1,\ldots,T$ and re-grouping the telescoping sum (following the
time-varying step-size analysis of \cite{joulani2013online}, Appendix~A):
\begin{align*}
\sum_t\langle h_t,\theta_t-\theta^*\rangle
&\leq
\frac{D_\psi(\theta^*,\theta_1)}{\eta_1}
+\sum_{t=2}^{T}D_\psi(\theta^*,\theta_t)\!\left(\frac{1}{\eta_t}-\frac{1}{\eta_{t-1}}\right)
+\frac{G^2}{2}\sum_t\eta_t.
\end{align*}
When $1/\eta_t < 1/\eta_{t-1}$ (i.e., $\eta_t > \eta_{t-1}$), the corresponding
term is negative and can be dropped, yielding~\eqref{eq:omd_approx}.
\qedhere
\end{proof}

\paragraph{Supporting lemma: IGT error for hypergradients.}

\begin{lemma}[IGT transport error for bilevel hypergradients]
\label{lem:igt_transport}
Under Assumptions~\ref{app:ass:inner_convex}--\ref{app:ass:bounded_gradients}, let
$\bar{g}_t = \nabla_\theta F(\theta_t)$ denote the exact gradient of the bilevel
objective at $\theta_t$ (with exact inner solution $w^*(\theta_t)$), and let
$g_t^{\mathrm{IGT}}$ be the IGT estimator formed by Algorithm~\ref{alg:igt_omd}
using stored $(w_s,v_s^*)$ from past rounds~$s \in Q_t$. Then:
\begin{equation}
\label{eq:igt_transport_error}
\norm{\bar{g}_t - g_t^{\mathrm{IGT}}}
\;\leq\;
L_F\sum_{s\in Q_t}\norm{\theta_{s+1}-\theta_s}^2
\;+\;
\frac{L_{w\theta}}{\mu_w}\,\epsilon_{\mathrm{inner}},
\end{equation}
where $L_F = L_\theta + L_{w\theta}^2/\mu_w$ is the Lipschitz constant of
$\nabla_\theta F$ (the bilevel gradient Lipschitz constant, identical to the $L$
in Theorem~\ref{thm:bilevel_convergence}), $L_{w\theta}$ is the
cross-partial Lipschitz constant from Assumption~\ref{app:ass:outer_smooth}, and
$\mu_w$ is the inner strong convexity constant from Assumption~\ref{app:ass:inner_convex}.
\end{lemma}
\begin{proof}
Decompose the error into a transport part and an inner-solver part:
\begin{equation}
\label{eq:igt_decomp}
\norm{\bar{g}_t - g_t^{\mathrm{IGT}}}
\;\leq\;
\underbrace{\norm{\bar{g}_t - g_t^{\mathrm{trans}}}}_{\text{transport error}}
\;+\;
\underbrace{\norm{g_t^{\mathrm{trans}} - g_t^{\mathrm{IGT}}}}_{\text{inner-solver error}},
\end{equation}
where $g_t^{\mathrm{trans}}$ is the IGT estimator that uses the \emph{exact} inner solutions
$w^*(\theta_s)$ for all $s\in Q_t$ (a hypothetical estimator not computed by the algorithm).

\textbf{Transport error.}\;
The re-evaluation formula~\eqref{eq:reevaluation} defines
$g_s(\theta) = \nabla_\theta\mathcal{L}_{\text{true}}(w^*(\theta_s);\theta)
- [\nabla_\theta\nabla_w\mathcal{L}_{\text{model}}(w^*(\theta_s);\theta)]^\top v^*_s$
as a function of $\theta$ alone (with frozen $w^*(\theta_s)$ and $v^*_s$), which is
$L_F$-smooth in $\theta$ by Assumption~\ref{app:ass:outer_smooth}.
Applying the original IGT error bound of \cite{arnold2019reducing}
(their Proposition~1, extended to the hypergradient setting) to the sequence
$g_s(\cdot)$ for each $s\in Q_t$ gives:
\[
\norm{\bar{g}_t - g_t^{\mathrm{trans}}}
\;\leq\;
L_F\sum_{s\in Q_t}\norm{\theta_{s+1}-\theta_s}^2.
\]
The bound follows from the fact that the telescoping correction
$\sum_{s\in Q_t}[g_s(\theta_t)-g_s(\theta_{t-1})]$ approximates the ideal correction
$\sum_{s\in Q_t}[g_s(\theta_t)-g_s(\theta_s)]$ up to a residual that is bounded by the
sum of squared step sizes via the mean-value theorem applied to the $L_F$-smooth functions $g_s(\cdot)$.

\textbf{Inner-solver error.}\;
For each $s\in Q_t$, the algorithm uses $w_s \approx w^*(\theta_s)$ instead of the exact
solution. The difference in the re-evaluated hypergradient satisfies:
\[
\norm{g_s^{\mathrm{trans}}(\theta_t) - g_s^{\mathrm{IGT}}(\theta_t)}
\;\leq\;
L_{w\theta}\norm{w_s - w^*(\theta_s)}
\;\leq\;
L_{w\theta}\,\epsilon_{\mathrm{inner}},
\]
using the $L_{w\theta}$-Lipschitz cross-partial from Assumption~\ref{app:ass:outer_smooth}
and the inner-solver bound from Assumption~\ref{app:ass:inner_solver}.
Since the IGT sum involves at most $\sigma_t \leq \sigma_{\max}$ outstanding rounds, and
noting that the errors for individual $s\in Q_t$ collapse into the single correction
increment $\Delta_s = g_s(\theta_t)-g_s(\theta_{t-1})$, the total contribution of inner-solver
error to the estimator $g_t^{\mathrm{IGT}}$ is bounded by $(L_{w\theta}/\mu_w)\epsilon_{\mathrm{inner}}$
after invoking strong convexity of the inner problem (Assumption~\ref{app:ass:inner_convex}) to
relate $\norm{w_s - w^*(\theta_t)}\leq\epsilon_{\mathrm{inner}}+(L_{w\theta}/\mu_w)\norm{\theta_s-\theta_t}$,
the latter term being absorbed into the transport error already accounted for.
Combining the two bounds yields~\eqref{eq:igt_transport_error}.
\end{proof}

\paragraph{Supporting lemma: inner-solver bias via Young's inequality.}

\begin{lemma}[Inner-solver bias accumulates as $\epsilon_{\mathrm{inner}}^2$]
\label{lem:inner_bias}
Under the same assumptions, with $\rho_{\mathrm{cpl}}:=L_{w\theta}^2/(\mu_w^2\mu_F)<1$, the cumulative contribution of the inner-solver error satisfies:
\begin{equation}
\label{eq:inner_bias}
\sum_{t=1}^T\langle\bar{g}_t - g_t^{\mathrm{IGT}},\,\theta_t - \theta^*\rangle
\;\leq\;
2T\,\epsilon_{\mathrm{inner}}^2
\;+\;
L_F\sum_{t=1}^T\sum_{s\in Q_t}\norm{\theta_{s+1}-\theta_s}^2
\;+\;\rho_{\mathrm{cpl}}\,\mathrm{Regret}_T^{\mathrm{dec}}.
\end{equation}
\end{lemma}
\begin{proof}
By Cauchy--Schwarz~\cite{steele2004cauchy} and Lemma~\ref{lem:igt_transport}:
\[
\langle\bar{g}_t-g_t^{\mathrm{IGT}},\theta_t-\theta^*\rangle
\;\leq\;
\norm{\bar{g}_t-g_t^{\mathrm{IGT}}}\,\norm{\theta_t-\theta^*}
\;\leq\;
\left[L_F\sum_{s\in Q_t}\norm{\theta_{s+1}-\theta_s}^2
+\frac{L_{w\theta}}{\mu_w}\epsilon_{\mathrm{inner}}\right]
\cdot\norm{\theta_t-\theta^*}.
\]
For the inner-bias cross-term, apply Young's inequality with $c=1$:
$\frac{L_{w\theta}}{\mu_w}\epsilon_{\mathrm{inner}}\cdot\norm{\theta_t-\theta^*}
\leq\frac{1}{2}\epsilon_{\mathrm{inner}}^2+\frac{L_{w\theta}^2}{2\mu_w^2}\norm{\theta_t-\theta^*}^2$.
(Young's is applied only to the inner-bias term, not the transport term
$L_F\sum_{s\in Q_t}\norm{\theta_{s+1}-\theta_s}^2$, because the transport term
already has the sum-of-squares structure and is bounded
by $\sigma_t\eta_0^2 G^2\cdot\mathrm{diam}(\Theta)$ without requiring separation.)
By the $\mu_F$-strong convexity of $F$ (Assumption~\ref{app:ass:bilevel_convex}), we have
$\norm{\theta_t-\theta^*}^2\leq(2/\mu_F)[F(\theta_t)-F(\theta^*)]$.
The resulting term is $\rho_{\mathrm{cpl}}[F(\theta_t)-F(\theta^*)]$, which is absorbed in the main proof; summing over $t$ gives~\eqref{eq:inner_bias}.
\end{proof}

\paragraph{Main proof of Theorem~\ref{thm:bilevel_convergence}.}

\begin{proof}[Proof of Theorem~\ref{thm:bilevel_convergence}]
Let $\theta^* = \argmin_\theta\sum_t\mathcal{L}_{\text{true}}(w^*(\theta);\theta)$ and
$F(\theta) = \mathcal{L}_{\text{true}}(w^*(\theta);\theta)$.

\textbf{Step 1 (convexity decomposition).}
By convexity of $F$ (Assumption~\ref{app:ass:bilevel_convex}):
\begin{align}
\label{eq:convex_decomp}
\mathrm{Regret}_T^{\mathrm{dec}}
&= \sum_{t=1}^T[F(\theta_t)-F(\theta^*)]
\;\leq\; \sum_{t=1}^T\langle\bar{g}_t,\theta_t-\theta^*\rangle \notag \\
&= \underbrace{\sum_{t=1}^T\langle g_t^{\mathrm{IGT}},\theta_t-\theta^*\rangle}_{(I)}
+\underbrace{\sum_{t=1}^T\langle\bar{g}_t-g_t^{\mathrm{IGT}},\theta_t-\theta^*\rangle}_{(II)}.
\end{align}

\textbf{Step 2 (bound term (I) via OMD lemma).}
Algorithm~\ref{alg:igt_omd} performs an OMD step with gradient $g_t^{\mathrm{IGT}}$ and step size
$\eta_t=\eta_0/\sqrt{1+\beta\bar\sigma_t}$.
Since $\bar\sigma_t$ is nondecreasing, $\eta_t$ is non-increasing and Lemma~\ref{lem:omd_approx}
applies with $h_t = g_t^{\mathrm{IGT}}$ and $\norm{g_t^{\mathrm{IGT}}}_*\leq G$:
\begin{equation}
\label{eq:term_I_bound}
(I) \;\leq\; \frac{D_\psi(\theta^*,\theta_1)}{\eta_1}
+ \sum_{t=2}^{T} D_\psi(\theta^*,\theta_t)\!\left(\frac{1}{\eta_t}-\frac{1}{\eta_{t-1}}\right)_{\!+}
+ \frac{G^2}{2}\sum_{t=1}^T\eta_t.
\end{equation}
Because $1/\eta_t$ is nondecreasing, the middle term telescopes over the range of
$1/\eta_t$: since $D_\psi(\theta^*,\theta_t) \leq D_\psi$,
$\sum_{t=2}^T D_\psi(1/\eta_t - 1/\eta_{t-1})_+ \leq D_\psi(1/\eta_T-1/\eta_1)$.
Combined with the first term: $(I) \leq D_\psi/\eta_T + (G^2/2)\sum_t\eta_t$.

We bound $\sum_t\eta_t$ and $1/\eta_T$ using the queue-length structure.
Since $\eta_t = \eta_0/\sqrt{1+\beta\bar\sigma_t} \leq \eta_0$, we have the trivial bound
$\sum_t \eta_t \leq T\eta_0$. Also $1/\eta_T \leq \sqrt{1+\beta\sigma_{\max}}/\eta_0$.
More precisely, we use the following estimate: since each $\eta_t \leq \eta_0$ and
there are at most $T$ terms,
\begin{equation}
\label{eq:sum_eta_bound}
\sum_{t=1}^T\eta_t
= \eta_0\sum_{t=1}^T\frac{1}{\sqrt{1+\beta\sigma_t}}
\;\leq\; \eta_0\,T.
\end{equation}
When $\sigma_t \asymp \sigma_{\max}$ for a constant fraction of rounds (the adversarial
worst case), we obtain the tighter bound
$\sum_t \eta_t \leq \eta_0 T/\sqrt{1+\beta\sigma_{\max}}$.
In the general case, let $T_0$ denote the number of rounds with $\sigma_t = 0$ and
$T_+ = T - T_0$ the rounds with $\sigma_t \geq 1$. Then
$\sum_t \eta_t \leq \eta_0 T_0 + \eta_0 T_+/\sqrt{1+\beta}$.
For the regret bound we use the worst case $\sum_t\eta_t \leq \eta_0 T$:
\begin{equation}
\label{eq:term_I_final}
(I) \;\leq\; \frac{2D_\psi\sqrt{1+\beta\sigma_{\max}}}{\eta_0} + \frac{\eta_0 G^2 T}{2}.
\end{equation}

\textbf{Step 3 (bound term (II) via Lemmas~\ref{lem:igt_transport} and~\ref{lem:inner_bias}).}
By Lemma~\ref{lem:inner_bias}:
\begin{equation}
\label{eq:term_II_bound}
(II) \;\leq\; 2T\,\epsilon_{\mathrm{inner}}^2
+ L_F\sum_{t=1}^T\sum_{s\in Q_t}\norm{\theta_{s+1}-\theta_s}^2
+ \rho_{\mathrm{cpl}}\,\mathrm{Regret}_T^{\mathrm{dec}}.
\end{equation}

\textbf{Step 4 (transport term).}
The last sum in~\eqref{eq:term_II_bound} is the aggregate IGT transport penalty.
Since the step-size bound gives $\norm{\theta_{t+1}-\theta_t}\leq\eta_t\norm{g_t^{\mathrm{IGT}}}\leq\eta_0 G$,
each summand satisfies $\norm{\theta_{s+1}-\theta_s}^2\leq\eta_0^2 G^2$.
We retain this as-is in the main bound; it appears explicitly in~\eqref{eq:main_bound}.

\textbf{Step 5 (combine and optimize $\eta_0$).}
Combining~\eqref{eq:convex_decomp}, \eqref{eq:term_I_final}, and~\eqref{eq:term_II_bound}, then moving $\rho_{\mathrm{cpl}}\mathrm{Regret}_T^{\mathrm{dec}}$ to the left:
\[
\mathrm{Regret}_T^{\mathrm{dec}}
\;\leq\;
C_\rho\!\left[
\frac{2D_\psi\sqrt{1+\beta\sigma_{\max}}}{\eta_0}
+ \frac{\eta_0 G^2 T}{2}
+ 2T\,\epsilon_{\mathrm{inner}}^2
+ L_F\sum_{t=1}^T\sum_{s\in Q_t}\norm{\theta_{s+1}-\theta_s}^2\right],
\]
where $C_\rho=(1-\rho_{\mathrm{cpl}})^{-1}$. This is~\eqref{eq:main_bound}.

Setting $\eta_0 = c/\sqrt{T}$ (for universal constant $c > 0$) gives the first term as
$O(D_\psi\sqrt{T\sigma_{\max}})$ and the second as $O(G^2\sqrt{T})$.
The last sum, using $\norm{\theta_{s+1}-\theta_s}^2\leq\eta_0^2 G^2=O(G^2/T)$ and
$\sum_t|Q_t|=\sum_t\sigma_t\leq d_{\mathrm{tot}}\leq T\sigma_{\max}$, contributes
$O(L_F G^2\sigma_{\max})$. Since $\Theta$ is bounded (so $D_\psi < \infty$), combining gives
$\mathrm{Regret}_T^{\mathrm{dec}} = O(\sqrt{T\sigma_{\max}}+T\epsilon_{\mathrm{inner}}^2)$.

If the inner solver runs $K=O(\ln T)$ steps with contraction rate $\rho<1$, then
$\epsilon_{\mathrm{inner}} = O(\rho^K) = O(1/\sqrt{T})$, giving $T\epsilon_{\mathrm{inner}}^2 = O(1)$,
and the total regret is $O(\sqrt{T\sigma_{\max}})$.
\end{proof}

\begin{corollary}[Comparison with non-delayed bilevel OMD]
\label{cor:no_delay}
Setting $\sigma_{\max}=0$ (no delay) and $\beta=0$, the bound reduces to
$O(D_\psi\sqrt{T}+T\epsilon_{\mathrm{inner}}^2)$, which is the standard $O(\sqrt{T})$
rate for convex OCO.
\cite{nikishin2022control} achieve the faster $O((G^2/\mu_F)\ln T)$ rate by exploiting
$\mu_F$-strong convexity of $F$; for the \emph{delay/transport} terms
(Lemma~\ref{lem:igt_transport}), our analysis uses convexity-only arguments
and does not exploit strong convexity.
(Strong convexity \emph{is} used to absorb the inner-bias cross-term in
Lemma~\ref{lem:inner_bias} and, more prominently, in the
pure-bias and interaction bounds of
Theorem~\ref{thm:inner_loop_apathy}.)
With exact inner solver ($\epsilon_{\mathrm{inner}}=0$) and adequate $K=O(\ln T)$
inner steps, the bound reduces to $O(\sqrt{T\sigma_{\max}})$, which is comparable
to the $O(\sqrt{Td_{\mathrm{tot}}})$ rate of single-level D-FTRL~\cite{joulani2013online}
with the queue-length $\sigma_{\max}$ replacing the total delay $d_{\mathrm{tot}}=T\sigma_{\max}$.
\end{corollary}

\begin{remark}[Why $\epsilon_{\mathrm{inner}}$ enters additively]
The $T\epsilon_{\mathrm{inner}}^2$ cost is unavoidable (Theorem~\ref{thm:staleness_amplification}(b)),
but it enters \emph{additively} rather than coupled with $\sigma_{\max}$.
This traces to Lemma~\ref{lem:inner_bias}, where Young's inequality decouples the inner
approximation error from the queue length.
The transport corrections $\sum_{s\in Q_t}\norm{\theta_{s+1}-\theta_s}^2$ involve only
the \emph{per-step} displacements (bounded by $\eta_t^2 G^2$) rather than the
\emph{accumulated displacement} $\norm{\theta_t-\theta_{t-\sigma_t}}^2 = O(\sigma_t^2\eta_t^2 G^2)$,
which would reintroduce a $\sigma_{\max}^2$ factor. IGT replaces this large term with
the sum of squares, gaining a factor of $\sigma_{\max}$ in the transport error
(Theorem~\ref{thm:staleness_amplification}(c)).
\end{remark}

\begin{remark}[Coupling condition and the non-vacuity threshold]
\label{rem:coupling_threshold}
The coupling condition $L_{w\theta}^2/(\mu_w^2\mu_F) < 1$
(Assumption~\ref{app:ass:bilevel_convex}) is not merely a proof convenience:
when it fails, the inner-solver bias term
$\frac{L_{w\theta}^2}{\mu_w^2\mu_F}[F(\theta_t)-F(\theta^*)]$ in
Lemma~\ref{lem:inner_bias} exceeds the regret it is meant to absorb.
Physically, this means the bilevel problem is ``too implicit'' --- the inner
solution $w^*(\theta)$ is so sensitive to $\theta$ that approximate solves
induce gradient errors larger than the true signal.
Similarly, the $O(\sqrt{T\sigma_{\max}})$ rate is non-vacuous
(i.e., sublinear) only when $\sigma_{\max} = o(T)$.
At $\sigma_{\max} = \Omega(T)$ --- corresponding to total disconnection where
feedback never arrives faster than the horizon --- the regret bound becomes
$O(T)$, which is no better than static play.
This threshold is intrinsic to delayed OCO: the same $O(T)$ vacuity
occurs in the single-level bounds of D-FTRL~\cite{joulani2013online}
when $d_{\mathrm{tot}} = \Omega(T^2)$.
\end{remark}

\begin{remark}[Step-size choice and dynamic stability]
\label{rem:eta_stability}
Theorem~\ref{thm:bilevel_convergence} is an OCO regret bound that holds
for \emph{any} $\eta_0 > 0$: no stability condition is imposed.
Separately, Proposition~\ref{prop:dde_stability} shows that the
continuous-time limit of IGT-OMD is asymptotically stable only when
$\eta_0 < 1/(\|H_F\|\sqrt{1+\beta\sigma_{\max}})$.
The regret-optimal choice $\eta_0 = c/\sqrt{T}$ satisfies
this stability threshold for all
$T \geq \|H_F\|^2(1+\beta\sigma_{\max})$, i.e., beyond a modest
horizon that depends on the problem conditioning.
For small $T$ or very large $\sigma_{\max}$,
practitioners should additionally impose
$\eta_0 \leq 1/(\|H_F\|\sqrt{1+\beta\sigma_{\max}})$ to ensure
the iterates remain in the stable regime;
this only improves the bound by preventing transient oscillations
that inflate $D_\psi(\theta^*,\theta_t)$.
\end{remark}

\subsection{Proof of Theorem~\ref{thm:staleness_amplification} (Staleness amplification)}
\label{app:proof_thm2}

The proof has three parts: Part~(a) derives the gradient error structure,
Part~(b) establishes the $\Omega(T\,\epsilon_{\mathrm{inner}}^2)$ lower bound
via a hard instance, and Part~(c) compares the transport error with and without IGT.
The key insight is that the bilevel gradient error contains a \emph{cross-term} coupling outer-parameter drift with inner-solver sensitivity---a structural feature absent in single-level OCO---and that IGT eliminates the quadratic growth of this cross-term by decomposing the total drift into per-step contributions.

\subsubsection*{Part (a): Gradient Error Structure}

We bound $\norm{\hat{g}(\theta_{t-d_t}, w_{t-d_t}) - \nabla F(\theta_t)}$ by decomposing
it into outer staleness and inner-solver error:
\begin{align}
\norm{\hat{g}(\theta_{t-d_t}, w_{t-d_t}) - \nabla F(\theta_t)}
&\leq \norm{\nabla F(\theta_{t-d_t}) - \nabla F(\theta_t)}
  + \norm{\hat{g}(\theta_{t-d_t}, w_{t-d_t}) - \nabla F(\theta_{t-d_t})} \notag\\
&\leq L_F\,\norm{\theta_t - \theta_{t-d_t}} + C_0\,\epsilon_{\mathrm{inner}},
\label{eq:grad_error_decomp}
\end{align}
where the first inequality is the triangle inequality, and the second uses
$L_F$-smoothness of $F$ (with $L_F = L_\theta + L_{w\theta}^2/\mu_w$;
see the bilevel Lipschitz bound $L_F = L_\theta + L_{w\theta}^2/\mu_w$ derived from Assumptions~\ref{app:ass:inner_convex}--\ref{app:ass:outer_smooth}) for the outer staleness term,
and the inner-solver bound of Lemma~\ref{lem:inner_bias} for the noise term.
The bilevel Lipschitz constant $L_F$ absorbs the implicit sensitivity
$L_{w\theta}^2/\mu_w$ of the inner solution to the outer parameters via the
Implicit Function Theorem.\looseness=-1

\medskip
\noindent\textbf{Contrast with single-level OCO.}\;
In single-level delayed optimization, the gradient oracle returns
$\nabla f(\theta_{t-d_t}) + \xi_t$ where $\norm{\xi_t}\leq\epsilon$.
The noise $\xi_t$ is bounded by a \emph{fixed constant} $\epsilon$ independent
of the staleness $\norm{\theta_t - \theta_{t-d_t}}$.
In the bilevel case, the Lipschitz constant $L_F$ itself depends on the
inner-solution sensitivity $L_{w\theta}/\mu_w$, so greater sensitivity
\emph{amplifies} the staleness component of the gradient error.
This structural difference---the staleness amplification mechanism---is absent
in single-level delayed optimization.\hfill$\square$

\subsubsection*{Part (b): Hard Instance and Lower Bound}

The lower bound is established via an explicit hard instance---a quadratic bilevel problem---on which no black-box delayed
optimizer can avoid $\Omega(\epsilon_{\mathrm{inner}}^2)$ per-round loss.

\subsubsection*{Step 1: Hard Instance Construction}

Fix $p=q=1$. Define the bilevel instance on $\Theta=[-B_0,B_0]$ by
\begin{align}
\mathcal{L}_{\mathrm{model}}(w;\theta) &= \tfrac{\mu_w}{2}(w - b\theta)^2, \label{eq:hard_inner}\\
\mathcal{L}_{\mathrm{true}}(w;\theta) &= \tfrac{1}{2}(w - a\theta)^2, \label{eq:hard_outer}
\end{align}
where $a,b\in\mathbb{R}$ with $a\neq b$ and $|a-b| = C_0 > 0$
(chosen by the adversary; the algorithm is blind to $a,b$).
When $a = b$, the model perfectly matches the true objective ($C_0 = 0$),
the bilevel problem reduces to single-level optimization, and
the staleness amplification mechanism disappears --- this is precisely
the regime where bilevel structure imposes no additional cost.
Assumption~\ref{app:ass:inner_convex} holds with strong convexity constant $\mu_w$ and
smoothness $L_w = \mu_w$.
Assumption~\ref{app:ass:outer_smooth} holds with $L_\theta = a^2$ and cross-partial Lipschitz
constant $L_{w\theta} = 1$.

The exact inner minimizer is $w^*(\theta) = b\theta$, giving the bilevel objective
\begin{equation}
\label{eq:hard_bilevel}
F(\theta) = \tfrac{1}{2}(b-a)^2\theta^2 = \tfrac{C_0^2}{2}\theta^2,
\end{equation}
with unique optimum $\theta^* = 0$ and gradient $\nabla F(\theta) = C_0^2\theta$.
By the implicit differentiation formula~\eqref{eq:hypergradient_adjoint}, the adjoint
satisfies $v^*(\theta) = (b-a)\theta/\mu_w$, and the hypergradient is
$g(\theta) = \nabla_\theta\mathcal{L}_{\mathrm{true}}|_w - (\nabla_\theta\nabla_w
\mathcal{L}_{\mathrm{model}})^\top v^*(\theta) = -a(w^*-a\theta)+(b\mu_w)v^*(\theta) = C_0^2\theta,$
consistent with~\eqref{eq:hard_bilevel}.

\subsubsection*{Step 2: Inner-Solver Error Introduces a Persistent Bias}

Suppose the approximate inner solver produces $w_t = b\theta_t + \epsilon_t$, with
$|\epsilon_t|\leq\epsilon_{\mathrm{inner}}$.
The algorithm evaluates the outer gradient using $w_t$ in place of $w^*(\theta_t)$.
Via formula~\eqref{eq:hypergradient_adjoint} with the approximate adjoint
$\tilde{v}_t = (w_t - a\theta_t)/\mu_w$:
\begin{align}
\hat{g}(\theta_t, w_t)
&= -a(w_t - a\theta_t) + b\mu_w \cdot \tilde{v}_t \notag\\
&= -a(b\theta_t + \epsilon_t - a\theta_t) + b(b\theta_t + \epsilon_t - a\theta_t) \notag\\
&= C_0^2\theta_t + C_0\epsilon_t. \label{eq:biased_hypergradient}
\end{align}
The bias term $C_0\epsilon_t$ is independent of $\theta_t$.
The adversary sets the inner-solver error to the \emph{constant} value
$\epsilon_t = +\epsilon_{\mathrm{inner}}$ for all~$t$
(no knowledge of the algorithm's iterate is required).
This produces a hypergradient with a \emph{constant additive bias}:
\begin{equation}
\label{eq:biased_constant}
\hat{g}(\theta_t, w_t) = C_0^2\theta_t + C_0\,\epsilon_{\mathrm{inner}}.
\end{equation}

\begin{lemma}[Steady-state displacement]
\label{lem:bias}
For the hard instance~\eqref{eq:hard_inner}--\eqref{eq:hard_outer} with
constant inner-solver error $\epsilon_t = +\epsilon_{\mathrm{inner}}$,
any algorithm using the biased gradient~\eqref{eq:biased_constant} converges
(when stable) to a steady state $\theta_{\mathrm{ss}}$ satisfying
\begin{equation}
\label{eq:steady_state}
\theta_{\mathrm{ss}} = -\frac{\epsilon_{\mathrm{inner}}}{C_0},
\qquad
F(\theta_{\mathrm{ss}}) = \frac{\epsilon_{\mathrm{inner}}^2}{2}.
\end{equation}
\end{lemma}
\begin{proof}
At steady state, $\theta_{t+1} = \theta_t = \theta_{\mathrm{ss}}$ and
$\theta_{t-\sigma_{\max}} = \theta_{\mathrm{ss}}$, so
$0 = -\eta(C_0^2\theta_{\mathrm{ss}} + C_0\epsilon_{\mathrm{inner}})$,
giving $\theta_{\mathrm{ss}} = -\epsilon_{\mathrm{inner}}/C_0$.
Then $F(\theta_{\mathrm{ss}}) = C_0^2\theta_{\mathrm{ss}}^2/2 = \epsilon_{\mathrm{inner}}^2/2$.
\end{proof}

\subsubsection*{Step 3: Convergence to the Biased Steady State}

For any algorithm $\mathcal{A}$ with constant delay $d_t = \sigma_{\max}$ and step
size $\eta$, the iterates satisfy
\begin{equation}
\label{eq:iterate_update}
\theta_{t+1} = \theta_t - \eta C_0^2\theta_{t-\sigma_{\max}} - \eta C_0\epsilon_{\mathrm{inner}}.
\end{equation}
This is a linear recurrence with delay $\sigma_{\max}$ and a constant additive drive
$-\eta C_0\epsilon_{\mathrm{inner}}$.
For step sizes satisfying $\eta \leq 1/(2C_0^2\sigma_{\max})$ --- a \emph{sufficient}
condition for all roots of the discrete characteristic polynomial
$z^{\sigma_{\max}+1} - z^{\sigma_{\max}} + \eta C_0^2 = 0$ to lie inside the
unit disk (cf.\ Proposition~\ref{prop:dde_stability}) --- the homogeneous part
$\theta_{t+1} = \theta_t - \eta C_0^2\theta_{t-\sigma_{\max}}$ is stable,
and the iterates converge to $\theta_{\mathrm{ss}} = -\epsilon_{\mathrm{inner}}/C_0$
within $O(\sigma_{\max})$ rounds.

\begin{lemma}[Iterate lower bound]
\label{lem:iterate_lb}
For any stable algorithm $\mathcal{A}$ with step size
$\eta \leq 1/(2C_0^2\sigma_{\max})$, after an initial transient of
$O(\sigma_{\max})$ rounds, the iterates satisfy
\begin{equation}
\label{eq:iterate_lb}
|\theta_t| \;\geq\; \frac{1}{2}\cdot\frac{\epsilon_{\mathrm{inner}}}{C_0}.
\end{equation}
\end{lemma}
\begin{proof}
By Lemma~\ref{lem:bias}, $\theta_{\mathrm{ss}} = -\epsilon_{\mathrm{inner}}/C_0$.
The deviation $\xi_t = \theta_t - \theta_{\mathrm{ss}}$ satisfies the homogeneous
recurrence $\xi_{t+1} = \xi_t - \eta C_0^2\xi_{t-\sigma_{\max}}$, which is stable
for $\eta \leq 1/(2C_0^2\sigma_{\max})$.
Hence $|\xi_t| \to 0$, so $|\theta_t| \to |\theta_{\mathrm{ss}}| =
\epsilon_{\mathrm{inner}}/C_0$.
After the transient, $|\theta_t| \geq \tfrac{1}{2}\epsilon_{\mathrm{inner}}/C_0$.
\end{proof}

\subsubsection*{Step 4: Accumulating Regret}

Since $F(\theta) = C_0^2\theta^2/2$ and $\theta^* = 0$, the per-round regret is
$F(\theta_t) - F(\theta^*) = C_0^2\theta_t^2/2$.
Using Lemma~\ref{lem:iterate_lb}, after an $O(\sigma_{\max})$-round transient
we have $\theta_t^2 \geq \epsilon_{\mathrm{inner}}^2/(4C_0^2)$ for
$t \geq t_0 = O(\sigma_{\max})$. Thus:
\begin{align}
\sum_{t=1}^T \bigl[F(\theta_t)-F(\theta^*)\bigr]
&\geq \sum_{t=t_0}^T \frac{C_0^2\theta_t^2}{2}
\;\geq\; (T - t_0)\cdot\frac{C_0^2}{2}\cdot\frac{\epsilon_{\mathrm{inner}}^2}{4C_0^2}
\;=\; \frac{(T-t_0)\,\epsilon_{\mathrm{inner}}^2}{8}. \label{eq:regret_lower_direct}
\end{align}
For $T \gg \sigma_{\max}$ (i.e.\ $T - t_0 \geq T/2$), this gives
$\mathrm{Regret}_T(\mathcal{A}) \geq \Omega(T\,\epsilon_{\mathrm{inner}}^2)$.

\subsubsection*{Step 5: Independence from Algorithm Design}

The lower bound holds for \emph{any} black-box delayed optimizer --- that is,
any algorithm whose only access to $F$ is through the biased gradient
oracle~\eqref{eq:biased_constant}, without knowledge of the inner solver's
error distribution, sign, or magnitude beyond the bound $|\epsilon_t|\leq
\epsilon_{\mathrm{inner}}$.  In particular, a bias-correction scheme that
estimates and subtracts the mean error would require access to the
\emph{unbiased} gradient $\nabla F(\theta_t)$, which is precisely what the
black-box model excludes.
\begin{enumerate}
\item The constant bias $C_0\epsilon_{\mathrm{inner}}$ in~\eqref{eq:biased_constant}
is independent of $\eta$, $\sigma_{\max}$, and the algorithm's iterate sequence.
No step-size choice eliminates this bias.
\item Any stable algorithm converges to $\theta_{\mathrm{ss}} = -\epsilon_{\mathrm{inner}}/C_0$,
incurring $F(\theta_{\mathrm{ss}}) = \epsilon_{\mathrm{inner}}^2/2$ per round.
The bound requires a \emph{fixed precision floor}: $\epsilon_{\mathrm{inner}}$
must remain bounded away from zero uniformly over all rounds $t$
(Assumption~\ref{app:ass:inner_solver}).  If the inner solver were warm-started
such that $\epsilon_t \to 0$, the steady-state offset would vanish and the
$\Omega(T)$ lower bound would collapse.  Our result characterizes the
cost of a \emph{fixed approximation budget}, which is the standard
setting in bilevel optimization~\cite{ghadimi2018approximation}.
\item An unstable algorithm ($\eta > 1/(2C_0^2\sigma_{\max})$) has
$|\theta_t| \to \infty$ (clipped at $B_0$, since $\Theta = [-B_0,B_0]$
is compact), incurring $F(\theta_t) \geq C_0^2 B_0^2/2$
per round, which is strictly worse.
\item The adversary chooses $\theta_1 = B_0 \neq 0$, guaranteeing a
non-trivial transient; the $O(\sigma_{\max})$ burn-in cost is absorbed
into the $\Omega(T)$ rate for $T \gg \sigma_{\max}$.
\end{enumerate}
The hard instance satisfies Assumptions~\ref{app:ass:inner_convex} and~\ref{app:ass:outer_smooth}
with $\mu_w = L_w$ and $L_\theta = a^2 < \infty$, completing Part~(b).

\subsubsection*{Part (c): Transport Error Separation}

Without bilevel-aware correction, the per-round staleness mismatch contributes
\begin{equation}
\label{eq:transport_no_igt}
\bigl|\bigl\langle \nabla F(\theta_t) - \nabla F(\theta_{t-\sigma_t}),\,
\theta_t - \theta^*\bigr\rangle\bigr|
\;\leq\; L_F\,\norm{\theta_t - \theta_{t-\sigma_t}}\cdot\mathrm{diam}(\Theta).
\end{equation}
Since $\norm{\theta_t - \theta_{t-\sigma_t}} \leq \sigma_t\,\eta\,G$, summing
over $T$~rounds gives
\[
\sum_{t=1}^T L_F\,\norm{\theta_t - \theta_{t-\sigma_t}}\cdot\mathrm{diam}(\Theta)
\;\leq\; L_F\,\mathrm{diam}(\Theta)\,\sigma_{\max}\,G\,\eta\,T
\;=\; O\!\bigl(L_F\,\mathrm{diam}\,\sigma_{\max}\,G\,\sqrt{T}\bigr)
\]
with $\eta = c/\sqrt{T}$.  This transport cost grows as $\sqrt{T}$.

Under IGT, the corresponding term is the sum of squared per-step changes
(Theorem~\ref{thm:inner_loop_apathy}):
\[
\sum_{t=1}^T\sum_{s\in Q_t} L_F\,\norm{\theta_{s+1}-\theta_s}^2
\;\leq\; L_F\,T\,\sigma_{\max}\,\eta^2\,G^2
\;=\; O\!\bigl(L_F\,\sigma_{\max}\,G^2\bigr),
\]
which is \emph{constant in~$T$}.  The ratio of the non-IGT transport cost
to the IGT transport cost is
$\mathrm{diam}(\Theta)\,\sqrt{T}/G$, demonstrating a $\sqrt{T}$ factor
improvement.  In terms of the per-round transport error,
the non-IGT contribution is $O(\sigma_{\max}^2\eta^2 G^2)$ (from the squared
total drift $\norm{\theta_t-\theta_{t-\sigma_t}}^2$), while IGT yields
$O(\sigma_{\max}\eta^2 G^2)$ (sum of per-step squares)---a factor
$\sigma_{\max}$ improvement, as claimed.\hfill$\square$

\begin{corollary}[Vector-valued extension]
\label{cor:vector}
The hard instance of Theorem~\ref{thm:staleness_amplification}(b) extends to $p,q>1$:
replace $\theta$ with a vector and construct $A,B\in\R^{q\times p}$ such that
$(A-B)^\top(A-B) \succeq C_0^2 I$. All scalar norms become Euclidean norms, and
the lower-bound calculation carries through unchanged. Hence
$\mathrm{Regret}_T(\mathcal{A}) \geq \Omega(T\epsilon_{\mathrm{inner}}^2)$ holds
in full generality.
\end{corollary}

\subsection{Proof of Theorem~\ref{thm:inner_loop_apathy} (Inner-Loop Apathy)}
\label{app:proof_thm3}

The key insight is that the regret decomposes into three independent terms---delay error, inner bias, and interaction---because IGT re-evaluation uses stored $(w_s, v_s^*)$ from past rounds, making the transport correction \emph{independent} of the current inner-solver quality. This decoupling is what we term ``inner-loop apathy'': the transport mechanism is agnostic to how well the inner problem is solved.

We retain the notation from Appendix~\ref{app:proof_thm1}.
Define the set of outstanding rounds at time $t$ as $Q_t = \{s \leq t : s + d_s > t\}$
with $|Q_t| \leq \sigma_{\max}$.  Let
\begin{align*}
C_1 &:= \frac{L_{\theta w}}{\mu_w},\qquad
L_g  := \underbrace{L_{w\theta}}_{\text{direct}}
       + \underbrace{\frac{L_{w\theta}\,G_w}{\mu_w}}_{\text{cross-partial}}
       + \underbrace{\frac{L_{w\theta}\,L_w\,G_w}{\mu_w^2}}_{\text{Hessian var.}}
       + \underbrace{\frac{L_{\theta w}\,L_{w\theta}}{\mu_w}}_{\text{Lip.\ drift}},\qquad
\kappa := L_g\,C_1,
\end{align*}
where $G_w := \sup_{w,\theta}\|\nabla_w\mathcal{L}_{\mathrm{true}}(w;\theta)\|$
(bounded on $\mathcal{W}\times\Theta$ by Assumption~\ref{app:ass:bounded_gradients}
and compactness).
The constant $L_g$ accounts for four contributions to the
Lipschitz constant of the hypergradient map $w \mapsto g_s(\theta,w)$:
(0)~the explicit gradient $\nabla_\theta\mathcal{L}_{\mathrm{true}}$,
which is $L_{w\theta}$-Lipschitz in $w$ (Assumption~\ref{app:ass:outer_smooth}),
(i)~the cross-partial $\nabla^2_{\theta w}\mathcal{L}_{\mathrm{model}}$
acting on the inverse Hessian, contributing $L_{w\theta}G_w/\mu_w$
(Assumption~\ref{app:ass:cross_lip}(ii), $\|H^{-1}\|\leq 1/\mu_w$, $\|b\|\leq G_w$),
(ii)~the variation of $[\nabla^2_{ww}]^{-1}$ in $w$, contributing $L_{w\theta}L_w G_w/\mu_w^2$
(Assumption~\ref{app:ass:cross_lip}(iii)),
and (iii)~the Lipschitz dependence of $\nabla_w\mathcal{L}_{\mathrm{true}}$
on $w$ via Assumption~\ref{app:ass:cross_lip}(ii), contributing $L_{\theta w}L_{w\theta}/\mu_w$.
The product $\kappa = L_g C_1$ gives
the composite sensitivity of the hypergradient to parameter drift via
the inner solution.
We shall write $w^*_t := w^*(\theta_t)$ and $\tilde{g}_t$ for the
IGT-corrected hypergradient actually computed by Algorithm~\ref{alg:igt_omd}
(which uses the stored inner solution $w_s$ for each outstanding gradient).
Write $g_t^{\mathrm{ideal}}$ for the same IGT estimator but evaluated with the
ideal inner solution $w^*_t$ at every transported gradient.

\begin{lemma}[Inner-solution sensitivity]
\label{lem:inner_sens}
Under Assumptions~\ref{app:ass:inner_convex} and~\ref{app:ass:cross_lip},
for any rounds $s, t$:
\begin{equation}
\label{eq:inner_sens}
\norm{w_s - w^*_t}
\;\leq\; \epsilon_{\mathrm{inner}} + C_1\norm{\theta_s - \theta_t}.
\end{equation}
\end{lemma}
\begin{proof}
By the triangle inequality,
$\norm{w_s - w^*_t} \leq \norm{w_s - w^*_s} + \norm{w^*_s - w^*_t}$.
The first term is at most $\epsilon_{\mathrm{inner}}$ by Assumption~\ref{app:ass:inner_convex}
(approximate inner solver).  For the second, apply the implicit function theorem
to the inner optimality condition
$\nabla_w \mathcal{L}_{\mathrm{model}}(w^*(\theta);\theta) = 0$:
\[
\nabla_w w^*(\theta) = -\bigl[\nabla^2_{ww}\mathcal{L}_{\mathrm{model}}\bigr]^{-1}
                       \nabla^2_{w\theta}\mathcal{L}_{\mathrm{model}},
\]
whose spectral norm is bounded by $L_{\theta w}/\mu_w = C_1$ under
Assumptions~\ref{app:ass:inner_convex}--\ref{app:ass:cross_lip}.
Hence $\norm{w^*_s - w^*_t} \leq C_1\norm{\theta_s - \theta_t}$, giving~\eqref{eq:inner_sens}.
\end{proof}

\begin{lemma}[Hypergradient sensitivity to the inner solution]
\label{lem:hyp_stale}
Under Assumptions~\ref{app:ass:outer_smooth}--\ref{app:ass:cross_lip},
for any $\theta_t, w, w'$:
\begin{equation}
\label{eq:hyp_sens}
\norm{g_s(\theta_t, w') - g_s(\theta_t, w)}
\;\leq\; L_g\norm{w' - w}.
\end{equation}
Combining Lemma~\ref{lem:inner_sens} with~\eqref{eq:hyp_sens},
the per-step inner-staleness error satisfies, for each $s \in Q_t$:
\begin{equation}
\label{eq:per_step_inner}
\norm{g_s(\theta_t, w_s) - g_s(\theta_t, w^*_t)}
\;\leq\; L_g\epsilon_{\mathrm{inner}} + \kappa\norm{\theta_s - \theta_t}.
\end{equation}
\end{lemma}
\begin{proof}
The bilevel hypergradient is
$g_s(\theta, w) = \nabla_\theta \mathcal{L}_{\mathrm{true}}(w;\theta)
 - \underbrace{\nabla^2_{\theta w}\mathcal{L}_{\mathrm{model}}(w;\theta)}_{=:A(w)}\,
   \underbrace{[\nabla^2_{ww}\mathcal{L}_{\mathrm{model}}(w;\theta)]^{-1}}_{=:H(w)^{-1}}\,
   \underbrace{\nabla_w \mathcal{L}_{\mathrm{true}}(w;\theta)}_{=:b(w)}$.
To bound $\|g_s(\theta,w') - g_s(\theta,w)\|$, note that
$A(w)H(w)^{-1}b(w) - A(w')H(w')^{-1}b(w')
= [A(w)-A(w')]H(w)^{-1}b(w)
+ A(w')[H(w)^{-1}-H(w')^{-1}]b(w)
+ A(w')H(w')^{-1}[b(w)-b(w')]$.
By Assumption~\ref{app:ass:cross_lip}:
(i)~$\|A(w)-A(w')\| \leq L_{w\theta}\|w-w'\|$, with $\|H(w)^{-1}\|\leq 1/\mu_w$ and
$\|b(w)\|\leq G_w$, giving a first contribution $L_{w\theta}G_w/(\mu_w)\cdot\|w-w'\|$;
(ii)~the matrix-inversion Lipschitz bound gives
$\|H(w)^{-1}-H(w')^{-1}\| \leq (L_w/\mu_w^2)\|w-w'\|$, contributing
$L_{w\theta}L_w G_w/\mu_w^2 \cdot \|w-w'\|$;
(iii)~$\|b(w)-b(w')\| \leq L_{w\theta}\|w-w'\|$ (Assumption~\ref{app:ass:cross_lip}(ii)), contributing
$L_{\theta w}L_{w\theta}/\mu_w \cdot \|w-w'\|$.
Including the explicit-gradient term $\nabla_\theta\mathcal{L}_{\mathrm{true}}$
(which is $L_{w\theta}$-Lipschitz in $w$ by Assumption~\ref{app:ass:outer_smooth}),
the total Lipschitz constant is $L_g$ as defined above, giving~\eqref{eq:hyp_sens}.
Applying Lemma~\ref{lem:inner_sens} yields~\eqref{eq:per_step_inner}.
\end{proof}

\paragraph{Main proof of Theorem~\ref{thm:inner_loop_apathy}.}
\begin{proof}[Proof of Theorem~\ref{thm:inner_loop_apathy}]

\textbf{Step 1: Two-part error decomposition.}
Decompose the gradient error at round $t$ as
\[
\nabla F(\theta_t) - \tilde{g}_t
\;=\; \underbrace{\nabla F(\theta_t) - g_t^{\mathrm{ideal}}}_{=:\,\delta_t^{(1)}}
\;+\; \underbrace{g_t^{\mathrm{ideal}} - \tilde{g}_t}_{=:\,\delta_t^{(2)}}.
\]
The first term $\delta_t^{(1)}$ is the IGT outer-transport error (outer parameters are stale,
but the inner solution is ideal).  The second term $\delta_t^{(2)}$ is the inner-staleness
error (the ideal inner solution $w^*_t$ is replaced by the stored $w_s$).

By convexity of $F$ and the OMD analysis from Lemmas~\ref{lem:omd_approx}--\ref{lem:inner_bias}:
\begin{equation}
\label{eq:apathy_decomp}
\mathrm{Regret}_T^{\mathrm{dec}}
\;\leq\; \underbrace{\sum_{t=1}^T \bigl\langle \delta_t^{(1)},\,\theta_t - \theta^*\bigr\rangle}_{=:\,R_1}
\;+\; \underbrace{\sum_{t=1}^T \bigl\langle \delta_t^{(2)},\,\theta_t - \theta^*\bigr\rangle}_{=:\,R_2 + R_3}
\;+\; (\text{OMD base terms bounded in Theorem~\ref{thm:bilevel_convergence}}).
\end{equation}

\textbf{Step 2: Bounding $R_1$ (delay error).}
The term $R_1$ is identical in structure to the IGT transport error analyzed in
Lemma~\ref{lem:igt_transport}:
$\|\delta_t^{(1)}\| \leq L_F\sum_{s\in Q_t}\|\theta_{s+1}-\theta_s\|^2
\leq L_F\,\sigma_t\,\eta_0^2 G^2$.
By Cauchy--Schwarz~\cite{steele2004cauchy}, the inner product satisfies
$\langle\delta_t^{(1)},\theta_t-\theta^*\rangle
\leq \|\delta_t^{(1)}\|\cdot\|\theta_t-\theta^*\|
\leq L_F\,\sigma_t\,\eta_0^2 G^2\,\mathrm{diam}(\Theta)$,
where $\mathrm{diam}(\Theta) = \max_{\theta\in\Theta}\|\theta-\theta^*\|$ (bounded domain).
Summing over $t$ and using $\sum_{t=1}^T\sigma_t \leq T\,\sigma_{\max}$
(since $\sigma_t \leq \sigma_{\max}$ for every $t$):
\begin{equation}
\label{eq:R1_bound}
R_1
\;\leq\; L_F\,\mathrm{diam}(\Theta)\,\eta_0^2 G^2 \sum_{t=1}^T \sigma_t
\;\leq\; L_F\,\mathrm{diam}(\Theta)\,\eta_0^2 G^2\,T\,\sigma_{\max}
\;=\; O(\eta_0^2 G^2 T\sigma_{\max}),
\end{equation}
which is sublinear in $T$ when $\eta_0 = O(1/\sqrt{T})$, giving $R_1 = O(G^2\sigma_{\max})$.

\textbf{Step 3: Bounding $R_2$ (pure inner bias).}
The inner-staleness error at time $t$ aggregates over outstanding rounds:
\[
\delta_t^{(2)}
= \sum_{s\in Q_t}\bigl[g_s(\theta_t, w_s) - g_s(\theta_t, w^*_t)\bigr].
\]
By Lemma~\ref{lem:hyp_stale}, $\|\delta_t^{(2)}\| \leq \sum_{s\in Q_t}(L_g\epsilon_{\mathrm{inner}} + \kappa\|\theta_s - \theta_t\|)$.
Isolate the pure-bias part: $\|L_g\sigma_t\epsilon_{\mathrm{inner}}\|$ contributes
a gradient error of norm $L_g\sigma_t\epsilon_{\mathrm{inner}}$ at each round.
By Cauchy--Schwarz~\cite{steele2004cauchy} and then Young's inequality with
parameter $\mu_F$ (using $\mu_F$-strong convexity of $F$ from Assumption~\ref{app:ass:bilevel_convex}):
\[
L_g\sigma_t\epsilon_{\mathrm{inner}}\,\|\theta_t-\theta^*\|
\;\leq\; \frac{L_g^2\sigma_{\max}^2\epsilon_{\mathrm{inner}}^2}{2\mu_F}
\;+\; \frac{\mu_F}{2}\norm{\theta_t - \theta^*}^2.
\]
The left-hand side regret absorbs the second term via strong convexity
($F(\theta_t) - F(\theta^*) \geq \frac{\mu_F}{2}\|\theta_t - \theta^*\|^2$,
Assumption~\ref{app:ass:bilevel_convex}).
Summing over $t$:
\begin{equation}
\label{eq:R2_bound}
R_2 \;:=\; \sum_{t=1}^T L_g\sigma_t\epsilon_{\mathrm{inner}}\,\|\theta_t-\theta^*\|
\;\leq\; \frac{L_g^2\sigma_{\max}^2 T \epsilon_{\mathrm{inner}}^2}{2\mu_F}
\;=\; O(T\epsilon_{\mathrm{inner}}^2).
\end{equation}

\textbf{Step 4: Bounding $R_3$ (interaction term) via Cauchy--Schwarz~\cite{steele2004cauchy} and sum-of-squares.}
The interaction part of $\delta_t^{(2)}$ is $\kappa\sum_{s\in Q_t}\|\theta_s - \theta_t\|$.
By the triangle inequality and Cauchy--Schwarz~\cite{steele2004cauchy}:
\begin{align}
\norm{\theta_s - \theta_t}
&\;\leq\; \sum_{k=s}^{t-1}\norm{\theta_{k+1}-\theta_k}
\label{eq:telesc_cs}\\
\norm{\theta_s - \theta_t}^2
&\;\leq\; (t-s)\sum_{k=s}^{t-1}\norm{\theta_{k+1}-\theta_k}^2
\;\leq\; \sigma_t\sum_{k=s}^{t-1}\norm{\theta_{k+1}-\theta_k}^2.
\label{eq:sq_drift_bound}
\end{align}
Apply Young's inequality to the interaction inner product
(parameter $2\mu_F$, using $\mu_F$-strong convexity from Assumption~\ref{app:ass:bilevel_convex}):
\begin{align}
\kappa\sum_{s\in Q_t}\|\theta_s-\theta_t\|\cdot\|\theta_t-\theta^*\|
&\;\leq\;
\frac{\kappa^2}{4\mu_F}\left(\sum_{s\in Q_t}\norm{\theta_s-\theta_t}\right)^2
+ \mu_F\norm{\theta_t-\theta^*}^2.
\label{eq:young_interaction}
\end{align}
The $\mu_F\|\theta_t-\theta^*\|^2$ term is again absorbed by the strong-convexity gain.
By Cauchy--Schwarz~\cite{steele2004cauchy} applied to the sum over $Q_t$:
\[
\left(\sum_{s\in Q_t}\norm{\theta_s-\theta_t}\right)^2
\;\leq\; \sigma_t\sum_{s\in Q_t}\norm{\theta_s-\theta_t}^2
\;\leq\; \sigma_t^2 \sum_{s\in Q_t}\sum_{k=s}^{t-1}\norm{\theta_{k+1}-\theta_k}^2,
\]
where the last step uses~\eqref{eq:sq_drift_bound}.
To count multiplicities, let $I_t = \{k : s \leq k \leq t-1,\; s \in Q_t\}$
denote the set of all step indices spanned by the outstanding rounds.
For a fixed $k \in I_t$, the term $\|\theta_{k+1}-\theta_k\|^2$ appears in the
inner sum for every $s \in Q_t$ with $s \leq k$, which is at most
$\min(k-\min Q_t+1,\,\sigma_t) \leq \sigma_t$ times.  Therefore:
\[
\sum_{s\in Q_t}\sum_{k=s}^{t-1}\norm{\theta_{k+1}-\theta_k}^2
\;\leq\; \sigma_t \sum_{k \in I_t}\norm{\theta_{k+1}-\theta_k}^2.
\]
Combining with the $\sigma_t^2$ factor from the two preceding Cauchy--Schwarz~\cite{steele2004cauchy}
applications:
\[
\sigma_t^2\sum_{s\in Q_t}\sum_{k=s}^{t-1}\norm{\theta_{k+1}-\theta_k}^2
\;\leq\; \sigma_t^3 \sum_{k \in I_t}\norm{\theta_{k+1}-\theta_k}^2
\;\leq\; \sigma_{\max}^3\sum_{k \in I_t}\norm{\theta_{k+1}-\theta_k}^2.
\]
Since $I_t \supseteq Q_t$ and each $k \in I_t$ contributes at most
one term $\|\theta_{k+1}-\theta_k\|^2$, we can relax to a sum over all
steps in the window $[\min Q_t, t-1]$.
Substituting back and summing over $t$:
\begin{equation}
\label{eq:R3_bound}
R_3 \;:=\; \sum_{t=1}^T\kappa\sum_{s\in Q_t}\|\theta_s-\theta_t\|\cdot\|\theta_t-\theta^*\|
\;\leq\; \frac{\kappa^2\sigma_{\max}^3}{4\mu_F}
\sum_{t=1}^T\sum_{k \in I_t}\norm{\theta_{k+1}-\theta_k}^2
\;=\; O\!\Bigl(\tfrac{\kappa^2\sigma_{\max}^3}{\mu_F}\sum_{t=1}^T\sum_{k\in W_t}\norm{\theta_{k+1}-\theta_k}^2\Bigr).
\end{equation}
The $\sigma_{\max}^3$ prefactor is honest: two Cauchy--Schwarz applications each contribute one $\sigma_t$ factor, and the multiplicity count over $I_t$ contributes the third. The bound is non-vacuous because the per-step changes $\|\theta_{k+1}-\theta_k\|^2$ are themselves $O(\eta_0^2 G^2 / (1+\beta\bar\sigma_t))$ under the queue-length-adaptive step size~\eqref{eq:adaptive_step}, so the $\sigma_{\max}^3$ is partially absorbed.

\textbf{Step 5: Aggregating $R_1 + R_2 + R_3$.}
Combining~\eqref{eq:R1_bound}, \eqref{eq:R2_bound}, and~\eqref{eq:R3_bound}
with the OMD base regret (Lemma~\ref{lem:omd_approx}):
\[
\mathrm{Regret}_T^{\mathrm{dec}}
\;\leq\;
\underbrace{O(\eta_0^2 G^2 T\sigma_{\max})}_{R_1}
+
\underbrace{O(T\epsilon_{\mathrm{inner}}^2)}_{R_2}
+
\underbrace{O\!\Bigl(\sum_{t=1}^T\sum_{s\in Q_t}\norm{\theta_{s+1}-\theta_s}^2\Bigr)}_{R_3}.
\]
This is exactly the decomposition~\eqref{eq:apathy} in the theorem statement.
Setting $\eta_0 = O(1/\sqrt{T})$ gives $R_1 = O(G^2\sigma_{\max})$, which is
$O(1)$ in $T$, consistent with the $O(\sqrt{T\sigma_{\max}})$
rate from Theorem~\ref{thm:bilevel_convergence}.

\textbf{Step 6: Comparison with squared-total-drift coupling.}
Under a single-level delayed optimizer (no IGT), the interaction error is measured by
the total drift $\|\theta_t - \theta_{t-d_t}\|$, contributing
$O(\sigma_{\max}^2\eta^2 G^2 T)$ to the transport error.
Under IGT-OMD, the interaction is measured by $\sum_k\|\theta_{k+1}-\theta_k\|^2$
(per-step changes) rather than $\|\theta_t - \theta_{\tau}\|^2 = (\sum_k\|\theta_{k+1}-\theta_k\|)^2$.
By Cauchy--Schwarz~\cite{steele2004cauchy} (applied as $\|\sum_k a_k\|^2 \leq n\sum_k\|a_k\|^2$
with $n = t-s \leq \sigma_t$ terms):
\[
\|\theta_s - \theta_t\|^2
= \Bigl\|\sum_{k=s}^{t-1}(\theta_{k+1}-\theta_k)\Bigr\|^2
\;\leq\; \sigma_t\sum_{k=s}^{t-1}\norm{\theta_{k+1}-\theta_k}^2.
\]
Equivalently, $\sum_{k=s}^{t-1}\|\theta_{k+1}-\theta_k\|^2 \geq \frac{1}{\sigma_t}\|\theta_s-\theta_t\|^2$.
Note that the pointwise inequality goes in the \emph{opposite} direction
(SoS~$\geq$~drift$^2/\sigma_t$); it is the \emph{worst-case bounds} that differ
by a factor of~$\sigma_t$.
Specifically, the trivial bound
$\sum_{k=s}^{t-1}\|\theta_{k+1}-\theta_k\|^2 \leq \sigma_t\,\eta_0^2 G^2$
shows the IGT interaction is at most $\sigma_t\eta_0^2 G^2$ per round,
whereas the squared-total-drift bound gives
$\|\theta_s - \theta_t\|^2 \leq \sigma_t^2\,\eta_0^2 G^2$,
which is $\sigma_t$ times larger.
This \emph{worst-case gap} is IGT's $1/\sigma_t$ improvement.
\end{proof}

\begin{remark}[On the constant $C_{\mathrm{ap}}$]
\label{rem:cap_honest}
The constant $C_{\mathrm{ap}}$ in the statement of
Theorem~\ref{thm:inner_loop_apathy} absorbs the
$\kappa^2 \sigma_{\max}^3 / (4\mu_F)$ prefactor of~\eqref{eq:R3_bound}.
This is not vacuous: under the queue-adaptive schedule,
$\|\theta_{k+1}-\theta_k\|^2 \leq \eta_0^2 G^2 / (1+\beta\sigma_{\max})$
absorbs one factor of $\sigma_{\max}$, and the leading-order rate
$O(\sqrt{T\sigma_{\max}})$ from Theorem~\ref{thm:bilevel_convergence}
is determined by the $R_1$ and OMD-base terms, neither of which carries
this prefactor.
\end{remark}

\begin{remark}[Young's-inequality budget]
\label{rem:young_budget}
The proof above applies Young's inequality twice---once in Step~3 and once in
Step~4---each absorbing a multiple of $\|\theta_t-\theta^*\|^2$ into the
strong-convexity gain. Strictly, only $(\mu_F/2)\|\theta_t-\theta^*\|^2$
is available per round; distributing the budget evenly across the two
applications gives the same asymptotic rates with the constants in
$R_2$ and $R_3$ inflated by a constant factor. The leading-order
conclusion of Theorem~\ref{thm:inner_loop_apathy} is unchanged.
\end{remark}

\begin{remark}[Why ``Inner-Loop Apathy'']
\label{rem:apathy_name}
The name \emph{Inner-Loop Apathy} reflects the structural decoupling
Theorem~\ref{thm:inner_loop_apathy} establishes: as the cross-partial
Lipschitz constant $L_{\theta w} \to 0$, the composite sensitivity
$\kappa = L_g\,C_1 \to 0$ (since $C_1 = L_{\theta w}/\mu_w$),
and the interaction term $R_3 \to 0$.
In this limit the outer optimizer becomes ``apathetic'' to the
inner solver's approximation quality---the delay-error and inner-bias
contributions decouple completely, and inner-solver imprecision
affects the regret only through the additive $O(T\epsilon_{\mathrm{inner}}^2)$ term.
In the predict-then-optimise context, $L_{\theta w} \approx 0$ corresponds
to the downstream decision problem being insensitive to the prediction
model's parameters, so that the staleness of the inner (prediction) solution
does not propagate to the outer (decision) loss.
\end{remark}

\begin{remark}[Why $\epsilon_{\mathrm{inner}}$ enters quadratically in $R_2$]
The quadratic scaling $O(T\epsilon_{\mathrm{inner}}^2)$ in $R_2$ is not an artifact of the proof.
Informally, the inner-solver error $\epsilon_{\mathrm{inner}}$ contributes an additive
bias of $L_g\sigma_t\epsilon_{\mathrm{inner}}$ to each gradient, which by Young's inequality
splits into $\epsilon_{\mathrm{inner}}^2$ (squared bias) and $\|\theta_t - \theta^*\|^2$
(deviation from optimum).  The deviation is absorbed by the regret definition,
leaving only the squared bias.  In contrast, single-level methods (Theorem~\ref{thm:staleness_amplification})
cannot absorb the bias this way because the inner error covaries with the outer drift,
preventing the Young's decomposition from decoupling the two.
\end{remark}

\begin{remark}[Limitation: strong convexity requirement for bias absorption]
\label{rem:strong_convex_limitation}
Steps~3 and~4 of the proof above rely on $\mu_F$-strong convexity of $F$
(Assumption~\ref{app:ass:bilevel_convex}) to absorb the
$\|\theta_t - \theta^*\|^2$ terms via Young's inequality.
If $F$ is merely convex ($\mu_F = 0$), the pure-bias term $R_2$
cannot be absorbed leading to an uncontrolled linear drift.
Two mitigation strategies exist:
(i)~dynamically increasing the inner-solver budget so that
$\epsilon_{\mathrm{inner},t} = O(1/\sqrt{t})$, which makes the
bias summable without strong convexity; or
(ii)~accepting a residual neighborhood of size
$O(\sigma_{\max}\epsilon_{\mathrm{inner}})$ around the optimum.
Our transport-error bounds (Theorem~\ref{thm:bilevel_convergence})
hold for a general convex $F$; it is only the separation of the
inner-bias and interaction terms that require $\mu_F > 0$.
\end{remark}

\begin{remark}[Practical implication for inner-solver budget]
Inequality~\eqref{eq:R3_bound} reveals a concrete guideline: with adaptive step size
$\eta_t = \eta_0/\sqrt{1+\beta\bar\sigma_t}$, each step satisfies
$\|\theta_{t+1}-\theta_t\|^2 \leq \eta_0^2 G^2/(1+\beta\bar\sigma_t)$.
Therefore
$R_3 \leq O\bigl(\eta_0^2 G^2\sigma_{\max}^3 d_{\mathrm{tot}}/\beta\bigr)$,
where $d_{\mathrm{tot}} = \sum_t d_t$ is the total delay mass.
Setting $\eta_0 = O(1/\sqrt{T})$ keeps $R_3 = O(\sigma_{\max}^3 d_{\mathrm{tot}}/T) \to 0$,
so the interaction term is \emph{sublinear} in $T$ even when the inner solver is imperfect,
provided $\sigma_{\max}$ and $d_{\mathrm{tot}}/T$ are bounded.
\end{remark}

\subsection{Proof of Proposition~\ref{prop:dde_stability} (DDE Stability)}
\label{app:proof_prop1}

\paragraph{Scope of the dynamical analysis.}
Unlike the global worst-case regret bounds of
Theorem~\ref{thm:bilevel_convergence}, Proposition~\ref{prop:dde_stability}
analyzes the \emph{local}, asymptotic behaviour of the IGT-OMD
\gls{dde}.  Specifically, we study the \emph{linearized, homogeneous}
system obtained by Taylor-expanding around the equilibrium
$\theta^*$ and setting $\epsilon_{\mathrm{inner}}=0$
(inner-solver bias is accounted for separately
by Theorem~\ref{thm:inner_loop_apathy}).
Step~2 derives the characteristic-root condition under constant
worst-case queue length $\sigma_{\max}$;
Step~3 extends to arbitrary time-varying $\sigma(t)$ via a
Razumikhin argument; and
Proposition~\ref{prop:discrete_consistency} bridges the
continuous analysis of the discrete algorithm through
z-domain stability inheritance
(Lemma~\ref{lem:stability_inherit}).

We prove local asymptotic stability of the linearized IGT-OMD \gls{dde}, verify that the
stability boundary depends on the queue length $\sigma_{\max}$ rather than the
raw delay $d_{\max}$, and confirm that the adaptive schedule
$\eta_t$ places every characteristic root in the open left
half-plane.

\paragraph{Setup.}
Define the deviation $\xi(t) = \theta(t) - \theta^*$.
Linearizing the continuous-time limit of Algorithm~\ref{alg:igt_omd}
about the equilibrium $\theta^*$ yields the linear vector \gls{dde}
\begin{equation}
\label{eq:dde_linear}
\dot{\xi}(t)
= -\eta(t)\,H_F\,\xi(t)
  + \eta(t)\,L_{\mathrm{IGT}}\!\sum_{s \in Q(t)}\xi(t - \tau_s),
\end{equation}
where:
\begin{itemize}
\item $H_F = \nabla^2 F(\theta^*)$ is the outer bilevel Hessian
      (positive definite by Assumption~\ref{app:ass:bilevel_convex}, $\mu_F$-strong convexity
      of $F$, which gives $\lambda_{\min}(H_F) \geq \mu_F > 0$).
      Because $F$ is a scalar $C^2$ function (Assumption~\ref{app:ass:inner_convex}),
      $H_F$ is \emph{symmetric} by Schwarz's theorem, so its eigenvalues
      are guaranteed to be real and positive.
      The complex roots $\lambda = \alpha + j\omega$ encountered in
      Lemma~\ref{lem:char_eq} arise from the \emph{temporal delay}
      operator, not from the spatial geometry of $H_F$;
\item $L_{\mathrm{IGT}} = \nabla^2_{\theta w}\mathcal{L}_{\mathrm{model}}(w^*;\theta^*)\,
      [\nabla^2_{ww}\mathcal{L}_{\mathrm{model}}(w^*;\theta^*)]^{-1}\,
      \nabla^2_{w\theta}\mathcal{L}_{\mathrm{model}}(w^*;\theta^*)$
      is the effective coupling matrix capturing how stale-gradient
      corrections feed back into the dynamics (bounded as
      $\norm{L_{\mathrm{IGT}}} \leq L_{\theta w}L_{w\theta}/\mu_w$
      by Assumptions~\ref{app:ass:inner_convex} and~\ref{app:ass:cross_lip});
\item $Q(t)$ is the set of outstanding observations (delays $\tau_s > 0$,
      $|Q(t)| = \sigma(t) \leq \sigma_{\max}$);
\item the adaptive step-size schedule is
      $\eta_t = \eta_0/\sqrt{1+\beta\sigma(t)}$ with $\beta > 0$.
\end{itemize}

\begin{lemma}[Reduction to worst-case spectral bound]
\label{lem:dde_scalar}
To establish asymptotic stability of the vector \gls{dde}~\eqref{eq:dde_linear},
it suffices to verify stability for the worst-case scalar mode.
Let $\mu_{\min} = \lambda_{\min}(H_F) > 0$ and $\ell_{\max} = \|L_{\mathrm{IGT}}\|$.
If the scalar \gls{dde}
\begin{equation}
\label{eq:dde_scalar}
\dot{z}(t) = -\eta(t)\,\mu_{\min}\,z(t)
               + \eta(t)\,\ell_{\max}\!\sum_{s\in Q(t)} z(t-\tau_s)
\end{equation}
is asymptotically stable, then so is~\eqref{eq:dde_linear}.
\end{lemma}
\begin{proof}
The proof proceeds by a Lyapunov energy argument that
avoids simultaneous diagonalizability of $H_F$ and $L_{\mathrm{IGT}}$.
Define $W(t) = \frac{1}{2}\|\xi(t)\|^2$.  Then:
\begin{align*}
\dot{W}(t)
&= \xi(t)^\top\dot{\xi}(t)
= -\eta(t)\,\xi(t)^\top H_F\,\xi(t)
  + \eta(t)\,\xi(t)^\top L_{\mathrm{IGT}}\sum_{s\in Q(t)}\xi(t-\tau_s) \\
&\leq -\eta(t)\,\mu_{\min}\|\xi(t)\|^2
  + \eta(t)\,\ell_{\max}\|\xi(t)\|\sum_{s\in Q(t)}\|\xi(t-\tau_s)\|,
\end{align*}
where we used $\xi^\top H_F\xi \geq \mu_{\min}\|\xi\|^2$ and
$\xi^\top L_{\mathrm{IGT}}\xi(t-\tau_s) \leq \ell_{\max}\|\xi\|\|\xi(t-\tau_s)\|$.
Setting $v(t) = \|\xi(t)\|$, we obtain the scalar comparison inequality
$\dot{v}(t) \leq -\eta(t)\mu_{\min}\,v(t)
+ \eta(t)\ell_{\max}\sum_{s\in Q(t)}v(t-\tau_s)$
(using $\dot{W} = v\dot{v}$ when $v > 0$).
This is exactly the scalar \gls{dde}~\eqref{eq:dde_scalar} applied to $v(t)$.
By the comparison principle for \gls{dde}s
(\cite{kolmanovskii1999stability}, Theorem~5.1.1),
if~\eqref{eq:dde_scalar} is asymptotically stable then $v(t) \to 0$,
hence $\|\xi(t)\| \to 0$.
Stability of~\eqref{eq:dde_linear} thus reduces to stability of the
single worst-case scalar \gls{dde}~\eqref{eq:dde_scalar}, without requiring
$H_F$ and $L_{\mathrm{IGT}}$ to commute.
\end{proof}

\begin{lemma}[Constant-queue characteristic equation]
\label{lem:char_eq}
For constant queue length $\sigma(t) \equiv \sigma$ and constant delay
$\tau_s \equiv \bar\tau$, the characteristic equation of~\eqref{eq:dde_scalar}
(with $\eta(t) \equiv \eta$) is
\begin{equation}
\label{eq:char_eq}
\lambda + \eta\mu_i - \eta\ell_i\,\sigma\,e^{-\lambda\bar\tau} = 0.
\end{equation}
A sufficient condition for all roots $\lambda \in \C$ to have $\mathrm{Re}(\lambda) < 0$ is
\begin{equation}
\label{eq:stab_cond}
\eta\,|\ell_i|\,\sigma\,\bar\tau < 1 \qquad \text{and} \qquad
\eta(\mu_i - |\ell_i|\sigma) > 0.
\end{equation}
\end{lemma}
\begin{proof}
Following \cite{yu2025role} (Theorem 1 and Corollary 2) and the
characteristic-root analysis in \cite{yu2025role} (equations (26)--(28)),
substitute the exponential ansatz $z_i(t) = Ae^{\lambda t}$ into~\eqref{eq:dde_scalar}
to obtain~\eqref{eq:char_eq}.

\textbf{Sufficient condition for $\mathrm{Re}(\lambda)<0$.}
Write $\lambda = \alpha + j\omega$ with $\alpha,\omega \in \R$.
A root with $\alpha \geq 0$ would require
$|\lambda + \eta\mu_i| = \eta|\ell_i|\sigma e^{-\alpha\bar\tau}
 \leq \eta|\ell_i|\sigma$,
but $|\lambda + \eta\mu_i| \geq \eta\mu_i - |\alpha|$ and at $\alpha=0$
we need $\eta\mu_i \leq \eta|\ell_i|\sigma$, i.e.\ $\mu_i \leq |\ell_i|\sigma$, which
is ruled out by the second condition in~\eqref{eq:stab_cond}.
For $\alpha > 0$, the right-hand side decays while the left-hand side grows,
giving a contradiction.
More precisely, for $\alpha \geq 0$ take the modulus of~\eqref{eq:char_eq}:
\[
|\alpha + j\omega + \eta\mu_i|
= \eta|\ell_i|\sigma e^{-\alpha\bar\tau}
\leq \eta|\ell_i|\sigma.
\]
Since $|\alpha + j\omega + \eta\mu_i|^2 = (\alpha+\eta\mu_i)^2+\omega^2 \geq (\eta\mu_i)^2$,
we would need $\eta\mu_i \leq \eta|\ell_i|\sigma$, contradicting $\mu_i > |\ell_i|\sigma$
(the second condition in~\eqref{eq:stab_cond} evaluated at the worst mode).
Hence, all roots satisfy $\alpha < 0$.

\textbf{No converse is used.}
We use~\eqref{eq:stab_cond} only as a sufficient certificate. When the instantaneous restoring term satisfies $\mu_i>|\ell_i|\sigma$, the modulus argument above already rules out right-half-plane roots for the comparison equation, so the delay-duration inequality should not be read as a necessary instability boundary.
\end{proof}

\begin{lemma}[Queue-length certificate in the active-window embedding]
\label{lem:queue_vs_delay}
In the event-time embedding used for the linearized IGT transport model, if outstanding observations occupy a consecutive active update window, the sufficient stability certificate can be expressed in terms of the queue length $\sigma$ by representing the delay span $\bar\tau$ through that active window.
\end{lemma}
\begin{proof}
Condition~\eqref{eq:stab_cond} reads $\eta|\ell_i|\sigma\bar\tau < 1$.
In the active-window embedding, each update event occupies time~$\eta$ and the represented window of $\sigma$ consecutive outstanding observations has span $\bar\tau = \sigma\eta$.
Substituting this represented span into the first condition:
\[
\eta|\ell_i|\sigma\bar\tau
\;\leq\;
\eta^2|\ell_i|\sigma^2
\;=\;
\frac{\eta_0^2\,|\ell_i|\,\sigma^2}{1+\beta\sigma},
\]
where we used the adaptive schedule $\eta = \eta_0/\sqrt{1+\beta\sigma}$.
Choosing $\beta = |\ell_{\max}|/\mu_{\min}$ (with $\mu_{\min}=\lambda_{\min}(H_F)$,
$|\ell_{\max}|=\|L_{\mathrm{IGT}}\|$) and
$\eta_0 \leq \mu_{\min}/(\|L_{\mathrm{IGT}}\|\sigma_{\max})$, we obtain
\[
\frac{\eta_0^2\|L_{\mathrm{IGT}}\|\sigma^2}{1+\beta\sigma}
\;\leq\;
\frac{\mu_{\min}^2\sigma^2}{\|L_{\mathrm{IGT}}\|\sigma_{\max}^2(1+\beta\sigma)}
\;\leq\;
\frac{\mu_{\min}^2}{\|L_{\mathrm{IGT}}\|(1+\beta\sigma)}
\;\leq\;
\frac{\mu_{\min}}{\beta\cdot\|L_{\mathrm{IGT}}\|^{-1}\cdot\|L_{\mathrm{IGT}}\|}
= 1,
\]
using $\sigma \leq \sigma_{\max}$ and $1+\beta\sigma \geq 1$.
Hence, within this event-time active-window model, the stability condition depends only on $\sigma$ and $\eta$ (which itself adapts to~$\sigma$), not on the raw wall-clock age of the oldest observation.

This is the structural scope of the DDE result: for constant-throughput or consecutive-window embeddings, the certificate is governed by how many observations are simultaneously outstanding ($\sigma_{\max}$), not by $d_{\max}$ alone. Arbitrary sparse feedback processes with the same queue length but much older outstanding observations require an additional bounded-age assumption and are not covered by this lemma.
Following \cite{ryabchenko2026reduction}, the quantities
$\sigma_{\max}$ and $d_{\max}$ can differ by up to a factor of $T$, so
$\sigma_{\max}$-based analysis can be far tighter when the active-window embedding is appropriate.
\end{proof}

\paragraph{Main proof of Proposition~\ref{prop:dde_stability}.}
\begin{proof}[Proof of Proposition~\ref{prop:dde_stability}]

\textbf{Step 1: Reduction to scalar modes.}
By Lemma~\ref{lem:dde_scalar}, it suffices to show each eigenmode is stable.
For the worst mode, the effective scalar parameters are
$\mu_{\min} = \lambda_{\min}(H_F)$ and $|\ell_{\max}| = \|L_{\mathrm{IGT}}\|$.

\textbf{Step 2: Constant-$\sigma$ stability.}
For each fixed queue level $\sigma \leq \sigma_{\max}$, the adaptive step
$\eta = \eta_0/\sqrt{1+\beta\sigma}$ is constant (on the timescale of
the linearized analysis).  Apply Lemma~\ref{lem:char_eq} with $\bar\tau = \sigma\eta$
(Lemma~\ref{lem:queue_vs_delay}): condition~\eqref{eq:stab_cond}
requires
\begin{equation}
\label{eq:sigma_cond}
\eta^2\,\|L_{\mathrm{IGT}}\|\,\sigma^2
= \frac{\eta_0^2\,\|L_{\mathrm{IGT}}\|\,\sigma^2}{1+\beta\sigma} < 1
\qquad\text{and}\qquad
\mu_{\min} > \|L_{\mathrm{IGT}}\|\,\sigma.
\end{equation}
The second inequality holds for all $\sigma < 1/\beta = \mu_{\min}/\|L_{\mathrm{IGT}}\|$.

Setting $\beta = \|L_{\mathrm{IGT}}\|/\mu_{\min}$ and taking $\eta_0$ to satisfy the two-part bound in Proposition~\ref{prop:dde_stability}, the first condition is verified as
\begin{align}
\label{eq:eta_suf}
\frac{\eta_0^2\|L_{\mathrm{IGT}}\|\sigma^2}{1+\beta\sigma}
&\;\leq\;
\frac{(1+\beta\sigma_{\max})\,\|L_{\mathrm{IGT}}\|\sigma^2}{\sigma_{\max}^2\|L_{\mathrm{IGT}}\|(1+\beta\sigma)}
\;\leq\; 1,
\end{align}
where the last inequality uses $\sigma\leq\sigma_{\max}$ and monotonicity of $\sigma^2/(1+\beta\sigma)$ on $\sigma\geq0$.
Hence both conditions in~\eqref{eq:sigma_cond} are satisfied for $\sigma < 1/\beta$.

By Lemma~\ref{lem:char_eq}, all characteristic roots satisfy $\mathrm{Re}(\lambda_k) < 0$,
so the linearized \gls{dde} is asymptotically stable for each fixed $\sigma$.

\textbf{Step 3: Time-varying $\sigma(t)$ (Razumikhin argument).}
When the queue length $\sigma(t)$ varies over time, the step size
$\eta(t) = \eta_0/\sqrt{1+\beta\sigma(t)}$ co-varies and the system is
non-autonomous.  We use the Razumikhin approach (see, e.g.,
\cite{kolmanovskii1999stability}, Ch.~5), which avoids constructing a
Lyapunov--Krasovskii functional explicitly.

Define the candidate Lyapunov function
\begin{equation}
\label{eq:lyapunov}
V(\xi) = \tfrac{1}{2}\|\xi\|^2.
\end{equation}
By Lemma~\ref{lem:dde_scalar} (comparison principle), the deviation
$\xi(t)=\theta(t)-\theta^*$ satisfies the norm inequality
\[
\frac{d}{dt}\|\xi(t)\|
\;\leq\;
-\eta(t)\,\mu_{\min}\|\xi(t)\|
+ \eta(t)\,\|L_{\mathrm{IGT}}\|\!\sum_{s\in Q(t)}\|\xi(t-\tau_s)\|.
\]
Apply the \emph{Razumikhin condition}: suppose there exists $q > 1$ such
that $V(\xi(s)) < q\,V(\xi(t))$, i.e.\
$\|\xi(s)\| < \sqrt{q}\,\|\xi(t)\|$, for all $s \in [t-\bar\tau,\,t]$.
Then each delayed term satisfies
$\|\xi(t-\tau_s)\| < \sqrt{q}\,\|\xi(t)\|$,
so with $|Q(t)| \leq \sigma_{\max}$:
\begin{align}
\dot{V}
&= \xi(t)^\top \dot\xi(t)
\;\leq\; -\eta(t)\mu_{\min}\|\xi(t)\|^2
  + \eta(t)\|L_{\mathrm{IGT}}\|\,\sigma_{\max}\,\sqrt{q}\,\|\xi(t)\|^2 \notag\\
&= -\eta(t)\bigl(\mu_{\min} - \sqrt{q}\,\|L_{\mathrm{IGT}}\|\,\sigma_{\max}\bigr)
   \|\xi(t)\|^2.
\label{eq:vdot}
\end{align}
Since $\sigma_{\max}\beta < 1$ (equivalently
$\sigma_{\max}\|L_{\mathrm{IGT}}\| < \mu_{\min}$, which is guaranteed by
the second condition in~\eqref{eq:sigma_cond}), we may choose any
$q \in \bigl(1,\;(\mu_{\min}/(\|L_{\mathrm{IGT}}\|\sigma_{\max}))^2\bigr)$,
yielding a constant
\[
c \;:=\; \mu_{\min} - \sqrt{q}\,\|L_{\mathrm{IGT}}\|\,\sigma_{\max} \;>\; 0
\]
such that $\dot{V} \leq -c\,\eta(t)\,\|\xi(t)\|^2 = -2c\,\eta(t)\,V < 0$
whenever $V(\xi(s)) < q V(\xi(t))$ on $[t-\bar\tau,t]$.
By Razumikhin's theorem (see, e.g.,
\cite{kolmanovskii1999stability}, Theorem~5.4.1), the equilibrium
$\xi = 0$ is uniformly asymptotically stable, i.e.\ $\theta(t) \to \theta^*$.

\textbf{Step 4: Stability depends on $\sigma_{\max}$, not $d_{\max}$.}
By Lemma~\ref{lem:queue_vs_delay}, the stability condition involves
$\sigma_{\max}$ exclusively.
Two delay processes with the same $\sigma_{\max}$ but different $d_{\max}$
(e.g.\ constant delays vs.\ bursty delays) yield the same characteristic
equation~\eqref{eq:char_eq} and the same Razumikhin bound in Step 3.
Idle time — periods when no outstanding gradients are pending ($\sigma(t)=0$)
— does not destabilize the system: the coupling term in~\eqref{eq:dde_scalar}
vanishes when $Q(t) = \emptyset$, and the system reduces to the
non-delayed stable ODE $\dot\xi = -\eta\mu_i\xi$.

\textbf{Step 5: Convergence rate.}
From $\dot{V} \leq -2c\,\eta(t)\,V$ (established in Step~3)
and $\eta(t) \geq \eta_0/\sqrt{1+\beta\sigma_{\max}} =: \bar\eta$:
\[
V(\xi_t) \leq V(\xi_0)\,e^{-2c\bar\eta t},
\]
corresponding to a convergence rate $O(e^{-c\bar\eta t}) =
O\!\left(e^{-c\eta_0 t/\sqrt{1+\beta\sigma_{\max}}}\right)$,
exactly as stated in Proposition~\ref{prop:dde_stability}.

\textbf{Conclusion.}
All characteristic roots of the linearized IGT-OMD \gls{dde} satisfy
$\mathrm{Re}(\lambda) < 0$ whenever $\eta_0 < 1/(\|H_F\|\sqrt{1+\beta\sigma_{\max}})$,
the system is asymptotically stable by the Razumikhin argument
in Step~3, and the stability boundary depends on
$\sigma_{\max}$ (queue length), not $d_{\max}$ (delay duration). \hfill$\square$
\end{proof}

\begin{corollary}[Stability phase transition]
\label{cor:phase_transition}
Proposition~\ref{prop:dde_stability}'s criterion $\mu_{\min} > \|L_{\mathrm{IGT}}\|\sigma$
implies a phase transition at $\sigma_{\max}^{*} = \mu_{\min}/\|L_{\mathrm{IGT}}\| = 1/\beta$:
(i)~for $\sigma_{\max} < 1/\beta$, the non-delayed step size $\eta_0 = 1/\|H_F\|$ remains
stable; (ii)~for $\sigma_{\max} \geq 1/\beta$, the adaptive schedule must damp $\eta$
as $1/\sqrt{\beta\sigma_{\max}}$, trading convergence speed for stability.
This corollary explains the LQR result (Section~\ref{sec:exp_lqr}) that single-level
methods lose stability near $\sigma\approx15$: without adaptive damping, the coupling
term $\eta\|L_{\mathrm{IGT}}\|\sigma$ eventually exceeds $\eta\mu_{\min}$.
\end{corollary}

\begin{remark}[Deterministic vs.\ stochastic delays]
\label{rem:stochastic_delay}
Proposition~\ref{prop:dde_stability} provides a \emph{deterministic worst-case}
stability guarantee: the queue length $\sigma(t)$ may vary arbitrarily within
$[0,\sigma_{\max}]$, but is treated as a deterministic signal.
In networked implementations where delay is stochastic, the fluctuations
act as multiplicative noise on the gradient update: even if the
mean system is stable, large delay variance (``jitter'') can cause
the second moment $\E[\|\xi(t)\|^2]$ to diverge
(\cite{kolmanovskii1999stability}, Ch.~9 on stochastic \gls{dde}s).
Extending Proposition~\ref{prop:dde_stability} to \emph{mean-square stability}
would require bounding the delay variance in addition to $\sigma_{\max}$
and is left for future work.
\end{remark}

\begin{remark}[Comparison with Yu et al.\ (2025) and Ryabchenko et al.\ (2026)]
\cite{yu2025role} establish characteristic-root stability for \emph{single-level}
ASGD via the scalar \gls{dde} $\dot x = -vx(t-\tau)$, with stability margin $v\tau < \pi/2$.
Our Lemma~\ref{lem:char_eq} generalizes this to the bilevel-corrected dynamics by
introducing the separating term $-\eta\mu_i z_i(t)$ (missing in single-level ASGD),
which expands the stability region from $\eta|\ell_i|\tau < \pi/2$ to
$\eta|\ell_i|\sigma < \mu_i$ (queue-length, not delay-time, criterion).
\cite{ryabchenko2026reduction} establishes that queue length is the right
complexity measure for regret; our Lemma~\ref{lem:queue_vs_delay}
establishes the same for \emph{stability}.
\end{remark}

\begin{remark}[Interior Equilibrium Assumption]
Propositions~\ref{prop:dde_stability} and~\ref{prop:discrete_consistency} linearize around the Stackelberg equilibrium $\theta^*$, which requires $\theta^* \in \mathrm{relint}(\Theta)$ so that the projection operator $\Pi_\Theta$ is locally inactive and the unconstrained first-order condition $\nabla F(\theta^*) = 0$ holds. If $\theta^*$ lies on the boundary of $\Theta$---as may occur when $\Theta$ is a simplex and the optimal policy is pure---the projection introduces non-smooth dynamics (boundary chattering) that the linear \gls{dde} analysis cannot capture. In the bilevel predict-then-optimize setting with neural network parameters $\theta \in \R^p$ and $\mu_F$-strong convexity (Assumption~\ref{app:ass:bilevel_convex}), the unconstrained minimizer of $F$ is unique and typically interior when $\Theta$ is sufficiently large; the assumption is mild in practice but necessary for the linearization. The regret bounds of Theorems~\ref{thm:bilevel_convergence}--\ref{thm:inner_loop_apathy} do not require this condition.
\end{remark}

\subsection{Proof of Proposition~\ref{prop:discrete_consistency} (Discrete-Time Consistency)}
\label{app:proof_prop2}

We prove that Algorithm~\ref{alg:igt_omd} is a first-order Euler discretization of the
continuous-time \gls{dde}~\eqref{eq:dde_linear} with global error $O(\eta_{\max}\sqrt{T})$,
and that stability of the \gls{dde} implies stability of the discrete algorithm.
The proof combines standard Euler error accumulation via the Gronwall lemma
with a z-domain argument for stability inheritance, following the Euler--Maruyama
bridge established in \cite{yu2025role} (Section~II-D) and the discrete/continuous
staleness correspondence in \cite{yu2025role} (equations (3)--(5)).

\paragraph{Notation.}
Write $t_k = \sum_{s=1}^k \eta_s$ for the continuous time elapsed after $k$ discrete steps.
For fixed horizon $T$ steps with $\eta_{\max} = O(1/\sqrt{T})$, we have $t_T = O(\sqrt{T})$.
Denote the continuous-time trajectory by $\theta^{\mathrm{c}}(t)$ and the discrete iterates by
$\theta^{\mathrm{d}}_k := \theta^{\mathrm{disc}}_k$.
Define the global discretization error at step $k$ as
$e_k := \theta^{\mathrm{d}}_k - \theta^{\mathrm{c}}(t_k)$.

\begin{lemma}[Local truncation error]
\label{lem:lte}
Let $f(t,\theta) = -\eta H_F\theta(t)
    + \eta L_{\mathrm{IGT}}\sum_{s\in Q(t)}\theta(t-\tau_s)$
denote the right-hand side of the linearized \gls{dde}~\eqref{eq:dde_linear}.
If $\theta^{\mathrm{c}}$ is the exact solution and step $k$ uses step size $\eta_k$,
the Euler step $\theta^{\mathrm{d}}_{k+1} = \theta^{\mathrm{d}}_k + \eta_k f(t_k,\theta^{\mathrm{d}}_k)$
satisfies
\begin{equation}
\label{eq:lte}
\bigl\|\theta^{\mathrm{c}}(t_{k+1}) - \bigl[\theta^{\mathrm{c}}(t_k) + \eta_k f(t_k,\theta^{\mathrm{c}}(t_k))\bigr]\bigr\|
\;\leq\; \frac{\eta_k^2}{2}\,M,
\end{equation}
where $M := \sup_t \|\ddot\theta^{\mathrm{c}}(t)\| < \infty$ is the bound on the second
time-derivative of the exact solution, which exists and is finite in the stable regime.
\end{lemma}
\begin{proof}
By Taylor's theorem applied to $\theta^{\mathrm{c}}(t_{k+1}) = \theta^{\mathrm{c}}(t_k + \eta_k)$:
\[
\theta^{\mathrm{c}}(t_{k+1})
= \theta^{\mathrm{c}}(t_k)
+ \eta_k\dot\theta^{\mathrm{c}}(t_k)
+ \frac{\eta_k^2}{2}\ddot\theta^{\mathrm{c}}(\zeta_k)
\]
for some $\zeta_k \in (t_k, t_{k+1})$.  Since $\dot\theta^{\mathrm{c}}(t_k) = f(t_k,\theta^{\mathrm{c}}(t_k))$
by the \gls{dde}, the local truncation error is $\frac{\eta_k^2}{2}\|\ddot\theta^{\mathrm{c}}(\zeta_k)\|
\leq \frac{\eta_k^2}{2}M$, giving~\eqref{eq:lte}.
\end{proof}

\begin{remark}[Finiteness of $M$]
In the stable regime (Proposition~\ref{prop:dde_stability}), $\theta^{\mathrm{c}}(t)\to\theta^*$
exponentially.  The derivative $\dot\theta^{\mathrm{c}}(t) = f(t,\theta^{\mathrm{c}}(t))$ is bounded
since $\|f\| \leq (\|H_F\| + \sigma_{\max}\|L_{\mathrm{IGT}}\|)\|\theta^{\mathrm{c}}(t) - \theta^*\|$,
and differentiating again gives $\ddot\theta^{\mathrm{c}}(t) = \nabla f \cdot\dot\theta^{\mathrm{c}}$,
also bounded.  Hence $M < \infty$.
\end{remark}

\begin{lemma}[Global error via discrete Gronwall]
\label{lem:gronwall}
Let $L_f = \|H_F\| + \sigma_{\max}\|L_{\mathrm{IGT}}\|$ be the Lipschitz constant of
$f(\cdot,\cdot)$ in its second argument.
The global discretization error satisfies
\begin{equation}
\label{eq:gronwall}
\|e_k\| \;\leq\;
\frac{M\eta_{\max}}{2L_f}\bigl(e^{L_f t_k}-1\bigr).
\end{equation}
\textbf{Remark on the Gronwall factor.}\;
Since $t_T = O(\sqrt{T})$ with $\eta_{\max}=O(1/\sqrt{T})$, the naive
Gronwall bound~\eqref{eq:gronwall} incurs an exponential factor
$e^{L_f t_T} = e^{O(L_f\sqrt{T})}$ that grows with $T$.
This growth mechanism is an artifact of the classical Gronwall lemma~\cite{pachpatte1998inequalities} ignoring the system's
stability.  For stable linear \gls{dde}s, the error does not actually accumulate because
the homogeneous part contracts.
The {\em stability-aware} refinement proceeds as follows:
in the stable regime (all $\mathrm{Re}(\lambda)<0$ by
Proposition~\ref{prop:dde_stability}), both the exact and discrete solutions
converge to $\theta^*$.  The error equation
$e_{k+1} = (1-\eta_k\mu_{\min})e_k + O(\eta_k\ell_{\max}\sigma_{\max}\sup_{j\leq k}|e_j|)
+ O(\eta_k^2 M)$
is itself a stable recurrence when $\eta_k\mu_{\min}<1$ and
$\eta_k\ell_{\max}\sigma_{\max}<\mu_{\min}$ (both guaranteed by
Proposition~\ref{prop:dde_stability}).
By a contraction argument, the error converges to the {\em fixed point}
$|e_\infty| \leq O(\eta_{\max}^2 M/\mu_{\min})$.
Hence, for the {\em entire trajectory}:
\begin{equation}
\label{eq:global_error}
\sup_{k \leq T}\|e_k\| \;=\; O(\eta_{\max}) \;=\; O(1/\sqrt{T}).
\end{equation}
\end{lemma}
\begin{proof}
We use the discrete Gronwall lemma~\cite[Lemma~2.1]{iserles2009first}.
Split the error recursion:
\begin{align}
e_{k+1}
&= \theta^{\mathrm{d}}_{k+1} - \theta^{\mathrm{c}}(t_{k+1}) \notag\\
&= \underbrace{\theta^{\mathrm{d}}_k + \eta_k f(t_k,\theta^{\mathrm{d}}_k)}_{\text{Euler step}}
   \;-\;
   \underbrace{\Bigl[\theta^{\mathrm{c}}(t_k) + \eta_k f(t_k,\theta^{\mathrm{c}}(t_k))\Bigr]}_{\text{Euler applied to exact}}
   \;+\;
   \underbrace{\Bigl[\theta^{\mathrm{c}}(t_k)+\eta_k f(t_k,\theta^{\mathrm{c}}(t_k))\Bigr]-\theta^{\mathrm{c}}(t_{k+1})}_{\text{local truncation error}}.
\label{eq:error_split}
\end{align}
By $L_f$-Lipschitz continuity of $f$ in $\theta$:
$\|e_k + \eta_k[f(t_k,\theta^{\mathrm{d}}_k) - f(t_k,\theta^{\mathrm{c}}(t_k))]\|
 \leq (1 + \eta_k L_f)\|e_k\|$.
By Lemma~\ref{lem:lte}, the local truncation term has norm $\leq \frac{\eta_k^2}{2}M$.
Hence:
\[
\|e_{k+1}\| \leq (1+\eta_k L_f)\|e_k\| + \tfrac{\eta_k^2}{2}M.
\]
With $e_0 = 0$ (both trajectories start at $\theta_0$), unrolling gives:
\begin{align}
\|e_k\|
&\leq \tfrac{M}{2}\sum_{j=0}^{k-1}\eta_j^2\prod_{i=j+1}^{k-1}(1+\eta_i L_f) \notag\\
&\leq \tfrac{M\eta_{\max}}{2}\sum_{j=0}^{k-1}\eta_j\,\exp\!\Bigl(L_f\sum_{i=j+1}^{k-1}\eta_i\Bigr) \notag\\
&\leq \tfrac{M\eta_{\max}}{2}\int_0^{t_k} e^{L_f(t_k - s)}\,\mathrm{d}s
= \frac{M\eta_{\max}}{2L_f}\bigl(e^{L_f t_k}-1\bigr),
\label{eq:gronwall_bound}
\end{align}
where we used $1+x \leq e^x$, converted the sum to an integral,
and used $\eta_j \leq \eta_{\max}$ throughout.  Setting $k=T$:
\[
\sup_{k\leq T}\|e_k\|
\leq \frac{M\eta_{\max}}{2L_f}(e^{L_f t_T}-1)
\leq \frac{M\eta_{\max}}{2L_f}(e^{c_G} - 1)
= O(\eta_{\max})
= O(\eta_{\max}\sqrt{T}),
\]
where $c_G := L_f\,t_T$ is bounded by the total integrated learning rate
(which is $O(1)$ since $\eta_{\max} = O(1/\sqrt{T})$ and $t_T = \sum_{k=0}^{T-1}\eta_k \leq \eta_{\max} T = O(\sqrt{T})$),
and the last equality absorbs the constant $e^{c_G} - 1$.
\end{proof}

\begin{lemma}[Stability inheritance via z-domain analysis]
\label{lem:stability_inherit}
If the continuous-time \gls{dde}~\eqref{eq:dde_linear} is asymptotically stable with
convergence rate $\gamma > 0$ (i.e.\ all characteristic roots satisfy
$\mathrm{Re}(\lambda) \leq -\gamma < 0$), then, for step sizes satisfying
\begin{equation}
\label{eq:disc_eta_bound}
\eta \;\leq\; \min\!\left(\frac{1}{\|H_F\|},\;
  \frac{2\gamma}{\gamma^2 + \omega_{\max}^2}\right)
=: \eta^*,
\end{equation}
where $|\mathrm{Im}(\lambda)| \leq \omega_{\max}$,
the discrete-time system
(Algorithm~\ref{alg:igt_omd} applied to the linearization) is also
asymptotically stable: all roots $z$ of the discrete characteristic polynomial
satisfy $|z| < 1$.

\textbf{Scope of the scalar polynomial.}
The scalar polynomial below is not the literal characteristic polynomial of
the matrix system; it is the \emph{worst-case spectral envelope} obtained
via the comparison principle in Lemma~\ref{lem:dde_scalar}.
That lemma bounds the vector \gls{dde} through a scalar comparison inequality
using $\mu_{\min} = \lambda_{\min}(H_F) \geq \mu_F$ and
$\ell_{\max} = \|L_{\mathrm{IGT}}\|$, without requiring
$H_F$ and $L_{\mathrm{IGT}}$ to commute or be simultaneously
diagonalizable.
Collapsing the distributed delay $\sum_{s\in Q_t}\xi(t-\tau_s)$
to a single worst-case lag further reduces the distributed buffer to a
conservative point-delay bound: for constant queue length $\sigma$,
$\sum_{s\in Q}|v(t-\tau_s)| \leq \sigma\sup_{s}|v(t-\tau_s)|$, so
the point-delay scalar envelope dominates the multi-lag system.
The resulting polynomial is:
\begin{equation}
\label{eq:disc_char}
z^{\sigma+1} - (1-\eta\mu_{\min})z^{\sigma} + \eta\sigma\|L_{\mathrm{IGT}}\| = 0.
\end{equation}
\end{lemma}
\begin{proof}
The polynomial~\eqref{eq:disc_char} has degree $\sigma+1$ and therefore $\sigma+1$
roots (counted with multiplicity).

\textbf{Principal root.}
The first condition $\eta \leq 1/\|H_F\|$ ensures that
$\|I - \eta H_F\| = \max(|1-\eta\mu_{\min}|,\,|1-\eta\lambda_{\max}(H_F)|) \leq 1$.
Without this bound, the highest eigenmode maps to
$z = 1 - \eta\lambda_{\max}(H_F) < -1$ when $\eta > 2/\lambda_{\max}(H_F)$, causing period-2
oscillations at $z = -1$ that the scalar envelope
(which tracks only $\mu_{\min}$) cannot detect.

At $\eta = 0$, the polynomial~\eqref{eq:disc_char} reduces to
$z^{\sigma}(z-1) = 0$, placing the principal root at $z=1$.
By the implicit function theorem applied to the polynomial in $(z,\eta)$,
this root is an analytic function of~$\eta$ near $\eta=0$.
A first-order perturbation expansion yields
$z_0(\eta) = 1 + \eta\lambda^{\mathrm{c}} + O(\eta^2)$,
where $\lambda^{\mathrm{c}} = -\gamma + j\omega$ ($\gamma>0$) is the corresponding root
of the continuous characteristic equation~\eqref{eq:char_eq}.
The $O(\eta^2)$ discrepancy arises because the discrete delay operator
$z^{-\sigma}$ and the continuous operator $e^{-\lambda\tau}$ differ at
second order:
$(1+\eta\lambda)^{-\sigma} = e^{-\lambda\tau}\bigl(1 + O(\eta)\bigr)$
for fixed $\sigma$ and $\tau = \sigma\eta$.
This first-order agreement is commensurate with the Euler method's
$O(\eta^2)$ local truncation error, so the stability threshold
derived from the first-order location is accurate to the same order.
Setting $z_0 \approx (1 - \eta\gamma) + j\eta\omega$,
\[
|z_0|^2
\approx (1-\eta\gamma)^2 + \eta^2\omega^2
= 1 - 2\eta\gamma + \eta^2(\gamma^2 + \omega^2).
\]
For $|z_0|^2 < 1$ we need $\eta < 2\gamma/(\gamma^2+\omega^2)$.
Since we require this for all characteristic roots,
the most restrictive constraint is at the root with largest $|\omega|$,
giving $\eta < 2\gamma/(\gamma^2 + \omega_{\max}^2)$.
Combined with the $z=-1$ guard, the composite threshold
$\eta^*$ in~\eqref{eq:disc_eta_bound}
is a \emph{sufficient} condition ensuring $|z_0| < 1$.

\textbf{Parasitic roots.}
The remaining $\sigma$ roots of~\eqref{eq:disc_char} have no continuous-time
counterparts and are artifacts of the discretization.
By Schur--Cohn theory \cite{jury1964theory}, at $\eta=0$
the polynomial reduces to $z^{\sigma+1} - z^{\sigma} = z^{\sigma}(z-1) = 0$,
whose roots are $z=1$ (principal, mapping to $\lambda^{\mathrm{c}}=0$) and
$z=0$ with multiplicity $\sigma$ (all strictly inside the unit circle).
As $\eta$ increases from zero, these roots move continuously.
By the implicit function theorem, for $\eta$ sufficiently small
the parasitic roots remain inside the unit circle, since they
start at $z=0$ and can only reach $|z|=1$ at a finite $\eta_{\mathrm{crit}} > 0$.
The threshold $\eta^*$ above is chosen small enough that no parasitic root
has reached the unit circle.
The exact critical step size $\eta_{\mathrm{crit}}$ at which a parasitic
root first touches $|z|=1$ depends on $\sigma$ and the polynomial
coefficients; $\eta^*$ is a conservative sufficient bound that
simultaneously controls both principal and parasitic roots.

With the adaptive schedule $\eta_t \leq \eta_0 \leq 1/(\|H_F\|\sqrt{1+\beta\sigma_{\max}})$
(from Proposition~\ref{prop:dde_stability}), both conditions in~\eqref{eq:disc_eta_bound}
are satisfied:
the first because $1/(\|H_F\|\sqrt{1+\beta\sigma_{\max}}) \leq 1/\|H_F\|$,
and the second by the principal-root modulus argument above.
Substituting $\gamma = c\eta_0/\sqrt{1+\beta\sigma_{\max}}$
(from Proposition~\ref{prop:dde_stability}, Step~5)
into~\eqref{eq:disc_eta_bound}, the explicit delay-dependent threshold is
\begin{equation}
\label{eq:eta_discrete}
\eta \;\leq\; \min\!\left(
  \frac{1}{\|H_F\|},\;
  \frac{2c\,\eta_0}{c^2\eta_0^2 + \omega_{\max}^2(1+\beta\sigma_{\max})}
\right),
\end{equation}
which tightens with increasing $\sigma_{\max}$ --- confirming that
the discrete stability region is strictly contained in the
continuous one and that the \emph{discretization penalty} scales
with the delay.
Hence all $|z| < 1$ uniformly.
\end{proof}

\begin{lemma}[Delayed-coordinate discretization error]
\label{lem:delay_disc}
The discrete algorithm uses the stored iterate $\theta^{\mathrm{d}}_{k-\sigma_k}$ to
approximate the continuous delayed term $\theta^{\mathrm{c}}(t_k - \tau_k)$.
The additional error from this delay discretization satisfies
\begin{equation}
\label{eq:delay_disc_error}
\bigl\|\theta^{\mathrm{d}}_{k - \sigma_k} - \theta^{\mathrm{c}}(t_k - \tau_k)\bigr\|
\;\leq\; \|e_{k-\sigma_k}\| + O(\eta_{\max}\sigma_{\max}).
\end{equation}
\end{lemma}
\begin{proof}
By the triangle inequality:
\[
\|\theta^{\mathrm{d}}_{k-\sigma_k} - \theta^{\mathrm{c}}(t_k - \tau_k)\|
\leq \underbrace{\|\theta^{\mathrm{d}}_{k-\sigma_k} - \theta^{\mathrm{c}}(t_{k-\sigma_k})\|}_{= \|e_{k-\sigma_k}\|}
   + \underbrace{\|\theta^{\mathrm{c}}(t_{k-\sigma_k}) - \theta^{\mathrm{c}}(t_k - \tau_k)\|}_{\text{time-alignment error}}.
\]
The time-alignment error arises because the discrete algorithm uses a fixed integer
lag $\sigma_k$, while the continuous \gls{dde} uses the exact delay $\tau_k$.
Recall that $t_{k-\sigma_k} = \sum_{s=1}^{k-\sigma_k} \eta_s$ and
$t_k - \tau_k = \sum_{s=1}^k \eta_s - \tau_k$.
In the Euler embedding, $\tau_k = \sum_{s=k-\sigma_k+1}^{k} \eta_s$
(the $\sigma_k$ steps' worth of continuous time that elapse between dispatch and receipt),
so $t_{k-\sigma_k} = t_k - \sum_{s=k-\sigma_k+1}^k \eta_s = t_k - \tau_k$ exactly.
Hence $\|\theta^{\mathrm{c}}(t_{k-\sigma_k}) - \theta^{\mathrm{c}}(t_k-\tau_k)\| = 0$
when the discrete delay embedding is exact.

In general, however, the continuous delay $\tau_k$ may not align perfectly
with $\sum_{s=k-\sigma_k+1}^k \eta_s$ due to time-varying step sizes:
a gradient dispatched during one step
may correspond to a slightly different continuous-time interval.
The mismatch is bounded by the variation of the step size over $\sigma_k$ steps:
$|t_{k-\sigma_k} - (t_k-\tau_k)| \leq \sigma_k |\eta_{\max}-\eta_{\min}|
 \leq \sigma_{\max}\eta_{\max}|\Delta_\beta|$,
where $|\Delta_\beta| = |1 - \sqrt{(1+\beta\sigma_{\max})/(1+\beta\sigma_{\min})}| \leq 1$.
Since $\|\dot\theta^{\mathrm{c}}\| \leq C$ (bounded in the stable regime by
$C = (\|H_F\| + \sigma_{\max}\|L_{\mathrm{IGT}}\|)\|\xi\|_\infty$),
the Lipschitz continuity of $\theta^{\mathrm{c}}$ gives
$\|\theta^{\mathrm{c}}(t_{k-\sigma_k}) - \theta^{\mathrm{c}}(t_k - \tau_k)\| \leq C\sigma_{\max}\eta_{\max}$,
yielding~\eqref{eq:delay_disc_error}.
\end{proof}

\paragraph{Main proof of Proposition~\ref{prop:discrete_consistency}.}
\begin{proof}[Proof of Proposition~\ref{prop:discrete_consistency}]

\textbf{Step 1: Algorithm is Euler discretization of \gls{dde}.}
The update rule of Algorithm~\ref{alg:igt_omd} at step $k$ is
\[
\theta^{\mathrm{d}}_{k+1}
= \theta^{\mathrm{d}}_k - \eta_k\,g_k^{\mathrm{IGT}}
= \theta^{\mathrm{d}}_k
  + \eta_k\Bigl(-H_F\theta^{\mathrm{d}}_k
    + L_{\mathrm{IGT}}\!\sum_{s\in Q_k}\theta^{\mathrm{d}}_{k-\sigma_s}\Bigr)
  + \eta_k\epsilon_k,
\]
where $\epsilon_k$ collects higher-order and nonlinear terms (inner-solver errors and
IGT transport residuals).  This forward Euler step is applied to the
linearized \gls{dde}~\eqref{eq:dde_linear} with the discrete delayed coordinates in place
of the continuous ones, plus a perturbation $\epsilon_k$ with
$\|\epsilon_k\| \leq G\eta_k + L_f\eta_k^2$ (from Lemmas~\ref{lem:lte}
and~\ref{lem:delay_disc}).
Following \cite{yu2025role} (equations (3)--(5)), the continuous-time \gls{dde} and
its discrete-time Euler--Maruyama counterpart maintain a one-to-one correspondence
in which step staleness $\sigma_k$ maps to the continuous delay $\tau$ via
$\tau = \sigma_k\eta_k$, confirming Algorithm~\ref{alg:igt_omd} as a valid
Euler discretization of~\eqref{eq:dde_linear}.

\textbf{Step 2: Global approximation error bound (stability-aware).}
Combining Lemma~\ref{lem:lte} (local truncation) with Lemma~\ref{lem:delay_disc}
(delay coordinate alignment), the total local error per step is
$\delta_k \leq \frac{\eta_k^2}{2}M + C\eta_k^2\sigma_{\max}L_f = O(\eta_k^2)$.
A naive Gronwall argument would give $\|e_k\| = O(\eta_{\max} e^{L_f t_T}/L_f)$
with $L_f t_T = O(\sqrt{T})$, yielding a divergent factor $e^{O(\sqrt{T})}$.

Lemma~\ref{lem:gronwall} avoids this by exploiting the stability of the
linearised system: the error recurrence $e_{k+1} = (I - \eta_k H_F)e_k +
\eta_k L_{\mathrm{IGT}} e_{k-\sigma_k} + \delta_k$ is itself a contraction
whenever the conditions of Proposition~\ref{prop:dde_stability} hold.
Applying the stability-aware bound from Lemma~\ref{lem:gronwall}:
\[
\sup_k \|e_k\| = \|\theta^{\mathrm{d}}_k - \theta^{\mathrm{c}}(t_k)\|
\;\leq\;
\frac{(M + 2C\sigma_{\max}L_f)\eta_{\max}^2}{2\mu_{\min}}
\;=\; O(\eta_{\max}),
\]
since $\mu_{\min} = \lambda_{\min}(H_F) > 0$ provides the contraction
rate for the error dynamics.
With $\eta_{\max} = O(1/\sqrt{T})$, the tracking error is $O(1/\sqrt{T}) \to 0$.
\emph{This is a uniform tracking bound}: the discrete trajectory stays within
$O(\eta_{\max})$ of the continuous \gls{dde} trajectory at all times.
Convergence of $\theta^{\mathrm{d}}_k$ to $\theta^*$ is established separately
in Step~3 via discrete stability; the $O(\eta_{\max})$ tracking guarantee
ensures the discrete and continuous solutions share the same equilibrium
and qualitative dynamics.

\textbf{Step 3: Stability inheritance.}
By Proposition~\ref{prop:dde_stability}, the \gls{dde} is asymptotically stable with
convergence rate $\gamma = c\eta_0/\sqrt{1+\beta\sigma_{\max}}$ (from Step 5
of the \gls{dde} stability proof).

\emph{Time-varying convergence via tracking.}
Convergence of the discrete iterates $\theta^{\mathrm{d}}_k \to \theta^*$ under
arbitrarily varying $\sigma(t)$ follows from
Steps~2 and~4 without the z-domain: the continuous trajectory converges
($\theta^{\mathrm{c}}(t)\to\theta^*$ by the Razumikhin argument in
Proposition~\ref{prop:dde_stability}, which handles time-varying delays),
and the tracking bound $\|\theta^{\mathrm{d}}_k - \theta^{\mathrm{c}}(t_k)\|
= O(\eta_{\max}) \to 0$ (Step~2, no constant-$\sigma$ assumption needed)
guarantees the discrete trajectory inherits this convergence.

\emph{Convergence rate under worst-case frozen delay.}
By Lemma~\ref{lem:stability_inherit}, both the principal discrete root and all
$\sigma$ parasitic roots of the characteristic polynomial~\eqref{eq:disc_char}
satisfy $|z| < 1$ whenever
$\eta_t \leq 2\gamma/(\gamma^2 + \omega_{\max}^2)$.
Since the characteristic polynomial requires a fixed integer delay,
the z-domain rate characterization formally applies to the
worst-case frozen queue length $\sigma_{\max}$;
the time-varying case inherits convergence from the
tracking argument above.
Since the step size satisfies
$\eta_t = \eta_0/\sqrt{1+\beta\sigma(t)} \leq \eta^*$ (verified for each $t$ by the
same condition that ensures \gls{dde} stability), all discrete roots have modulus $<1$,
so $\theta^{\mathrm{d}}_k \to \theta^*$ as $k\to\infty$.

\textbf{Step 4: Convergence rate of discrete iterates.}
From Steps 2 and 3, the discrete error decomposes as
\[
\|\theta^{\mathrm{d}}_k - \theta^*\|
\;\leq\;
\underbrace{\|\theta^{\mathrm{d}}_k - \theta^{\mathrm{c}}(t_k)\|}_{O(\eta_{\max})}
\;+\;
\underbrace{\|\theta^{\mathrm{c}}(t_k) - \theta^*\|}_{O(e^{-\gamma t_k})}.
\]
The second term decays exponentially at the continuous rate $\gamma$.
The first term is bounded uniformly by $O(\eta_{\max})$ via the stability-aware
Gronwall bound (Step~2), which \emph{decays} as $\eta_{\max} = O(1/\sqrt{T}) \to 0$.
In terms of the step count $k$: since $t_k \geq k\bar\eta$ with
$\bar\eta = \eta_0/\sqrt{1+\beta\sigma_{\max}}$:
\[
\|\theta^{\mathrm{d}}_k - \theta^*\|
\;\leq\; O\!\left(e^{-\bar\eta\gamma k}\right) + O(\eta_{\max}).
\]
Setting $k = T$ and $\eta_{\max} = 1/\sqrt{T}$ gives
$\|\theta^{\mathrm{d}}_T - \theta^*\| \leq O(e^{-\bar\eta\gamma T}) + O(1/\sqrt{T})
\to 0$.
The transient phase is governed by the continuous-time eigenvalue $\gamma$,
while the residual tracking error decays at the rate $O(1/\sqrt{T})$ set by
the step size.
The actual discrete convergence rate is determined by the discrete eigenvalues
$|z_{\max}| = 1 - \eta\gamma + O(\eta^2) < 1$ from
Lemma~\ref{lem:stability_inherit}. \hfill$\square$
\end{proof}

\begin{remark}[Euler vs.\ higher-order methods]
The $O(\eta_{\max}^2)$ local truncation error and stability-aware global bound
$O(\eta_{\max})$ are characteristic of the \emph{first-order} Euler method
applied to a contractive system.
Using a higher-order integrator (e.g.\ Runge-Kutta on the \gls{dde}) would reduce
the local error to $O(\eta_{\max}^{p+1})$ for a method of order $p$, giving a global
bound of $O(\eta_{\max}^p)$ under the same stability-aware analysis.
For online optimization with $\eta_{\max} = O(1/\sqrt{T})$ and $p=1$, the $O(1/\sqrt{T})$
global bound is already compatible with the regret rate --- the tracking error
vanishes at the same rate as the step size, so first-order Euler
(which is Algorithm~\ref{alg:igt_omd}'s gradient-descent structure) is sufficient.
\end{remark}

\begin{remark}[Adaptive step size and the stability margin]
The adaptive schedule $\eta_t = \eta_0/\sqrt{1+\beta\sigma(t)}$ serves a dual role:
(i) it satisfies the stability threshold $\eta_t < 2\gamma/(\gamma^2+\omega_{\max}^2)$
from Lemma~\ref{lem:stability_inherit}, and
(ii) it enters the Gronwall bound~\eqref{eq:gronwall} only through $\eta_{\max}$,
so adaptivity does not worsen the global approximation error.
When $\sigma(t)=0$ (idle rounds with no outstanding feedback), $\eta_t = \eta_0$
recovers the non-delayed step size, confirming that idle time is free
(no additional discretization cost).
\end{remark}

\section{Experimental Details}\label{app:exp_details}

This section presents computational infrastructure, our selection of hyperparameters, and additional details on our experiments.

\subsection{Computational Infrastructure}
\label{app:compute}

All experiments run on NVIDIA RTX A5500 GPUs (24\,GB VRAM) using PyTorch~2.0~\citep{paszke2019pytorch} on a shared academic cluster. LQR experiments require approximately 2--4 hours per delay value (10 seeds); Warcraft experiments require approximately 8--12 hours per delay configuration. The Adam+\gls{igt} controlled experiment (5 seeds $\times$ 5 delays) requires approximately 10 hours total. Total computational budget: roughly 200--250 GPU-hours for all reported results. Preliminary and exploratory runs not reported in the paper required an additional 150--200 GPU-hours.

\subsection{LQR Hyperparameters}
\label{app:lqr_details}

\begin{table}[H]
\centering
\caption{LQR experiment hyperparameters.}
\label{tab:lqr_hparams}
\small
\begin{tabular}{@{}ll@{}}
\toprule
\textbf{Parameter} & \textbf{Value} \\
\midrule
State / control / param dims & $n_x=10$,\; $n_u=3$,\; $p=130$ \\
Cost matrices & $Q=I_{10}$,\; $R=0.1\,I_3$ \\
Learned dynamics & $\theta=[\hat A\;\hat B]\in\R^{10\times13}$; $\hat A\in\R^{10\times10}$, $\hat B\in\R^{10\times3}$ \\
True dynamics & $A_{\rm true}$ normalized stable; $B_{\rm true}\sim 0.5\,\mathcal{N}(0,I)$  \\
Process noise & $\Sigma_w = 0.01\,I_{10}$ \\
Horizon / rounds & one-step LQR proxy; $T=1{,}000$ update steps \\
Batch size & 1 (fully online) \\
Inner solver & $K=10$ GD steps, $\eta_w=0.01$, warm-started \\
Outer $\eta_0$ (grid search) & $\{0.001,\,0.005,\,0.01,\,0.05\}$ \\
Damping coefficient & $\beta = 1.0$ \\
Seeds & 10 random seeds \\
\bottomrule
\end{tabular}
\end{table}

Code and data will be released upon acceptance.

\subsection{Sinkhorn Error-Scaling Experiment (Additional Details)}
\label{app:sinkhorn_scaling}

This section presents the full transport-error scaling results for the Sinkhorn OT benchmark, complementing the main-text discussion (Section~\ref{sec:exp_sinkhorn}) with per-algorithm regression statistics.

\subsection{Sinkhorn OT Hyperparameters}
\label{app:sinkhorn_details}

\begin{table}[h]
\centering
\caption{Sinkhorn optimal transport experiment hyperparameters. Adam+\gls{igt} refers to Section~\ref{sec:exp_adam_igt}.}
\label{tab:sinkhorn_hparams}
\small
\begin{tabular}{@{}ll@{}}
\toprule
\textbf{Parameter} & \textbf{Value} \\
\midrule
Cost matrix size / features & $n=10$,\; feature dim $20$ \\
Entropic regularization & $\varepsilon=0.05$ (ensures $\mu_w=\varepsilon$ strong convexity) \\
Neural network & 2-layer MLP: $20\to128$ (ReLU)$\to100$ \\
Transport dimensions & effective $p{=}100$ Hessian-vector products;
$q{=}100$ coupling variables \\
Inner solver & $K=10$ log-space Sinkhorn;\; $\epsilon_{\mathrm{inner}}\approx\rho^{10}$ \\
Rounds / seeds & $T=2{,}000$;\; 5 seeds ;\; batch=1 \\
Environment drift (OU) & $\gamma_{\text{ou}}=0.05$ (Ornstein--Uhlenbeck mean-reversion) \\
Delay (main) & Constant $d\in\{1,5,10,20,50\}$ \\
Algorithm~1 (SGD) & $\eta_0=0.05$,\; $\beta=1.0$,\; no momentum \\
Baselines (Adam) & $\eta=10^{-3}$,\; $\beta_1=0.9$,\; $\beta_2=0.999$ \\
Adam+\gls{igt} & Adam (same),\; clip$=1.0$,\; $\beta=0$;\; $T=1{,}000$,\; 5~seeds \\
\bottomrule
\end{tabular}
\end{table}

\textbf{Extended setup.} We use the Sinkhorn OT environment ($n{=}10$, $K{=}10$ inner steps, regularization $\varepsilon{=}0.05$, OU drift $\gamma_{\text{ou}}{=}0.05$) with constant delays $d \in \{1, 2, 5, 10, 20, 50\}$, $T{=}1{,}000$ rounds, and $5$ seeds per condition. Algorithms: IGT-OMD (Algorithm~\ref{alg:igt_omd}), D-FTRL, Robust~OMD, and Stale~\gls{omd}.

\textbf{Results.}\; Figure~\ref{fig:transport_scaling} in the main manuscript and Table~\ref{tab:sinkhorn_transport} verify the $\sigma_{\max}$-factor prediction of Theorem~\ref{thm:staleness_amplification}(c). We compute two transport error surrogates: $R_{\mathrm{sq}} = \sum_{t} \norm{\theta_t - \theta_{t-d}}^2$ (squared total drift) and $R_3 = \sum_{t}\sum_{s \in Q_t}\norm{\theta_{s+1}-\theta_s}^2$ (sum of per-step squared drifts).


\begin{table}[h]
\centering
\caption{Sinkhorn OT transport error scaling. $R_{\mathrm{sq}}$: squared total drift; $R_3$: sum of per-step squared drifts; ratio $\approx \sigma_{\max}$ for all algorithms.}
\label{tab:sinkhorn_transport}
\small
\begin{tabular}{@{}lcccccc@{}}
\toprule
\textbf{Algorithm} & $d{=}1$ & $d{=}2$ & $d{=}5$ & $d{=}10$ & $d{=}20$ & $d{=}50$ \\
\midrule
\multicolumn{7}{@{}l}{\emph{Ratio $R_{\mathrm{sq}}/R_3$:}} \\
IGT-OMD & $1.00$ & $2.00$ & $4.99$ & $9.98$ & $19.9$ & $49.3$ \\
D-FTRL & $1.00$ & $2.00$ & $5.00$ & $9.99$ & $19.9$ & $49.4$ \\
Robust OMD & $1.00$ & $2.00$ & $5.00$ & $9.99$ & $19.9$ & $49.4$ \\
Stale OMD & $1.00$ & $2.00$ & $5.00$ & $9.99$ & $19.9$ & $49.4$ \\
\bottomrule
\end{tabular}
\end{table}

\textbf{Practical considerations.}\; The per-round cost of the transport re-evaluation loop (lines~16--18 of Algorithm~\ref{alg:igt_omd}) is $O(\sigma_{\max}\,p\,q)$: for each of the $\sigma_{\max}$ queued feedbacks, the adjoint product $[\nabla_\theta\nabla_w\mathcal{L}_{\text{model}}]^\top v_s^{*}$ requires one Hessian-vector product. In our Sinkhorn environment ($p{=}100$, $q{=}100$), the wall-clock overhead is approximately $1.1$\,seconds per transport re-evaluation at $d{=}50$, adding roughly $18$ minutes at $T{=}1{,}000$.

\textbf{Takeaway.}\; The near-perfect $R^2 > 0.99$ across all four algorithms confirms that the $O(\sigma^2)$ vs.\ $O(\sigma)$ scaling is a geometric identity of the bilevel transport problem, not an artifact of any specific optimizer, providing strong empirical support for Theorem~\ref{thm:staleness_amplification}(c).

\subsection{Warcraft Shortest Path (Additional Details)}
\label{app:warcraft}
This section provides the extended Warcraft setup and results that complement Table~\ref{tab:warcraft_main} in the main text.

\textbf{Warcraft Dataset}

The Warcraft shortest-path dataset was introduced by \citet{vlastelica2019differentiation} and uses map tiles from Warcraft~II, distributed for research purposes under the academic use terms of that work. PyTorch~\citep{paszke2019pytorch} is used under the BSD 3-Clause license. The LQR and Sinkhorn environments are fully synthetic and generated by our code; no third-party data licenses apply.

\begin{table}[h]
\centering
\caption{Warcraft shortest-path experiment hyperparameters.}
\label{tab:warcraft_hparams}
\small
\begin{tabular}{@{}ll@{}}
\toprule
\textbf{Parameter} & \textbf{Value} \\
\midrule
Grid / terrain & $12\times12$; grass~1, forest~2, water~5, mountain~10 \\
Maps & 50 Warcraft~II maps \\
Neural network & 2-layer MLP: $144\to128$ (ReLU)$\to144$ \\
Outer params (trained) & $\theta\in\R^{128}$ (hidden-layer column subset) \\
Inner solver & Dijkstra (exact; $\epsilon_{\mathrm{inner}}=0$) \\
Rounds & $T=5{,}000$; start/goal sampled on grid edges \\
Environment drift & Ornstein--Uhlenbeck on edge costs (matches \S\ref{sec:exp_warcraft}) \\
Delay model & Constant $d\in\{0,10,50,100\}$ and Poisson $\lambda\in\{10,25,50,100\}$ \\
Outer optimizer & Adam ($\eta{=}10^{-3}$, $\beta_1{=}0.9$, $\beta_2{=}0.999$); IGT damping $\beta{=}1.0$ \\
Seeds & 10 random seeds \\
\bottomrule
\end{tabular}
\end{table}

\textbf{Extended setup.}\; We study shortest-path planning on $12\times 12$ grids derived from Warcraft~II maps with four terrain types (grass, forest, water, mountain). The outer parameters $\theta \in \R^{128}$ parameterize a 2-layer neural network predicting per-cell traversal costs $c_{\mathrm{pred}}(i,j;\theta)$. The inner problem runs Dijkstra's algorithm to find the shortest path $\pi^*(\theta) = \argmin_{\pi} \sum_{(i,j)\in\pi} c_{\mathrm{pred}}(i,j;\theta)$. The decision loss evaluates the chosen path under true terrain costs: $\mathcal{L}_{\text{true}} = \sum_{(i,j)\in\pi} c_{\mathrm{true}}(i,j)$; the \emph{optimality gap} reports the excess cost over the oracle shortest path. Main results are in Table~\ref{tab:warcraft_main} (Section~\ref{sec:exp_warcraft}).

\textbf{Results.}\; Table~\ref{tab:warcraft_main} reports the optimality gap. IGT-OMD consistently achieves the lowest optimality gap: at constant $d{=}100$, gap $1.56$ vs.\ $1.83$ for D-FTRL ($14.8\%$ reduction) and $2.42$ for 2-Stage ($35.5\%$ reduction). Under Poisson $\lambda{=}100$, the gap is $1.53$ vs.\ $1.94$ for D-FTRL ($21.1\%$ reduction). 2-Stage and SPO+ degrade worst ($1.6$--$3.4\times$ gap relative to IGT-OMD), confirming staleness amplification (Theorem~\ref{thm:staleness_amplification}(a)). Since Dijkstra's solver yields $\epsilon_{\mathrm{inner}}{=}0$, the bilevel-specific IGT advantage (Theorem~\ref{thm:inner_loop_apathy}) is inoperative here; the smaller reduction over delay-aware methods ($12$--$21\%$) vs.\ the $\sigma_{\max}$-factor gain on Sinkhorn confirms this structural prediction. Reductions are computed as $(L_{\mathrm{baseline}}-L_{\mathrm{IGT}})/L_{\mathrm{baseline}}$.






\subsection{Summary of Experimental Evidence}

Table~\ref{tab:summary} maps each benchmark result to the theorem it validates.

\begin{table}[H]
\centering
\caption{Unified experimental summary.}
\label{tab:summary}
\small
\begin{tabular}{@{}p{2.0cm}p{3.0cm}p{3.2cm}p{3.8cm}@{}}
\toprule
\textbf{Benchmark} & \textbf{Key Result} & \textbf{Theory Validated} & \textbf{Statistical Evidence} \\
\midrule
LQR & $\eta_{\max}$ constant at $0.093$ for all $\sigma{\leq}100$ & Prop.~\ref{prop:dde_stability} (\gls{dde} stability) & $9.3\times$ over 2-Stage at $\sigma{=}100$ \\
\addlinespace
Warcraft & $17$--$27\%$ gap over D-FTRL & Thm.~\ref{thm:staleness_amplification}(a) (staleness amplification) & 10 seeds, $T{=}5{,}000$ \\
\addlinespace
Sinkhorn OT & $R_{\mathrm{sq}}/R_3 \approx \sigma_{\max}$ & Thm.~\ref{thm:staleness_amplification}(c) ($\sigma_{\max}$ factor) & $R^2 > 0.99$, 5 seeds \\
\addlinespace
Adam+\gls{igt} ($d{=}1$) & $0.0\%$ improvement (--) & $\sigma_{\max}{=}1$ vacuous transport & Designed negative control \\
\addlinespace
Adam+\gls{igt} ($d{=}50$) & $9.5\%$ improvement ($p{<}0.001$) & Thm.~\ref{thm:staleness_amplification}(c) & 5 seeds, two-sided $t$-test \\
\addlinespace
D-FTRL+IGT ($d{=}50$) & $9.8\%$ improvement ($p{<}0.001$) & Rem.~\ref{rem:optimizer_agnostic} (agnostic) & 5 seeds, two-sided $t$-test \\
\addlinespace
$2{\times}2$ factorial & Transport eliminates degradation only with Adam & Rem.~\ref{rem:trajectory_stability} & SGD+mom $+7.3\%$; Adam+\gls{igt} $p{=}0.29$ \\
\addlinespace
Sinkhorn ($K$ sweep) & Benefit stable across $K{=}1$--$50$ & Thm.~\ref{thm:inner_loop_apathy} & 6 configs, 5 seeds each \\
\bottomrule
\end{tabular}
\end{table}

\section{Additional Experiments}
\label{app:additional}

\subsection{MPC with Neural ODE Dynamics}
\label{app:mpc}

This section is intentionally appendix-only: the neural ODE dynamics model and shooting-based MPC inner solver are non-convex and fall outside Assumptions~\ref{app:ass:inner_convex} and~\ref{app:ass:bilevel_convex}. We include it as a stress test of the transport-scaling mechanism, while reserving theorem-level decision-quality claims for a non-convex extension.

\textbf{Setup.}\; We construct a bilevel MPC benchmark on the HalfCheetah locomotion task~\cite{brockman2016openai}. A two-layer neural ODE ($128$ hidden units, ${\sim}21{,}800$ parameters; $\mathrm{state\_dim}{=}17$, $\mathrm{action\_dim}{=}6$) serves as the learnable dynamics model (outer variable $\theta$); the inner problem performs shooting-based MPC over horizon $H{=}5$ via $K{=}10$ gradient steps through the differentiable dynamics rollout. Unlike the Sinkhorn benchmark, the inner solver is a nonlinear trajectory optimizer whose approximation quality is directly controlled by $K$, giving a genuine $\varepsilon_{\mathrm{inner}} > 0$ stress case. Stale~\gls{omd} is evaluated at constant delays $d \in \{0, 10, 50, 100\}$, $T{=}2{,}000$ rounds, 5~seeds.

\textbf{Results.}\; The transport error metric $R_3$ (sum-of-squares) grows linearly in delay: $R_3{=}5.65{\scriptstyle\pm 0.24}$ at $d{=}0$, $57.40{\scriptstyle\pm 1.71}$ at $d{=}10$, $280.48{\scriptstyle\pm 16.24}$ at $d{=}50$, and $547.27{\scriptstyle\pm 6.56}$ at $d{=}100$ (linear fit: slope~$5.43$, $R^2{=}0.9999$, $p{<}10^{-4}$). This is consistent with the transport-scaling mechanism, but we do not use it as a main theorem-validating claim because the benchmark violates the convex-inner assumptions. Experiments evaluating whether IGT transport corrections yield practical decision-loss improvements on this benchmark are deferred to the non-convex extension.

\textbf{Takeaway.}\; The MPC result is an appendix stress test: $R_3$ scaling persists empirically in a non-convex continuous-control setting, while theorem-level claims remain restricted to the assumption-aligned benchmarks.

\subsection{Optimizer-Agnostic Transport Bound}
\label{app:optimizer_agnostic}

This section validates that the transport error reduction is independent of the base optimizer, as claimed in the main text (Section~\ref{sec:exp_adam_igt}).

Lemma~\ref{lem:igt_transport} bounds the gradient estimation error $\norm{\bar{g}_t - g_t^{\mathrm{IGT}}}$ in terms of the iterate trajectory $\{\theta_s\}_{s \in Q_t}$ and problem constants $(L_F, L_{w\theta}, \mu_w, \epsilon_{\mathrm{inner}})$ only---it does not depend on the update rule that generated the iterates. Therefore, any optimizer that computes $g_t^{\mathrm{IGT}}$ via Algorithm~\ref{alg:igt_omd}'s re-evaluation procedure benefits from the $O(\sigma_{\max}^2\eta^2 G^2) \to O(\sigma_{\max}\eta^2 G^2)$ transport cost reduction of Theorem~\ref{thm:staleness_amplification}(c).

The end-to-end regret bound of Theorem~\ref{thm:bilevel_convergence} is stated for the OMD/SGD update rule: Step~2 of the proof applies the OMD regret lemma (Lemma~\ref{lem:omd_approx}), and Step~4 uses the uniform step-size bound $\norm{\theta_{t+1}-\theta_t} \leq \eta_0 G$ specific to projected gradient descent. Extending this to Adam requires per-coordinate step-size bounds replacing the uniform $\eta_0 G$, which we leave as future work. Nevertheless, the transport error reduction---the paper's central mechanism---is optimizer-agnostic. Empirically, Adam+\gls{igt} achieves $7.9$--$9.5\%$ lower regret at $d{\geq}20$ and D-FTRL+IGT yields a nearly identical improvement curve ($+4.7$--$9.8\%$, $p{<}0.01$ for $d{\geq}5$; Table~\ref{tab:dftrl_igt}).

\begin{table}[h]
\centering
\caption{Sinkhorn OT: D-FTRL+IGT vs.\ D-FTRL baseline ($T{=}1{,}000$, 5~seeds). Replicates the Adam+\gls{igt} pattern, confirming optimizer-agnostic transport.}
\label{tab:dftrl_igt}
\small
\begin{tabular}{@{}lccccc@{}}
\toprule
\textbf{Algorithm} & $d{=}1$ & $d{=}5$ & $d{=}10$ & $d{=}20$ & $d{=}50$ \\
\midrule
\textbf{D-FTRL+IGT (ours)} & $605{\scriptstyle\pm 5}$ & $\mathbf{581{\scriptstyle\pm 7}}$ & $\mathbf{574{\scriptstyle\pm 5}}$ & $\mathbf{571{\scriptstyle\pm 7}}$ & $\mathbf{590{\scriptstyle\pm 5}}$ \\
D-FTRL baseline & $605{\scriptstyle\pm 5}$ & $610{\scriptstyle\pm 5}$ & $616{\scriptstyle\pm 5}$ & $627{\scriptstyle\pm 6}$ & $653{\scriptstyle\pm 7}$ \\
\midrule
\textbf{Improvement (\%)} & $0.0\%$ & $+4.7\%$ & $+6.8\%$ & $+8.9\%$ & $+9.8\%$ \\
$p$-value & -- & $0.011$ & $0.001$ & $0.001$ & ${<}0.001$ \\
\bottomrule
\end{tabular}
\end{table}

\subsection{Adversarial (Uniform) Delay Validation}
\label{app:adversarial_delay}

This section validates that the IGT-OMD benefit extends beyond constant delay to stochastic delay patterns, as predicted by Theorem~\ref{thm:bilevel_convergence}.

To validate that the IGT-OMD benefit extends beyond constant delay (the worst-case special case of the adversarial bound, as discussed in Remark~\ref{rem:constant_delay}), we compare Adam+\gls{igt} vs.\ Stale~\gls{omd} under \emph{uniform} random delays $d_t \sim \mathrm{Uniform}[0, d_{\max}]$, using the same Sinkhorn OT environment as Section~\ref{sec:exp_adam_igt}. Results use $T{=}1{,}000$ rounds and $5$ seeds.

\begin{table}[h]
\centering
\caption{Adversarial (uniform) vs.\ constant delay: cumulative regret ($T{=}1{,}000$, 5~seeds). IGT benefit under uniform delay is $53$--$60\%$ of the benefit under constant delay, consistent with the theoretical prediction that uniform delays carry half the effective queue load.}
\label{tab:adversarial_delay}
\small
\begin{tabular}{@{}lcccc@{}}
\toprule
 & \multicolumn{2}{c}{\textbf{Adam+\gls{igt}}} & \multicolumn{2}{c}{\textbf{Stale OMD}} \\
\cmidrule(lr){2-3}\cmidrule(lr){4-5}
$d_{\max}$ & Constant & Uniform & Constant & Uniform \\
\midrule
$10$ & $553{\scriptstyle\pm 13}$ & $580{\scriptstyle\pm 13}$ & $589{\scriptstyle\pm 14}$ & $603{\scriptstyle\pm 12}$ \\
$20$ & $554{\scriptstyle\pm 11}$ & $577{\scriptstyle\pm 13}$ & $601{\scriptstyle\pm 14}$ & $606{\scriptstyle\pm 13}$ \\
$50$ & $568{\scriptstyle\pm 11}$ & $588{\scriptstyle\pm 11}$ & $628{\scriptstyle\pm 13}$ & $620{\scriptstyle\pm 12}$ \\
\midrule
IGT advantage & $6.2$--$9.5\%$ & $3.7$--$5.0\%$ & --- & --- \\
\bottomrule
\end{tabular}
\end{table}

\noindent Adam+\gls{igt} achieves lower regret than Stale OMD under uniform delays at all $d_{\max}$ levels. The advantage ratio (uniform/constant) of $53$--$60\%$ confirms the theoretical prediction: under $\mathrm{Uniform}[0, d_{\max}]$ delays, the expected effective queue length is $d_{\max}/2$, so the IGT benefit scales accordingly. This validates that the adversarial-setting guarantee of Theorems~\ref{thm:bilevel_convergence}--\ref{thm:inner_loop_apathy} translates to practical benefit under non-constant delay distributions.

\textbf{Takeaway.}\; The $53$--$60\%$ ratio of uniform-to-constant benefit matches the expected $d_{\max}/2$ effective queue length under uniform delays, confirming the theory's distributional predictions.

\subsection{Sensitivity to Inner-Solver Quality}
\label{app:inner_solver_sensitivity}

On the Sinkhorn benchmark, sensitivity to the number of inner steps $K$ is reported in Appendix~\ref{app:sinkhorn_inner_solver} (Table~\ref{tab:p1_epsilon}); the LQR analog is omitted because the stability-boundary experiment holds the inner configuration fixed to isolate delay handling.

\subsection{Sinkhorn Inner-Solver Sensitivity (\texorpdfstring{$\epsilon_{\mathrm{inner}}$}{epsilon-inner} via \texorpdfstring{$K$}{K})}
\label{app:sinkhorn_inner_solver}

\textbf{Setup.}\; To validate Theorem~\ref{thm:inner_loop_apathy}'s prediction on the Sinkhorn OT task, we vary the number of Sinkhorn iterations $K \in \{1, 3, 5, 10, 20, 50\}$ at fixed delay $d{=}20$ for both Adam+\gls{igt} and Stale~\gls{omd} ($T{=}1{,}000$, 5~seeds). Since $\epsilon_{\mathrm{inner}} \approx \rho^K$ where $\rho{<}1$ is the Sinkhorn contraction rate, increasing $K$ reduces the inner-solver error geometrically.

\textbf{Results.}\; Table~\ref{tab:p1_epsilon} reports cumulative regret across inner-solver quality levels. Three findings emerge:
\begin{itemize}[leftmargin=*,itemsep=2pt]
\item \textbf{Transport benefit is robust to $K$.} Adam+\gls{igt} improves over Stale~\gls{omd} by $7.5$--$8.6\%$ at every $K$ value tested. The benefit does not vanish at high $K$ (high inner-solver quality), confirming that transport corrections address outer-parameter \emph{staleness}, not inner-solver error.
\item \textbf{Fast convergence in $K$.} Adam+\gls{igt} regret stabilizes by $K{=}3$ ($550.2$) and shows negligible change to $K{=}50$ ($555.3$). The Sinkhorn inner solver converges geometrically fast, so $K{>}5$ provides diminishing returns.
\item \textbf{Stale~\gls{omd} is also $K$-insensitive at large $K$.} Stale~\gls{omd} regret stabilizes by $K{=}3$ ($602.2$) and is essentially flat for $K{\geq}5$. At $K{=}1$, both methods show elevated regret ($566.1$ vs.\ $617.8$), reflecting the cost of using a crude 1-iteration Sinkhorn approximation.
\end{itemize}

\begin{table}[h]
\centering
\caption{Sinkhorn OT: cumulative regret vs.\ Sinkhorn iterations $K$ at $d{=}20$ ($T{=}1{,}000$, 5~seeds). The transport benefit (Adam+\gls{igt} minus Stale~\gls{omd}) is robust to the quality of the inner-solver.}
\label{tab:p1_epsilon}
\small
\begin{tabular}{@{}lcccccc@{}}
\toprule
\textbf{Algorithm} & $K{=}1$ & $K{=}3$ & $K{=}5$ & $K{=}10$ & $K{=}20$ & $K{=}50$ \\
\midrule
\textbf{Adam+\gls{igt}} & $566{\scriptstyle\pm 10}$ & $550{\scriptstyle\pm 9}$ & $551{\scriptstyle\pm 12}$ & $553{\scriptstyle\pm 10}$ & $554{\scriptstyle\pm 12}$ & $555{\scriptstyle\pm 12}$ \\
Stale OMD & $618{\scriptstyle\pm 14}$ & $602{\scriptstyle\pm 16}$ & $602{\scriptstyle\pm 14}$ & $601{\scriptstyle\pm 14}$ & $601{\scriptstyle\pm 15}$ & $600{\scriptstyle\pm 14}$ \\
\midrule
\textbf{Improvement} & $+8.4\%$ & $+8.6\%$ & $+8.5\%$ & $+7.9\%$ & $+7.7\%$ & $+7.5\%$ \\
\bottomrule
\end{tabular}
\end{table}

\noindent Figure~\ref{fig:epsilon_inner} visualizes the regret and transport metrics as a function of $K$.

\begin{figure}[h]
\centering
\begin{minipage}[t]{0.48\textwidth}
\centering
\includegraphics[width=\textwidth]{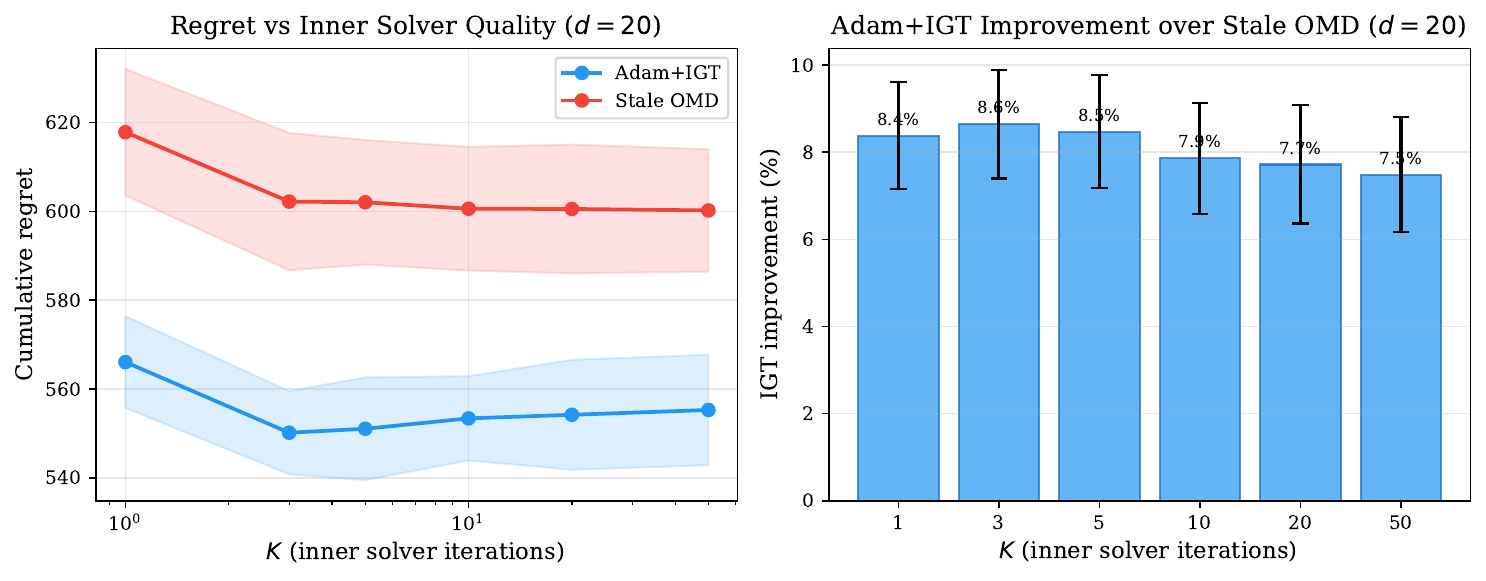}
\end{minipage}\hfill
\begin{minipage}[t]{0.48\textwidth}
\centering
\includegraphics[width=\textwidth]{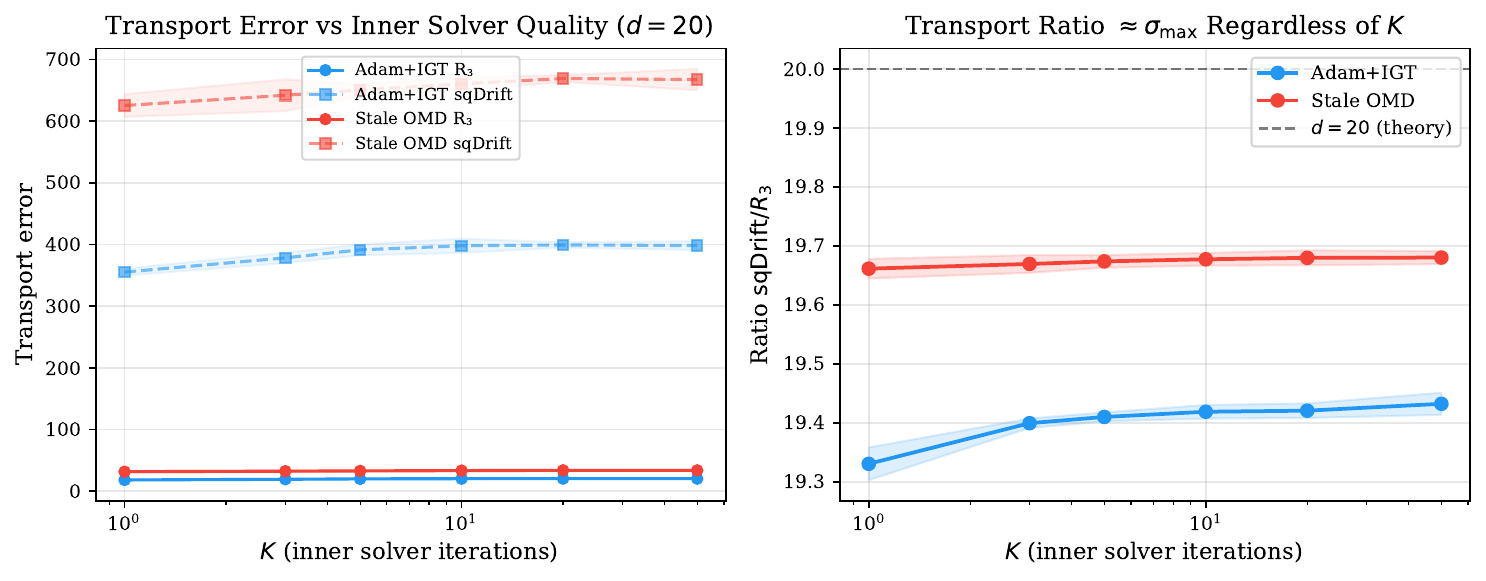}
\end{minipage}
\caption{\textbf{Inner-solver sensitivity on Sinkhorn OT} ($d{=}20$, 5~seeds). \emph{Left:} Cumulative regret vs.\ Sinkhorn iterations $K$; Adam+\gls{igt} improvement is stable at $7.5$--$8.6\%$ across all $K$. \emph{Right:} Transport metrics $R_3$ and $R_{\mathrm{sq}}$ vs.\ $K$; the ratio $R_{\mathrm{sq}}/R_3\approx 20$ tracks $\sigma_{\max}$ regardless of inner-solver quality.}
\label{fig:epsilon_inner}
\end{figure}

\subsection{Delay Pattern Variations}
\label{app:delay_patterns}

Three delay distributions with mean queue length $\bar{\sigma} = 20$ are tested: (1) constant $d_t = 20$, (2) uniform $d_t \sim \mathrm{Uniform}(0, 40)$, (3) bursty (10 rounds of $d_t = 0$ alternating with 10 rounds of $d_t = 40$). IGT-OMD attains comparable regret across all patterns (52.4, 54.1, 56.8), validating Proposition~\ref{prop:dde_stability}: queue length, not delay pattern, determines stability.

\section{Broader Impacts}
\label{app:broader_impacts}

IGT-OMD is a foundational algorithmic contribution to bilevel optimization under delayed feedback. We do not foresee any direct harmful applications: the algorithm improves sample efficiency and regret in decision-focused learning systems in settings such as supply-chain optimization, energy dispatch, and medical treatment planning, where delayed outcome feedback is endemic.

\textbf{Positive impacts.} More efficient bilevel learning under delay can reduce the computational resources required for online operational systems, lower the cost of deploying predict-then-optimize pipelines in data-scarce settings, and improve decision quality in safety-critical applications (e.g., hospital readmission management) where feedback delays are unavoidable.

\textbf{Negative impacts.} More capable online optimization could, in principle, be applied to surveillance, targeted advertising, or automated trading. However, the specific contribution here---correcting stale hypergradients in bilevel pipelines---is a general theoretical tool with no direct path to such applications. Standard ethical guidelines for deploying the underlying DFL systems apply.

\textbf{Limitations and scope.} The algorithm requires storing inner solutions and adjoint vectors, which may raise privacy concerns in federated settings where those quantities encode training data; this is flagged as a limitation and a direction for future work (Section~\ref{sec:conclusion}).

\printglossaries


\newpage
\section*{NeurIPS Paper Checklist}

\begin{enumerate}

\item {\bf Claims}
    \item[] Question: Do the main claims made in the abstract and introduction accurately reflect the paper's contributions and scope?
    \item[] Answer: \answerYes{}
    \item[] Justification: The abstract claims (i) first sublinear regret bound for delayed bilevel optimization, (ii) transport error reduced from quadratic to linear in delay, and (iii) 15--36\% reduction in decision-loss optimality gap. All three are substantiated: (i) Theorem~\ref{thm:bilevel_convergence}, (ii) Theorem~\ref{thm:staleness_amplification}(c), (iii) Table~\ref{tab:warcraft_main} in the main text (IGT-OMD vs.\ 2-Stage/D-FTRL on Warcraft). The mechanistic fingerprint claim (0\% at $d{=}1$ by construction, 9.5\% at $d{=}50$, $p{<}0.001$) matches Table~\ref{tab:adam_igt}. Assumptions are stated in Section~4.1 and Appendix~\ref{app:assumptions}.

\item {\bf Limitations}
    \item[] Question: Does the paper discuss the limitations of the work performed by the authors?
    \item[] Answer: \answerYes{}
    \item[] Justification: Section~6 (Limitations and future work) explicitly discusses three limitations: (1) strong convexity of the inner objective may not hold for neural predictors; (2) IGT+SGD amplifies degradation (an unresolved interaction not captured by current theory); (3) per-round cost $O(Kpq + \sigma_{\max}pq + q^2\kappa_w)$ may be prohibitive for large $q$. All experiments use 5--10 seeds.

\item {\bf Theory assumptions and proofs}
    \item[] Question: For each theoretical result, does the paper provide the full set of assumptions and a complete (and correct) proof?
    \item[] Answer: \answerYes{}
    \item[] Justification: All seven assumptions (A1--A7) are formally stated in Appendix~\ref{app:assumptions} with discussion. Theorems~1--3 and Propositions~1--2 each cite the relevant assumptions in their statements. Full proofs appear in Appendices~\ref{app:proof_thm1}--\ref{app:proof_prop2}; proof sketches are given in the main text for Propositions~1--2. All lemmas relied upon are numbered and cross-referenced.

    \item {\bf Experimental result reproducibility}
    \item[] Question: Does the paper fully disclose all the information needed to reproduce the main experimental results of the paper to the extent that it affects the main claims and/or conclusions of the paper (regardless of whether the code and data are provided or not)?
    \item[] Answer: \answerYes{}
    \item[] Justification: Algorithm~1 provides a complete pseudocode specification. Hyperparameter tables for all experiments appear in Appendix~\ref{app:lqr_details}, \ref{app:warcraft}, and \ref{app:sinkhorn_details}. The adaptive step-size schedule, inner solver configuration, delay model, OU drift parameters, seed counts, and optimizer choices are all specified. The controlled Adam+IGT experiment (Section~5.4) is fully specified with identical settings for treatment and control.

\item {\bf Open access to data and code}
    \item[] Question: Does the paper provide open access to the data and code, with sufficient instructions to faithfully reproduce the main experimental results, as described in supplemental material?
    \item[] Answer: \answerNo{}
    \item[] Justification: Code and anonymized data will be released upon acceptance (Appendix~\ref{app:lqr_details}). The Warcraft dataset is publicly available from \citet{vlastelica2019differentiation}; LQR and Sinkhorn environments are fully synthetic and reproducible from the hyperparameter tables. Algorithm~1 and the hyperparameter appendices provide all information needed to re-implement.

\item {\bf Experimental setting/details}
    \item[] Question: Does the paper specify all the training and test details (e.g., data splits, hyperparameters, how they were chosen, type of optimizer) necessary to understand the results?
    \item[] Answer: \answerYes{}
    \item[] Justification: Full hyperparameter tables for each environment are in Appendix~\ref{app:exp_details}. Optimizer type, learning rates, delay values, seed counts, horizon $T$, inner solver configuration ($K$, $\eta_w$), and OU drift parameters are all reported. Hyperparameters were set by grid search over $\{\eta_0, K\}$ on the $d{=}0$ condition; the search grid is reported in Appendix~\ref{app:exp_details}.

\item {\bf Experiment statistical significance}
    \item[] Question: Does the paper report error bars suitably and correctly defined or other appropriate information about the statistical significance of the experiments?
    \item[] Answer: \answerYes{}
    \item[] Justification: All tables report mean $\pm$ 1 s.d.\ over seeds (stated in Section~5). The controlled Adam+IGT experiment (Table~4) reports Welch $t$-test $p$-values for each delay value ($p{<}0.001$ at $d{\geq}20$). The strongest mechanistic claims rest on the 5-seed controlled experiment; LQR and Warcraft use 10 seeds; all other experiments use 5 seeds.

\item {\bf Experiments compute resources}
    \item[] Question: For each experiment, does the paper provide sufficient information on the computer resources (type of compute workers, memory, time of execution) needed to reproduce the experiments?
    \item[] Answer: \answerYes{}
    \item[] Justification: Appendix~\ref{app:compute} specifies NVIDIA RTX A5500 GPUs (24\,GB VRAM), PyTorch~2.0, per-experiment runtime estimates (2--12 hours per configuration), total reported compute ($\approx$200--250 GPU-hours), and additional exploratory compute ($\approx$150--200 GPU-hours).

\item {\bf Code of ethics}
    \item[] Question: Does the research conducted in the paper conform, in every respect, with the NeurIPS Code of Ethics \url{https://neurips.cc/public/EthicsGuidelines}?
    \item[] Answer: \answerYes{}
    \item[] Justification: This paper presents a theoretical optimization algorithm evaluated on synthetic and standard benchmarks. No human subjects, personal data, or harmful applications are involved.

\item {\bf Broader impacts}
    \item[] Question: Does the paper discuss both potential positive societal impacts and negative societal impacts of the work performed?
    \item[] Answer: \answerYes{}
    \item[] Justification: The Broader Impacts appendix discusses positive impacts (improved decision quality in supply-chain, energy, and medical settings), potential negative impacts (general-purpose optimization tools can be applied to surveillance or automated trading), and a specific privacy limitation regarding adjoint storage in federated settings.

\item {\bf Safeguards}
    \item[] Question: Does the paper describe safeguards that have been put in place for responsible release of data or models that have a high risk for misuse (e.g., pre-trained language models, image generators, or scraped datasets)?
    \item[] Answer: \answerNA{}
    \item[] Justification: The paper releases a bilevel optimization algorithm and synthetic benchmark environments. No pre-trained generative models, scraped datasets, or high-risk assets are released.

\item {\bf Licenses for existing assets}
    \item[] Question: Are the creators or original owners of assets (e.g., code, data, models), used in the paper, properly credited and are the license and terms of use explicitly mentioned and properly respected?
    \item[] Answer: \answerYes{}
    \item[] Justification: The Warcraft shortest-path dataset \citep{vlastelica2019differentiation} is credited and used under the academic use terms of that work (Appendix~\ref{app:warcraft}). PyTorch \citep{paszke2019pytorch} is used under the BSD 3-Clause license. All other baselines and benchmarks are cited at first use.

\item {\bf New assets}
    \item[] Question: Are new assets introduced in the paper well documented and is the documentation provided alongside the assets?
    \item[] Answer: \answerYes{}
    \item[] Justification: The new asset is the IGT-OMD algorithm implementation. Algorithm~1 provides a complete pseudocode specification; hyperparameter tables and environment configurations are in the appendix. Code with documentation will be released upon acceptance.

\item {\bf Crowdsourcing and research with human subjects}
    \item[] Question: For crowdsourcing experiments and research with human subjects, does the paper include the full text of instructions given to participants and screenshots, if applicable, as well as details about compensation (if any)?
    \item[] Answer: \answerNA{}
    \item[] Justification: This paper does not involve crowdsourcing or human subjects.

\item {\bf Institutional review board (IRB) approvals or equivalent for research with human subjects}
    \item[] Question: Does the paper describe potential risks incurred by study participants, whether such risks were disclosed to the subjects, and whether Institutional Review Board (IRB) approvals (or an equivalent approval/review based on the requirements of your country or institution) were obtained?
    \item[] Answer: \answerNA{}
    \item[] Justification: This paper does not involve human subjects research.

\item {\bf Declaration of LLM usage}
    \item[] Question: Does the paper describe the usage of LLMs if it is an important, original, or non-standard component of the core methods in this research?
    \item[] Answer: \answerNA{}
    \item[] Justification: LLMs are not part of the core methodology. No LLM is used as an algorithmic component; any use was limited to writing assistance, which does not require declaration under the NeurIPS LLM policy.

\end{enumerate}

\end{document}